\documentclass[letter]{article}

\PassOptionsToPackage{numbers}{natbib}
\usepackage[dvipsnames]{xcolor}
\usepackage{iclr2020_conference,times}
\usepackage{microtype}
\usepackage{graphicx}
\usepackage{booktabs} %
\usepackage[utf8]{inputenc} %
\usepackage[T1]{fontenc}    %
\usepackage{hyperref}       %
\usepackage{url}            %
\usepackage{booktabs}       %
\usepackage{amsfonts}       %
\usepackage{nicefrac}       %
\usepackage{microtype}      %
\usepackage{mathtools}
\usepackage{amssymb}
\usepackage{amsthm}
\usepackage{bbm}
\usepackage{paralist}
\usepackage{subcaption}
\usepackage{multirow}
\usepackage{wrapfig}
\usepackage{thmtools,thm-restate}
\usepackage{enumitem}
\usepackage[nameinlink,capitalise]{cleveref}
\usepackage{glossaries}
\usepackage{todonotes}
\usepackage{amartya_preamble}
\graphicspath{{./figs/}{./figs/st_pr_70_img/}{./figs/sn_pr_img/}{./figs/wgan_img/}{./figs/celeb_images_proper/}{./figs/sn_celeb_img_proper/}}

\usepackage{algpseudocode}
\usepackage{algorithm}

\usepackage{array}
\newcommand{\PreserveBackslash}[1]{\let\temp=\\#1\let\\=\temp}
\newcolumntype{C}[1]{>{\PreserveBackslash\centering}p{#1}}
\newcolumntype{R}[1]{>{\PreserveBackslash\raggedleft}p{#1}}
\newcolumntype{L}[1]{>{\PreserveBackslash\raggedright}p{#1}}

\usepackage{paralist}

\definecolor{red}{HTML}{E41A1C}
\definecolor{orange}{HTML}{FF7F00}
\definecolor{yellow}{HTML}{FFC020}
\definecolor{green}{HTML}{4DAF4A}
\definecolor{blue}{HTML}{377EB8}
\definecolor{purple}{HTML}{984EA3}

\hypersetup{%
  colorlinks,
  linkcolor={violet!70!black},
  citecolor={YellowOrange!70!black},
  urlcolor={Aquamarine!70!black}
}

\newtheorem{proposition}{Proposition}[section]

\newtheorem{definition}{Definition}[section]

\Crefname{algocf}{Algorithm}{Algorithms}
\crefname{algorithm}{Algorithm}{Algorithms}
\crefname{figure}{Figure}{Figure}
\crefname{table}{Table}{Table}
\crefname{section}{Section}{Section}
\crefformat{equation}{(#2#1#3)}
\crefname{property}{Property}{Property}
\crefname{theorem}{Theorem}{Theorem}
\crefname{proposition}{Proposition}{Proposition}
\crefname{lemma}{Lemma}{Lemma}
\crefname{definition}{Definition}{Definition}

\glsdisablehyper{}
\newglossaryentry{lip}
{
    name=lipschitz,
    description={Lipschtiz constant}
}
\newglossaryentry{Lip}
{
    name=lipschitz,
    description={Lipschtiz constant}
}

\newacronym{nn}{NN}{neural networks}
\newacronym{sgd}{SGD}{stochastic gradient descent}
\newacronym{svd}{SVD}{Singular Value Decomposition}
\newacronym{sota}{SOTA}{state-of-the-art}
\newacronym{gan}{GAN}{Generative Adversarial Networks}
\newacronym{srn}{SRN}{Stable Rank Normalization}
\newacronym{srngan}{SRN-GAN}{Stable Rank Normalization GAN}
\newacronym{sngan}{SN-GAN}{Spectral Normalization GAN}
\newacronym{sn}{SN}{Spectral Normalization}
\presetkeys{todonotes}{%
  backgroundcolor=blue!10!white,
  linecolor=blue!10!white,
  bordercolor=blue!10!white
}{}

\DeclareMathOperator*{\argmin}{arg\,min} %

\newcommand{\bfx}{\mathbf{x}}
\newcommand{\bfy}{\mathbf{y}}
\newcommand{\real}{\mathbb{R}}
\newcommand{\bfa}{\mathbf{a}}
\newcommand{\bfb}{\mathbf{b}}
\newcommand{\bfz}{\mathbf{z}}
\newcommand{\bfh}{\mathbf{h}}
\newcommand{\bfu}{\mathbf{u}}
\newcommand{\bfv}{\mathbf{v}}

\newcommand{\eye}{\mathbb{I}}

\newcommand{\ip}[2]{\left\langle{#1},{#2}\right\rangle}

\newcommand{\wmat}{\mathbf{W}}

\newcommand{\wmathw}{\widehat{\mathbf{W}}}

\newcommand{\amat}{\mathbf{A}}
\newcommand{\bmat}{\mathbf{B}}

\newcommand{\norm}[1]{\left\lVert#1\right\rVert}
\newcommand{\srank}[1]{\mathrm{srank}(#1)}

\newcommand{\forb}{\mathrm{F}}

\newcommand{\data}{\mathcal{D}}

\newcommand{\eg}{{\em e.g.~}}

\newtheorem{thm}{Theorem}

\newtheorem{lem}[thm]{Lemma}

\newtheorem{claim}[thm]{Claim}  %
\iclrfinalcopy

\newcommand{\tildeO}[1]{\widetilde{\cO}\br{{#1}}}

\renewcommand{\vec}[1]{{\mathbf{#1}}}
\newcommand{\bc}[1]{\left\{{#1}\right\}}
\newcommand{\br}[1]{\left({#1}\right)}
\newcommand{\bs}[1]{\left[{#1}\right]}

\title{Stable Rank Normalization for Improved Generalization in Neural
  Networks and GANs}

\author{Amartya Sanyal \\
  Department of Computer Science\\
  University of Oxford,\\
The Alan Turing Institute \\
\texttt{amartya.sanyal@cs.ox.ac.uk} \\
\And
Philip H. Torr \\
Department of Engineering Science\\
University of Oxford \\
\texttt{philip.torr@eng.ox.ac.uk} \\
\And
Puneet K. Dokania \\
Department of Engineering Science\\
University of Oxford \\
\texttt{puneet@robots.ox.ac.uk}
}

\newcommand{\abs}[1]{\left| {#1} \right|}
\begin{document}
\maketitle

\begin{abstract}

Exciting new work on generalization bounds for \gls{nn} given
by~\citet{bartlett2017spectrally,neyshabur2018a} closely depend on two parameter-dependant quantities
\begin{inparaenum}[a)]
\item the \Gls{lip} constant upper bound and 
\item the {\em stable rank} (a softer version of rank).
\end{inparaenum} %
Even though these bounds typically have minimal practical utility, they facilitate questions on whether controlling such quantities together could improve the generalization behaviour of \gls{nn}s in practice. To this end, we propose {\em stable rank normalization} (SRN), a novel, provably optimal, and computationally efficient weight-normalization scheme which minimizes the {\em stable rank} of a linear operator. 
Surprisingly we find that SRN, despite being {\em non-convex}, can be shown to have a unique optimal solution.
We provide extensive analyses across a wide variety of \gls{nn}s (DenseNet, WideResNet, ResNet, Alexnet, VGG), where applying SRN to their linear layers leads to improved classification accuracy, while simultaneously showing  improvements in generalization, evaluated empirically using shattering experiments~\citep{Zhang2016}; and three measures of sample complexity by~\citet{bartlett2017spectrally}, \citet{neyshabur2018a}, \& \citet{wei2019}. %
Additionally, we show that, when applied to the discriminator of GANs, it improves Inception, FID, and Neural divergence scores, while learning mappings with a low empirical \Gls{lip} constant.
\end{abstract}

\section{Introduction}
\label{sec:intro}

Deep neural networks have shown astonishing ability in tackling a wide
variety of machine learning problems including a great ability to
generalize under extreme over-parameterization. Within this work we
leverage very recent, and important, theoretical results on the
generalization bounds of deep networks to yield a practical low cost
method to normalize the weights within a network using a scheme -
which we call \gls{srn}. The motivation for \gls{srn} comes from the
generalization bound for \glspl{nn} given by~\citet{neyshabur2018a}
and \citet{bartlett2017spectrally}, \( \small \cO{ \sqrt{\prod_{i}^d\norm{\wmat_i}_2^2\sum_{i=1}^d  \srank{\wmat_i}}} \)
\footnote{$d$ and $\norm{\wmat_i}_2$ represents the number of layers and the spectral norm of the $i$-th linear layer $\wmat_i$, respectively.}
, that depend on two parameter-dependent quantities: 
\begin{inparaenum}[a)]
\item the scale-dependent \Gls{lip} constant upper-bound $\prod_{i}^d\norm{\wmat_i}_2$ (product of spectral norms) and 
\item the sum of scale-independent {\em stable ranks} ($\srank{\wmat}$).
\end{inparaenum} 
Stable rank is a softer version of the rank operator and is defined as
the squared ratio of the Frobenius norm to the spectral norm.
Although these two terms appear frequently in these bounds, the empirical impact of simultaneously
controlling them on the generalization behaviour of
\gls{nn}s has not been explored yet possibly because of the
difficulties associated with optimizing stable rank. 
This is precisely the goal of this work and based on extensive experiments across a wide variety of \gls{nn} architectures, 
we show that, indeed, controlling them simultaneously improves the generalization behaviour, while improving the classification performance of \gls{nn}s. We observe improved training of
\gls{gan}~\cite{goodfellow2014generative} as well. 

To this end, we propose \acrfull{srn} which allows us to
simultaneously control the \Gls{lip} constant and the stable rank of a
linear operator. Note that the widely used
\gls{sn}~\citep{miyato2018spectral} allows explicit control over the
\Gls{lip} constant, however, as will be discussed in the paper, it
does not have any impact on the stable rank. We would like to
emphasize that, as opposed to \gls{sn},~the \gls{srn} solution is optimal and unique even in situations when it is non-convex. It is one of those rare cases where an optimal solution to a provably non-convex problem could be obtained. Computationally, our proposed \gls{srn} for \gls{nn}s is no more complicated than \gls{sn}, just requiring computation of the largest singular value which can be efficiently obtained using the power iteration method~\citep{Mises1929}.
\paragraph{Experiments} Although \gls{srn} is in principle applicable to any problem involving a sequence of affine transformations, considering recent interests, we show its effectiveness when applied to the linear layers of deep neural networks. We perform extensive experiments on a wide variety of \gls{nn} architectures (DenseNet, WideResNet, ResNet, Alexnet, VGG) for the analyses and show that, \gls{srn}, while providing the best classification accuracy (compared against standard, or vanilla, training and \gls{sn}), consistently shows improvement on the generalization behaviour. 
We also experiment with \gls{gan}s and show that, \gls{srn} prefers learning discriminators with low empirical \Gls{lip} while providing improved Inception, FID and Neural Divergence scores~\citep{gulrajani2018towards}. 

We would like to note that although \gls{sn} is being widely used for the training of \gls{gan}s, its effect on the generalization behaviour over a wide variety of \gls{nn}s has not yet been explored. To the best of our knowledge, we are the first to do so.
\paragraph{Contributions}
\begin{compactitem}
	\item We propose \gls{srn}--- a novel normalization scheme for simultaneously controlling the \Gls{lip} constant and the stable rank of a linear operator.
	\item Optimal and unique solution to the provably non-convex stable rank normalization problem.
	\item Efficient and  easy to implement \gls{srn} algorithm for \gls{nn}s.
\end{compactitem}

\section{Background and Intuitions}
\label{sec:background}
\vspace{-2pt}
\paragraph{Neural Networks} Consider $f_{\theta}: \real^m \rightarrow \real^k$ to be a feed-forward multilayer \gls{nn} parameterized by $\theta \in \real^n$, each layer of which consists of a linear followed by a non-linear\footnote{\eg ReLU, tanh, sigmoid, and maxout.} mapping. Let $\bfa_{l-1} \in \real^{n_{l-1}}$ be the input (or pre-activations) to the $l$-th layer, then the output (or activations) of this layer is represented as $\bfa_l = \phi_l(\bfz_l)$, where $\bfz_l = \wmat_l \bfa_{l-1} + \bfb_l$ is the output of the linear (affine) layer parameterized by the weights $\wmat_l \in \real^{n_{l-1} \times n_l}$ and biases $\bfb_l \in \real^{n_l}$, and $\phi_l(.)$ is the element-wise non-linear function applied to $\bfz_l$. For classification tasks, given a dataset with input-output pairs denoted as $(\bfx \in \real^m, \bfy \in \{0,1\}^k; \sum_j y_j = 1)$
\footnote{$y_j$ is the $j$-th element of vector $\bfy$. Only one class is assigned as the ground-truth label to each $\bfx$.},
the parameter vector  $\theta$ is learned using back-propagation to optimize the classification loss ({\em e.g.}, cross-entropy). 
\paragraph{\gls{svd}} Given $\wmat \in \real^{s \times r}$ with rank $k \leq \min (s,r)$, we denote $\{\sigma_i\}_{i=1}^k$, $\{\bfu_i\}_{i=1}^k$, and $\{\bfv_i\}_{i=1}^k$ as its singular values, left singular vectors, and right singular vectors, respectively. Throughout this work, a set of singular values is assumed to be sorted $\sigma_1 \geq \cdots \geq \sigma_k$. $\sigma_i(\wmat)$ denotes the $i$-th singular value of the matrix $\wmat$. Using singular values, the matrix 2-norm $\norm{\wmat}_2$ and the Frobenius norm $\norm{\wmat}_\forb$ can be computed as $\sigma_1$ and $\sqrt{\sum_i \sigma_i^2}$, respectively.
\paragraph{Stable Rank} Below we provide the formal definition and
some properties of stable rank. %
\begin{definition}
\label{def:stableRank}
The Stable Rank~\citep{Rudelson2007} of an arbitrary matrix $\wmat$ is
defined as \( \srank{\wmat} =
\frac{\norm{\wmat}_\forb^2}{\norm{\wmat}_2^2} = \frac{\sum_{i=1}^k
  \sigma_i^2(\wmat)}{\sigma_1^2(\wmat)} \), where $k$ is the rank of
the matrix. Stable rank is
\begin{compactitem}
\item  a soft version of the rank operator and, unlike rank, is less sensitive to small perturbations. 
\item differentiable as both Frobenius and Spectral norms are almost always differentiable.
\item upper-bounded by the rank: \( \srank{\wmat} = \frac{\sum_{i=1}^k \sigma_i^2(\wmat)}{\sigma_1^2(\wmat)} \le \frac{\sum_{i=1}^k \sigma_1^2(\wmat)}{\sigma_1^2(\wmat)}  = k\).
\item invariant to scaling, implying, \( \srank{\wmat} = \srank{\frac{\wmat}{\eta}} \), for any $\eta \in \real\setminus\bc{0}$.
\end{compactitem}
\end{definition}

\paragraph{\Gls{lip} Constant} Here we describe the global and the
local \Gls{lip} constants. Briefly, the \Gls{lip} constant is a
quantification of the sensitivity of the output with respect to the
change in the input. A function $f: \real^m \mapsto \real^k$ is {\em
  globally L-\Gls{lip} continuous} if \(\exists L \in \real_+: \norm{f(\bfx_i) - f(\bfx_j)}_q \leq L \norm{\bfx_i - \bfx_j}_p, \forall (\bfx_i, \bfx_j) \in \real^m \), where $\norm{\cdot}_p$ and $\norm{\cdot}_q$ represents the norms in the input and the output metric spaces, respectively. The global \Gls{lip} constant $L_g$ is:
\begin{align}
\label{eq:lipGlobal}
L_g = \max_{\substack{\bfx_i, \bfx_j \in \real^m\\\bfx_i\neq\bfx_j}} \frac{\norm{f(\bfx_i) - f(\bfx_j)}_q}{\norm{\bfx_i - \bfx_j}_p}.
\end{align}
The above definition of the \Gls{lip} constant depends on all  pairs
of inputs in the domain $\real^m\times\real^m$,~(thus,
global). However, one can define the local \Gls{lip} constant based on
the sensitivity of $f$ in the vicinity of a given point
$\bfx$. Precisely, at $\bfx$, for an arbitrarily small $\delta>0$, the
local \Gls{lip} constant is computed on the open ball of radius
$\delta$ centered at $\bfx$. Let $\bfh \in \real^m$, $\norm{\bfh}_p <
\delta$, then, similar to $L_g$, the {\em local \Gls{lip} constant} of
$f$ at $\bfx$, $L_l(\bfx)$, is greater than or equal to $\sup_{\bfh \neq 0, \norm{\vec{h}}_p < \delta}~\frac{\norm{f(\bfx + \bfh) - f(\bfx)}_q}{\norm{\bfh}_p}$. 
Assuming $f$ to be Fr\'echet differentiable, as $\bfh \rightarrow 0$,
using $f(\bfx + \bfh) - f(\bfx) \approx J_f(\bfx) \bfh$, $L_l(\bfx)$
is the matrix (operator) norm of the Jacobian $\br{J_f(\bfx) = \frac{\partial f\br{\vec{z}}}{\partial \vec{z}}\vert_{\vec{x}} \in \real^{k \times m}}$ as follows:\footnote{Here, (a) is
 by definition of local lipschitzness and (b) is due to the property
 of  norms that for any non-negative scalar~$c$, 
 $\norm{c\vec{x}} = c\norm{\vec{x}}$.}

\begin{align}
\label{eq:lipLocal}
 L_l(\bfx) \stackrel{(a)}{=} \lim_{\delta\rightarrow 0}\sup_{\substack{\bfh \neq 0 \\ \norm{\bfh}_p < \delta}} \frac{\norm{J_f(\bfx) \bfh}_q}{\norm{\bfh}_p} \stackrel{(b)}{=}  \sup_{ \substack{\bfh \neq 0 \\ \bfh \in \real^m}} \frac{\norm{J_f(\bfx) \bfh}_q}{\norm{\bfh}_p} = \norm{J_f(\bfx)}_{p,q}.
\end{align}

A function is said to be \emph{locally \Gls{lip}} with
\emph{local Lipschitz constant} $L_l$ if for all $\bfx \in \real^m$
he function is  \emph{$L_l$ locally-\Gls{lip}}
at $\vec{x}$. %
 Thus, $L_l  =\sup_{\vec{x}\in\real^m}{L_l\br{\vec{x}}}$.
\noindent
Notice that the \Gls{lip} constant (global or local), greatly depends
on the chosen norms. When $p = q = 2$, the upperbound on the local
\Gls{lip} constant at $\bfx$ boils down to the 2-matrix norm (maximum
singular value) of the Jacobian $J_f(\bfx)$~(see last equality of~\cref{eq:lipLocal}). With these preliminary definitions, in~\cref{sec:whyStable}, we discuss
more optimistic~(or empirical) estimates of $L_l$ and $L_g$, its link
with generalization and then in~\cref{sec:experiments},we show
empirically the effect of SRN on them and on generalization.
\paragraph{The local \Gls{lip} upper-bound for Neural Networks}
As mentioned earlier~\cref{eq:lipLocal}, $L_l(\bfx) = \norm{J_f(\bfx)}_{p,q}$, where,
in the case of \gls{nn}s~(proof along with why it is loose in~\cref{sec:LipNNUB})
\begin{align}
    \label{eq:lipboundNN1}
    L_l(\bfx) = \norm{J_f(\bfx)}_{p,q} \leq \norm{\wmat_l}_{p,q}
  \cdots \norm{\wmat_1}_{p,q}\quad\text{and}\quad  L_l = L_l(\bfx)
\end{align}

\section{Why Stable Rank Normalization?}
\label{sec:whyStable}
\paragraph{\Gls{lip} alone is not sufficient} Although learning low \Gls{lip} functions has been shown to provide better generalization~\citep{anthony2009neural,bartlett2017spectrally,neyshabur2018a,neyshabur2015norm,Yoshida2017,Gouk2018}, enable stable training of GANs~\citep{Arjovsky2017,Gulrajani2017,miyato2018spectral} and help provide robustness against adversarial attacks~\citep{cisse17a}, controlling \Gls{lip} upper bound alone is not sufficient to provide assurance on the generalization error. One of the reasons is that it is scale-dependent, implying, for example, even though scaling an entire ReLU network would not alter the classification behaviour, it can massively increase the \Gls{lip} constant and thus the theoretical generalization bounds.
This suggests that either the bound is of no practical utility, or at least one should regulate both---the \Gls{lip} constant, and the stable rank (scale-independent)---in a hope to see improved generalization in practice.

\paragraph{Stable rank controls the noise-sensitivity}
As shown by~\citet{arora18b}, one of the critical properties of generalizable \gls{nn}s 
is low noise sensitivity--- the ability of a network to
preferentially carry over the true signal in the data. For a given  
noise distribution $\cN$, it can be quantified as 
\[\Phi_{f_{\theta},\cN} = \max_{\bfx \in \data}\Phi_{f_{\theta},\cN}\br{\vec{x}}, \quad \textit{where} \quad \Phi_{f_{\theta},\cN}\br{\vec{x}} :=
  \bE_{\eta\sim\cN} \bs{\dfrac{\norm{f_\theta\br{\vec{x} +
        \eta\norm{\vec{x}}} -
      f_\theta\br{\vec{x}}}^2}{\norm{f_\theta\br{\vec{x}}}^2}}. \] 
For a linear mapping with parameters $\vec{W}$ and the noise
distribution being normal-$\cN\br{0,\vec{I}}$, it can be shown
that $\Phi_{f_\vec{W},\cN} \ge \srank{\vec{W}}$~(c.f. Proposition 3.1 in~\citet{arora18b}). 
Thus, decreasing the stable rank directly decreases the lower bound of the
noise sensitivity. In~\cref{fig:noise_sensit}, we show $\Phi_{f_\theta, \cN}$
of a ResNet110 trained on CIFAR100.
Note that although the \Gls{lip} upper bound of \gls{srn} and \gls{sn} are the same,
\gls{srn} (algorithmic details in~\cref{sec:stable_rank_alg}) is much less sensitive to noise as the constraints imposed
enforce the stable rank to decrease to $30\%$ of its original value, which in effect reduces the noise sensitivity.

\begin{wrapfigure}{r}{0.3\linewidth}
  \centering
  \def\svgwidth{0.99\columnwidth}
  \resizebox{0.99\linewidth}{!}{
  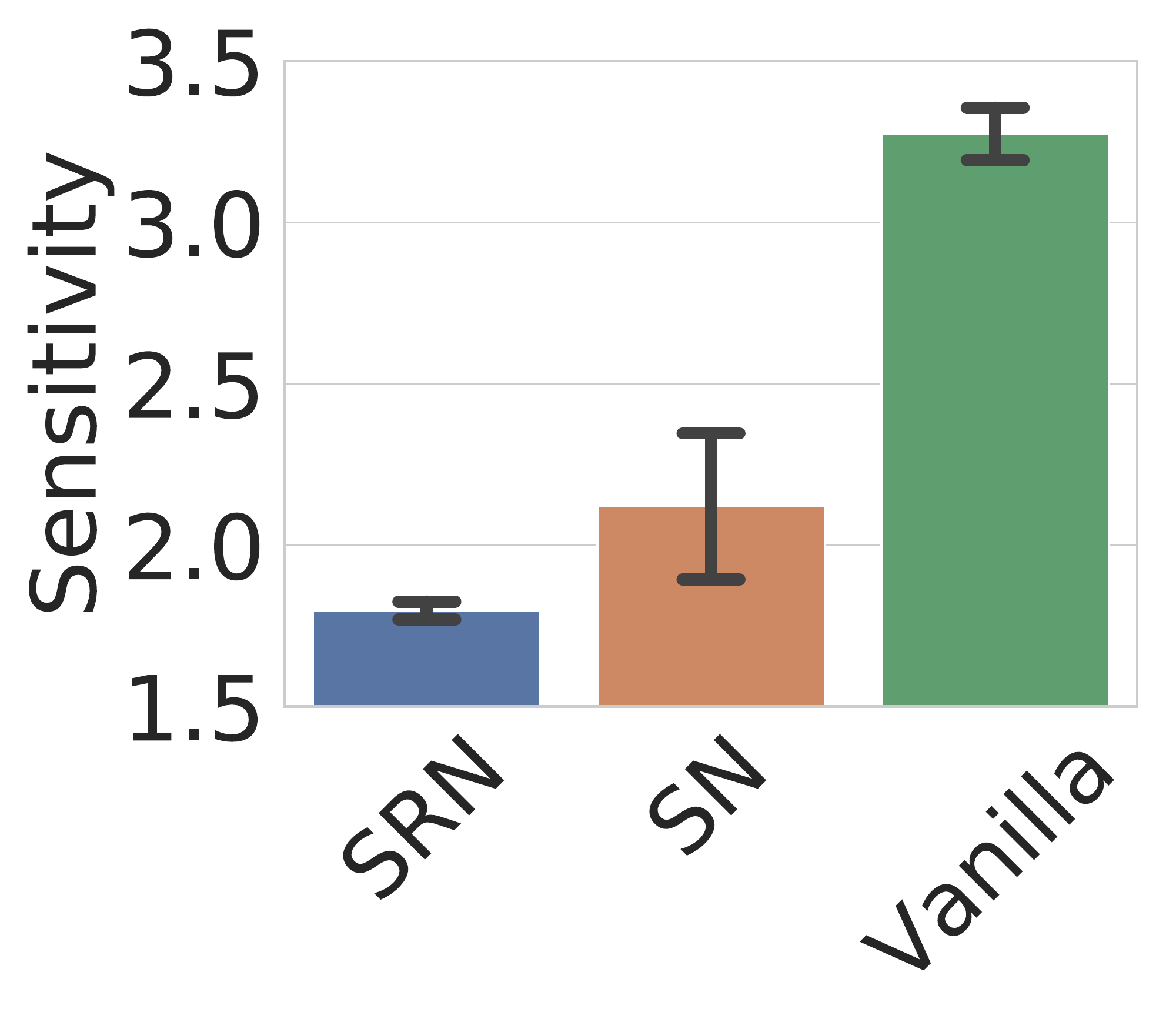
  } \vspace{-2ex}
  \caption{Noise Sensitivity (lower the better). Test accuracy: SRN ($73.1\%$), SN ($71.5\%$), and Vanilla ($72.4\%$).
  }\vspace{-2ex} \label{fig:noise_sensit}
\end{wrapfigure}
\paragraph{Stable rank impacts empirical \Gls{lip} constant}
It is apparent that the \Gls{lip} constant upper bound (\( \prod_{i}^d\norm{\wmat_i}_2\)), along with being scale-dependent, also is {\em data-independant} and hence, provides a pessimistic estimate of the behaviour of a  model on a particular task or dataset. Considering this, a relatively optimistic estimate of the model's behaviour would be an {\em empirical} estimate of the \Gls{lip} constant ($L_e$) on a task-specific dataset
\footnote{It can be the train-/test- data, the generated data ({\em
  e.g.}, in \gls{gan}s), or some interpolated
data points.}.
Note that local $L_e$ is just the norm of the Jacobian at a given point
\footnote{For completeness, we provide the relationship between the global and the local $L_e$ in~\cref{prop:empiricalLocalLip}.}. 
Empirically,~\citet{novak2018sensitivity} provided results
showing how local $L_e$ (in the vicinity of train data) is correlated
with the generalization error of \gls{nn}s. 
This observation is further supported by the work of~\citet{wei2019,nagarajan2018deterministic,arora18b} 
whereby the variants of $L_e$ are used to derive generalization bounds. 
Thus, a tool that favours low $L_e$ is likely to provide better generalization behaviour in practice.
To this end, we first consider a simple two layer linear-\gls{nn} example in~\cref{sec:rankEffectLe} 
and show that low rank mappings do favour low $L_e$. 
Since direct minimization of rank for \gls{nn}s is non-trivial,
the expectation is that learning low stable rank (softer version of rank)
might induce similar behaviour. 
We experimentally validate this hypothesis by showing that, as we decrease the stable rank, the empirical \Gls{lip} decreases. 
This shows \gls{srn} indeed prefers mappings with a low empirical \Gls{lip} constant.

\vspace{-4pt}
\section{Stable Rank Normalization}
\label{sec:stable_rank_alg}

Here we provide a theoretically sound procedure to do \gls{srn}. A big challenge in stable rank normalization comes from the fact that it is scale-invariant~(refer~\cref{def:stableRank}), thus, any normalization scheme that modifies $\wmat = \sum_i \sigma_i \bfu_i \bfv_i^{\top}$ to $\widehat{\wmat} = \sum_i \frac{\sigma_i}{\eta}  \bfu_i \bfv_i^{\top}$ will have no effect on the stable rank, making \gls{srn} non-trivial. Examples of such schemes are \gls{sn}~\citep{miyato2018spectral} where $\eta = \sigma_1$, and Frobenius normalization where $\eta = \norm{\wmat}_\forb$. 
As will be shown, our approach to stable rank normalization is optimal and efficient. Note, the widely used \gls{sn}~\citep{miyato2018spectral} is not optimal (proof in~\cref{sec:spectralNormOptimal}).

\paragraph{The \gls{srn} Problem Statement} Given a matrix $\wmat \in \real^{m \times n}$ with rank $p$ and spectral partitioning index $k$ ($0 \leq k < p$), we formulate the SRN problem as:
\begin{align}
\label{eq:srankProblem}
& \argmin_{\widehat{\wmat}_k\in\real^{m\times n}} \norm{\wmat - \widehat{\wmat}_k}_{\forb}^2 \quad \textit{s.t.} \quad \underbrace{\srank{\widehat{\wmat}_k} = r}_{\textrm{stable rank constraint}}, \; \underbrace{\lambda_i = \sigma_i, \forall i \in \{1, \cdots, k\}}_{\textrm{spectrum preservation constraints}}. 
\end{align}

where, $1 \leq r < \srank{\wmat}$ is the desired stable rank, $\lambda_i$s and $\sigma_i$s are the singular values of $\widehat{\wmat}_k$ and $\wmat$, respectively. The partitioning index $k$ is used for the {\em singular value (or the spectrum) preservation constraint}. It gives us the flexibility to obtain $\widehat{\wmat}_k$ such that its top $k$ singular values are exactly the same as that of the original matrix. Note, the problem statement is more general in the sense that putting $k=0$ removes the spectrum preservation constraint. 

\paragraph{The Solution to \gls{srn}}
The optimal unique solution to the above problem is provided
in~\cref{thm:srankOptimal} and proved in~\cref{sec:srankProof}. 
Note, at $k=0$, the problem~\cref{eq:srankProblem} is non-convex, otherwise convex.
\begin{restatable}{thm}{optimalthm}
\label{thm:srankOptimal}
Given a real matrix $\wmat\in\real^{m\times n}$ with rank $p$, a target spectrum (or singular value) preservation index $k$ $(0\le k < p)$, and a target stable rank of $r$ $(1 \leq r < \srank{\wmat})$, the optimal solution $\wmathw_k$ to problem~\cref{eq:srankProblem} is $\wmathw_k = \gamma_1\vec{S}_{1} + \gamma_{2} \vec{S}_{2}$, where  $\vec{S}_1 = \sum_{i=1}^{\mathrm{max}\br{1,k}} \sigma_i \bfu_i \bfv_i^\top$ and $\vec{S}_{2} = \wmat - \vec{S}_{1}$. $\{\sigma_i\}_{i=1}^k$, $\{\bfu_i\}_{i=1}^k$ and $\{\bfv_i\}_{i=1}^k$ are the top $k$ singular values and vectors of $\wmat$, and, depending on $k$, $\gamma_1$ and $\gamma_2$ are defined below. For simplicity, we first define $\gamma = \frac{\sqrt{r \sigma_1^2 - \norm{\vec{S}_1}_\forb^2}}{\norm{\vec{S}_2}_\forb}$, then
\begin{compactitem}
\item[a)]If $k=0$ (no spectrum preservation), the problem becomes
  non-convex, the optimal solution to which is obtained for $\gamma_2
  = \dfrac{\gamma + r - 1}{r}$ and $\gamma_1 =
  \dfrac{\gamma_2}{\gamma}$, when $r>1$. If $r=1$, then $\gamma_2 = 0$
  and $\gamma_1 = 1$. Since, in this case, $\norm{\vec{S}_1}_{\forb}^2 = \sigma_1^2$, \( \gamma = \frac{\sqrt{r-1} \sigma_1}{\norm{\vec{S}_2}_\forb}\).
\item[b)] If $k\ge 1$, the problem is convex. If $r\ge
  \frac{\norm{\vec{S}_1}_\forb^2}{\sigma_1^2}$ the optimal solution is
  obtained for $\gamma_1 = 1$,  and $\gamma_2=\gamma$ and if not, the
  problem is not feasible.
\item [c)] Also, $\norm{\wmathw_k - \wmat}_\forb$ monotonically increases with $k$ for $k \geq 1$.
\end{compactitem}
\end{restatable}
Intuitively,~\cref{thm:srankOptimal} partitions the given matrix into two parts, depending on $k$, and then scales them differently in order to obtain the optimal solution. The value of the partitioning index $k$ is a design choice. 
If there is no particular preference to $k$, then $k=0$ provides the
most optimal solution. We provide a simple example to better understand this.
Given $\wmat = \eye_3$ (rank =
$\srank{\vec{W}}$ = 3), the objective is to project it to a new matrix
with stable rank of $2$. Solutions to this problem are:
\begin{align}\small
\label{eq:srnExample}
\widehat{\wmat}_1=
  \begin{bmatrix}
    1 & 0 & 0\\
    0 & 1 & 0\\
    0 & 0 & 0
  \end{bmatrix}
  , \;
  \widehat{\wmat}_2=
  \begin{bmatrix}
    1 & 0 & 0\\
    0 & \frac{1}{\sqrt{2}} & 0\\
    0 & 0 & \frac{1}{\sqrt{2}}
  \end{bmatrix}
  , \;
  \widehat{\wmat}_3=
  \begin{bmatrix}
    \frac{\sqrt{2}+1}{2} & 0 & 0\\
    0 & \frac{\sqrt{2}+1}{2\sqrt{2}} & 0\\
    0 & 0 & \frac{\sqrt{2}+1}{2\sqrt{2}}
  \end{bmatrix}
\end{align}
Here, $\widehat{\wmat}_1$ is obtained using the standard rank minimization (Eckart-Young-Mirsky~\citep{Eckart1936}) while $\widehat{\wmat}_2$ and $\widehat{\wmat}_2$ are the solutions of~\cref{thm:srankOptimal} with $k=1$ and $k=0$, respectively. %
It is easy to verify that the stable rank of all the above solutions is $2$. However, the Frobenius distance (lower the better) of these solutions from the original matrix follows the order 
\(
\norm{\wmat - \widehat{\wmat}_1}_{\forb} > \norm{\wmat - \widehat{\wmat}_2}_{\forb} > \norm{\wmat - \widehat{\wmat}_3}_{\forb}
\).
As evident from the example, the solution to \gls{srn}, instead of completely removing a particular singular value, scales them (depending on $k$) such that the new matrix has the desired stable rank. Note that for $\widehat{\wmat}_1$ (true for any $k\geq1$), the spectral norm of the original and the normalized matrices are the same, implying, $\gamma_1 = 1$. However, for $k=0$, the spectral norm of the optimal solution is greater than that of the original matrix. It is easy to verify from~\cref{thm:srankOptimal} that as $k$ increases, $\gamma_2$ decreases. Thus, the amount of scaling required for the second partition $\vec{S}_2$ is more aggressive. In all situations, the following inequality holds: $\gamma_2 \leq 1 \leq \gamma_1$.

\begin{minipage}[t]{0.45\linewidth}
  \begin{algorithm}[H]
    \centering
    \caption{Stable Rank Normalization}
    \label{alg:stablerankNorm}
    \begin{algorithmic}[1]\footnotesize
      \Require $\wmat \in \real^{m \times n}$, $r$, $k \geq 1$
      \State $\vec{S}_1 \gets \mathbf{0}$, $\beta  \gets \norm{\wmat}_\forb^2$, $\eta \gets 0$, $l \gets 0$
      \For {$i \in \{1, \cdots, k\}$}
      \State $\{\bfu_i, \bfv_i, \sigma_i\} \gets SVD(\wmat, i)$\\
      \Comment{Power method to get $i$-th singular value}
      \If{$r \geq \br{\sigma_i^2 + \eta}/\sigma_1^2$}
      \label{eq:stableIf}
      \State $\vec{S}_1 \gets \vec{S}_1 + \sigma_i \bfu_i \bfv_i^{\top}$
      \label{eq:greedyStep}
      \State $\eta \gets  \eta + \sigma_i^2, \beta \gets \beta - \sigma_i^2$
      \State $l \gets  l+1$
      \Else
      \State break
      \EndIf
      \EndFor
      \State $\eta \gets r\sigma_1^2 - \eta$ \\
      \Return $\widehat{\wmat}_l \gets \vec{S}_1 +\sqrt{\frac{ \eta}{\beta}} (\wmat - \vec{S}_1)$, $l$
      \label{eq:returnStep}
    \end{algorithmic}
  \end{algorithm}
\end{minipage}\;\;\;\;\begin{minipage}[t]{0.45\linewidth}
  \begin{algorithm}[H]
    \centering
    \caption{\label{algo:spectralStable} SRN for a Linear Layer in NN}
    \label{alg:final}
    \begin{algorithmic}[1]\footnotesize
      \Require $\wmat \in \real^{m \times n}$, $r$, learning rate $\alpha$, mini-batch dataset $\mathcal{D}$
      \State Initialize $\bfu \in \real^m$ with a random vector.%
      \State $\bfv \gets \frac{\wmat^\top \bfu}{\norm{\wmat^\top \bfu}}$, $\bfu \gets \frac{\wmat^\top \bfv}{\norm{\wmat^\top \bfv}}$\\
      \Comment{Perform power iteration}
      \State  $\sigma(\wmat)= \bfu^{\top} \wmat \bfv$
      \State  $\wmat_f = \wmat/ \sigma(\wmat)$
      \Comment {Spectral Normalization}
      \State $\widehat{\wmat} = \wmat_f - \bfu \bfv^{\top}$
      \If {$\norm{\widehat{\wmat}}_{\forb} \le \sqrt{r-1}$}
      \State \Return $\wmat_f$ 
      \EndIf
      \State $\wmat_f = \bfu \bfv^{\top} + \widehat{\wmat} \frac{ \sqrt{r- 1} }{ \norm{\widehat{\wmat}}_{\forb} }$
      \Comment {Stable Rank Normalization}
      \State \Return $\wmat \leftarrow \wmat - \alpha \nabla_{\wmat} L(\wmat_f, \mathcal{D})$
    \end{algorithmic}  
  \end{algorithm}
\end{minipage}

\paragraph{Algorithm for Stable Rank Normalization} We provide a general procedure in~\cref{alg:stablerankNorm} to solve the stable rank normalization problem for $k \geq 1$ (the solution for $k=0$ is straightforward from~\cref{thm:srankOptimal}).~\cref{claim:stableRankAlgo} provides the properties of the algorithm. The algorithm is constructed so that the prior knowledge of the rank of the matrix is not necessary. %
\begin{claim} 
\label{claim:stableRankAlgo}
Given a matrix $\wmat$, the desired stable rank $r$, and the partitioning index $k\ge 1$,~\cref{alg:stablerankNorm} requires computing the top $l$ $(l \leq k)$ singular values and vectors of $\wmat$. It returns $\widehat{\wmat}_l$ and the scalar $l$ such that $\srank{\widehat{\wmat}_l} = r$, and the top $l$ singular values of $\wmat$ and $\widehat{\wmat}_l$ are the same. If $l=k$, then the solution provided is the optimal solution to the problem~\eqref{eq:srankProblem} with all the constraints satisfied, otherwise, it returns the largest $l$ up to which the spectrum is preserved. %
\end{claim}
\paragraph{Combining Stable Rank and Spectral Normalization for \gls{nn}s}
Following the arguments provided in~\cref{sec:intro,sec:whyStable}, for better generalizability, we propose to 
normalize {\em both} the stable rank and the spectral norm of each linear layer of a \gls{nn} simultaneously.
To do so, we first perform approximate \gls{sn}~\citep{miyato2018spectral}, and then perform
optimal \gls{srn} (using \cref{alg:stablerankNorm}).
We use $k=1$ to ensure that the first singular value (which is now
normalized) is preserved.~\cref{algo:spectralStable} provides a
simplified procedure for the same for a given linear layer of a
\gls{nn}. Note, the computational cost of this algorithm is
  {\em exactly the same as that of~\gls{sn}}, which is to compute the top singular value using the power iteration method.

\section{Experiments}
\label{sec:experiments}
\paragraph{Dataset and Architectures} For classification, we perform experiments on ResNet-110~\citep{HZRS:2016}, WideResNet-28-10~\citep{Zagoruyko2016}, DenseNet-100~\citep{Huang2017},
VGG-19~\citep{simonyan2014very}, and
AlexNet~\citep{krizhevsky2009learning} using the
CIFAR100~\citep{krizhevsky2009learning}
dataset. We present further
experiments with CIFAR10 in~\cref{sec:gener-behav} in~\cref{fig:test-acc-epoch-c10,fig:cifar-10-gen}.
We train them using standard
training recipes with SGD, using a learning rate of $0.1$~(except
AlexNet where we use a learning rate of $0.01$), and a momentum of
$0.9$ with a batch size of $128$ (further details in
Appendix~\ref{sec:experimental-details}). In addition to training for a fixed
number of epochs, we also present results in the Appendix in ~\cref{fig:test-acc-stop,fig:test-acc-stop-c10} where the
training accuracy~(as opposed to number of iterations) is used as a stopping
criterion to show that
our regularizor performs well with a range of stopping criterions.

For GAN experiments, we use CIFAR100, CIFAR10, and CelebA~\citep{liu2015faceattributes} datasets. We show results on both, conditional and unconditional
GANs. Please refer to~\cref{sec:gansetup} for further
details about the training setup.

\paragraph{Choosing stable rank} Given a matrix $\wmat \in \real^{m \times
  n}$, the desired stable rank $r$ is controlled using a single
hyperparameter $c$ as $r = c \; \min(m,n)$, where $c \in (0, 1]$.
For simplicity, we use the same $c$ for all the linear layers. 
It is trivial to note that if $c=1$, or for a given $c$, if $\srank{\wmat} \leq r$,
then SRN boils down to SN.
For classification, we choose $c = \{0.3, 0.5\}$, and compare SRN against standard
training (Vanilla) and training with \gls{sn}. For GAN experiments, we choose $c = \{0.1, 0.3, 0.5, 0.7, 0.9\}$, and
compare SRN-GAN against SN-GAN~\citep{miyato2018spectral}, WGAN-GP~\citep{Gulrajani2017}, and orthonormal regularization GAN
(Ortho-GAN)~\citep{Brock2016}.

\paragraph{Result Overview}
\begin{compactitem}
\item SRN improves classification accuracy on a wide variety of architectures. 
\item Normalizing stable rank improves the learning capacity of spectrally normalized networks.
\item SRN shows remarkably less memorization, even on settings very hard to generalize.
\item SRN shows much improved generalization behaviour evaluated using recently proposed sample complexity measures.
\item As we decrease the stable rank, the empirical \Gls{lip} of SRN-GAN decreases. Proving our arguments provided in~\cref{sec:whyStable}.
\item SRN-GAN provides much improved Neural divergence score (ND)~\citep{gulrajani2018towards} compared to SN-GAN, proving that it is robust to memorization in GANs as well.
\item SRN-GAN also provide improved Inception and FID scores in all our experiments (except one where SN-GAN is better).
\end{compactitem}

\begin{figure}[t]
  \centering\small
  \begin{subfigure}[!t]{0.19\linewidth}
    \def\svgwidth{0.98\linewidth}
    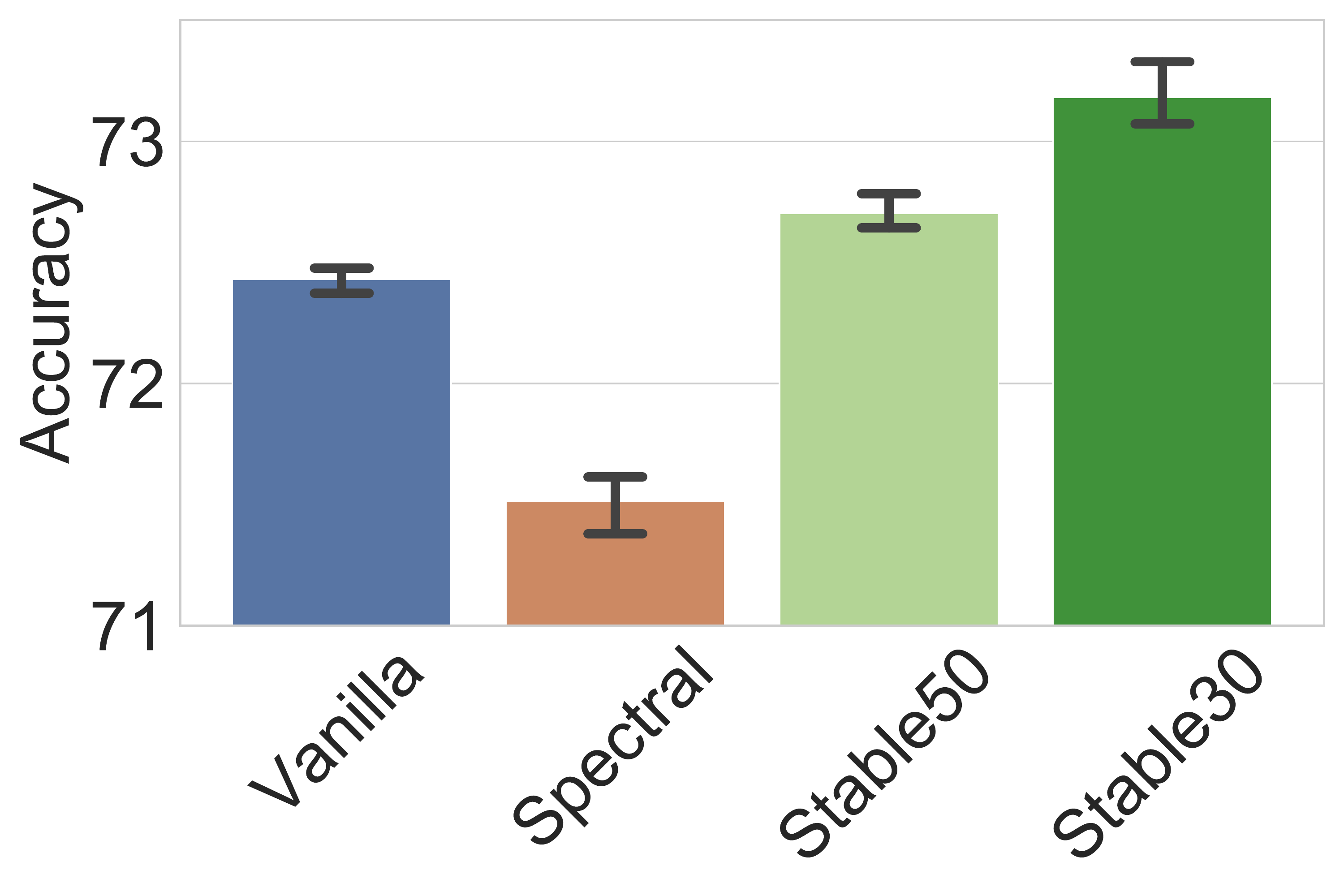
    \subcaption{ Resnet110}
  \end{subfigure}
  \begin{subfigure}[!t]{0.19\linewidth}
    \def\svgwidth{0.98\linewidth}
    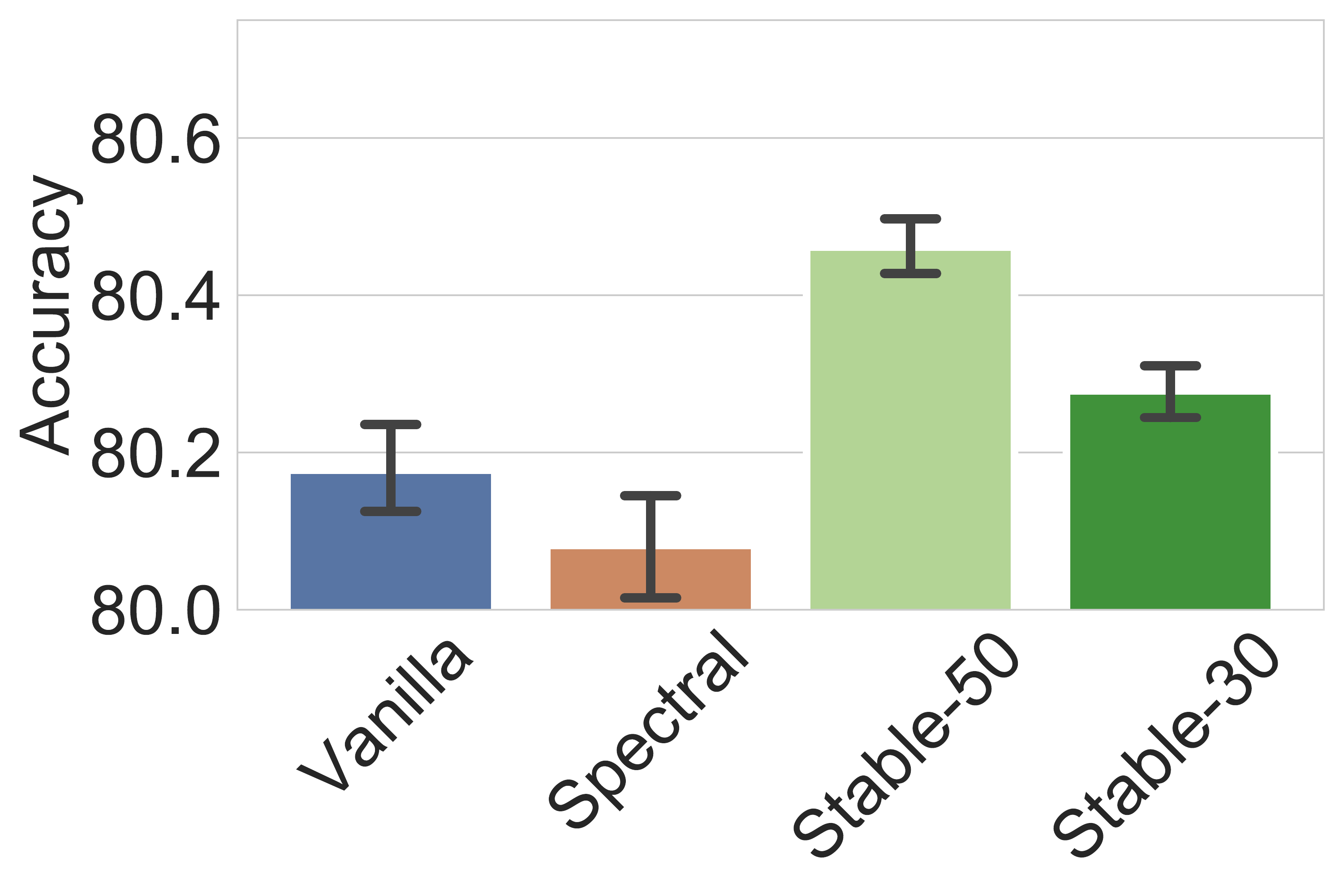
    \subcaption{ WideResnet-28}
  \end{subfigure}
   \begin{subfigure}[!t]{0.19\linewidth}
    \def\svgwidth{0.98\linewidth}
    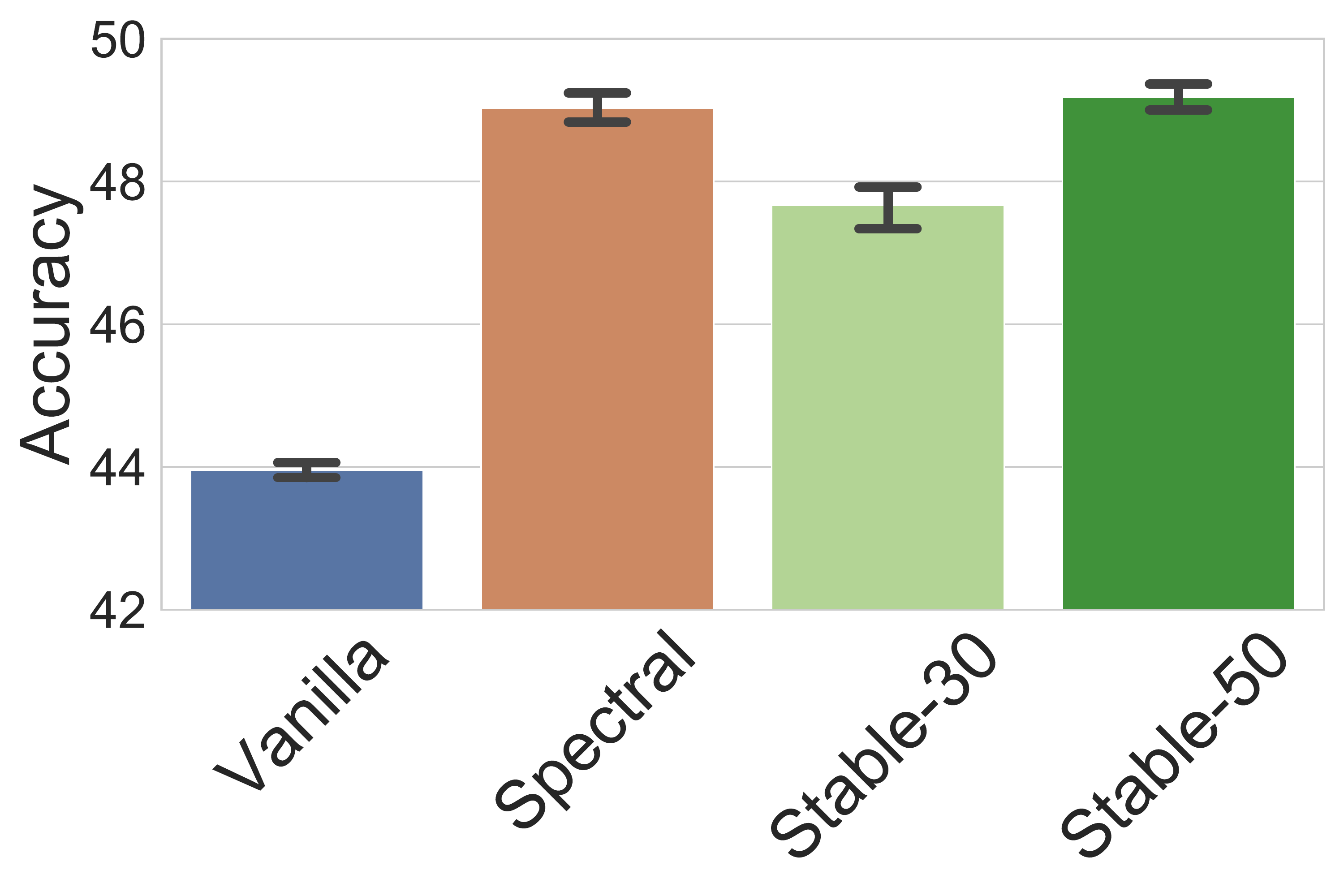
    \subcaption{ Alexnet}
  \end{subfigure}
  \begin{subfigure}[!t]{0.19\linewidth}
    \def\svgwidth{0.98\linewidth}
    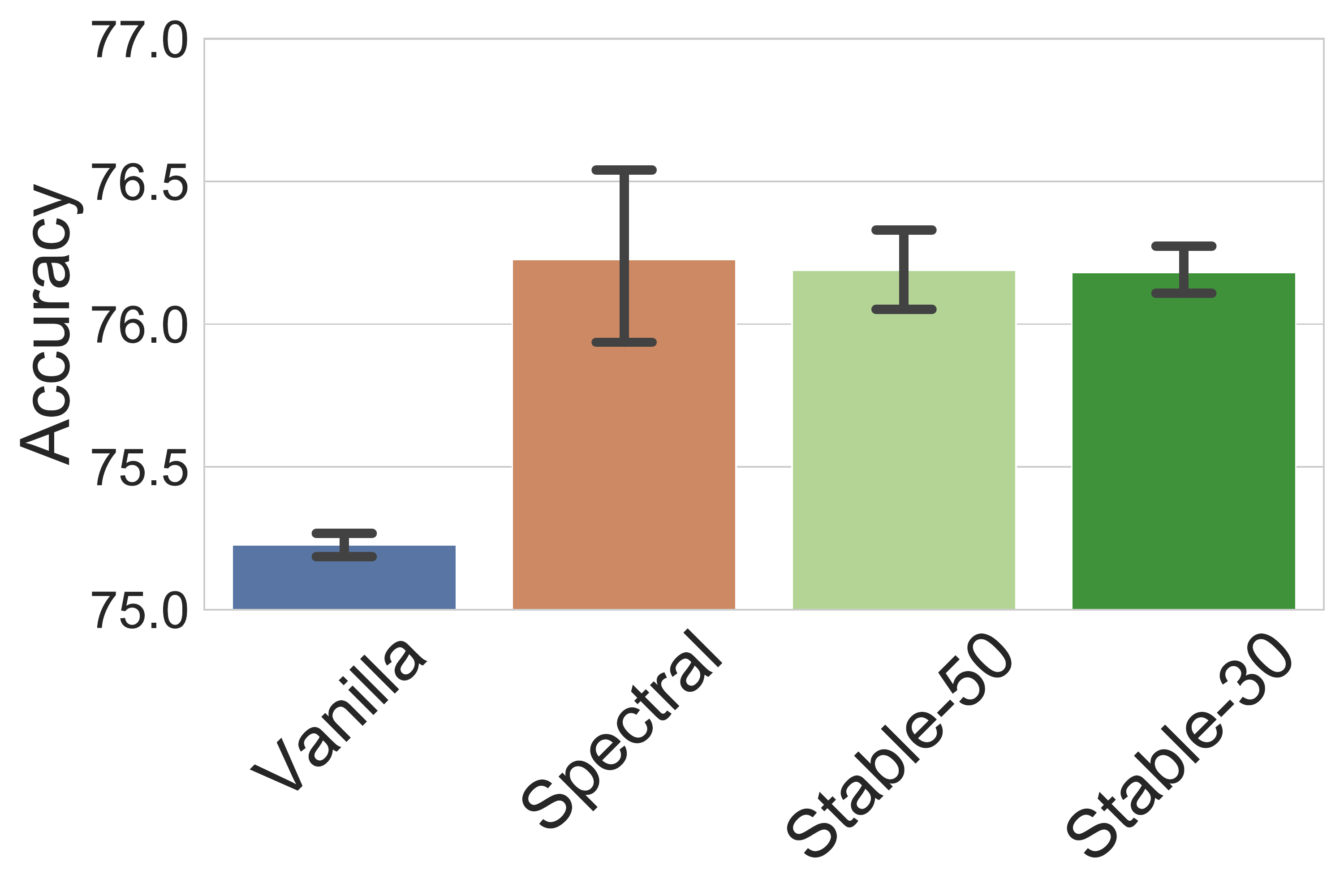
    \subcaption{ Densenet-100}
  \end{subfigure}
   \begin{subfigure}[!t]{0.19\linewidth}
    \def\svgwidth{0.98\linewidth}
    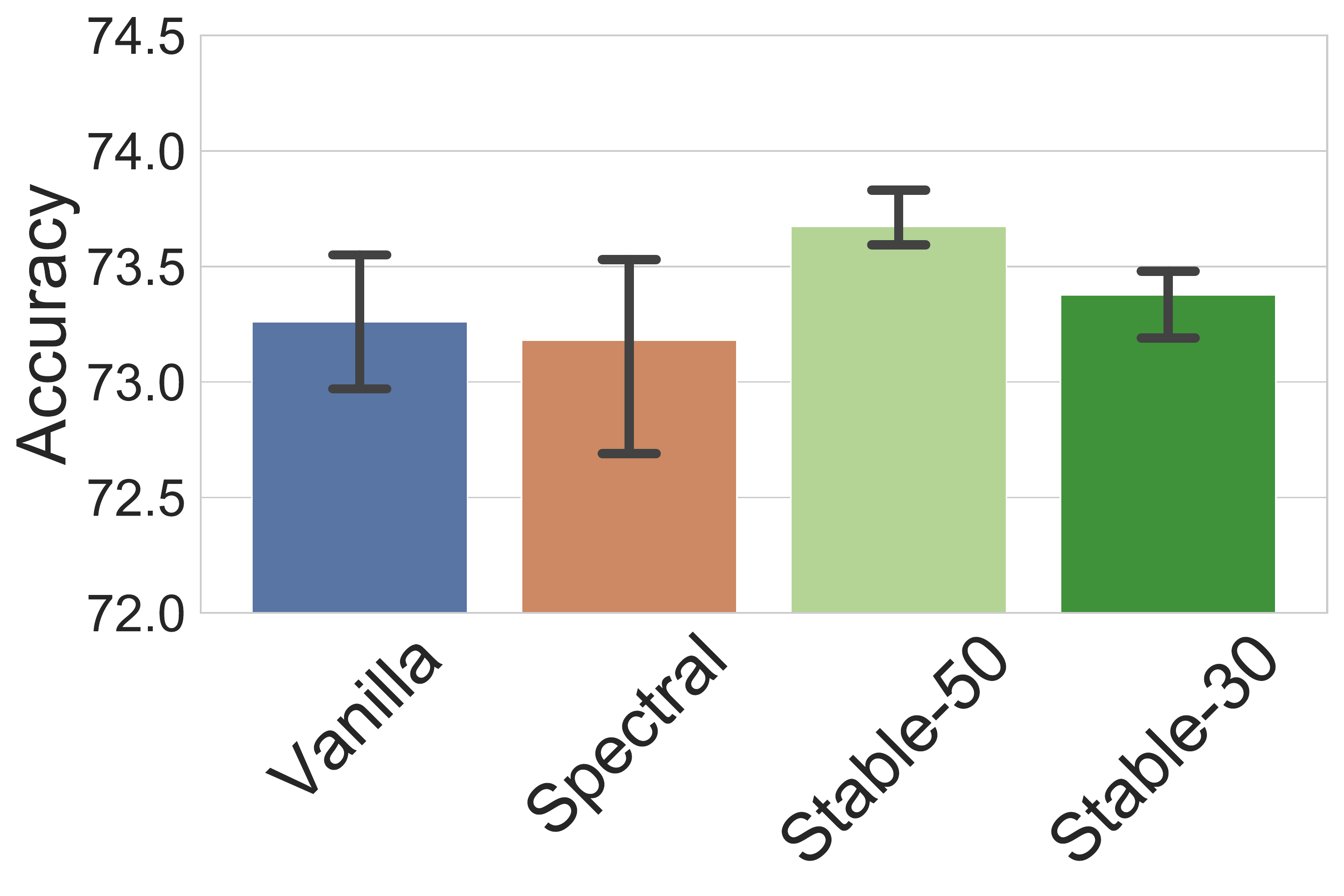
    \subcaption{ VGG-19}
  \end{subfigure}
  \caption{Test accuracies on CIFAR100 for clean data. Higher is better.}
  \label{fig:test-acc}
\end{figure}

\subsection{Classification Experiments}\label{sec:gen_exp_results}
We perform each experiment $5$ times using a new random seed each time and report the mean, 
and the $75\%$ confidence interval for the test error in~\cref{fig:test-acc}. 
These experiments show that the test accuracy of \gls{srn}, on a wide variety \gls{nn}s, 
is always higher than the Vanilla and SN (except for SRN-50  on Alexnet where \gls{srn} and \gls{sn} are almost equal). 
However, \gls{sn} performs slightly worse than Vanilla for WideResNet-28 and ResNet110. 
The fact that \gls{srn} does involve \gls{sn}, combined with the above statement, indicate that
even though \gls{sn} reduced the learning capability of these networks, 
normalizing stable rank must have improved it significantly in order for \gls{srn} to outperform Vanilla. For example,
in the case of ResNet110, \gls{sn} is $71.5\%$ accurate whereas \gls{srn} provides an accuracy of $73.2\%$. In addition to this,
we would like to note that even though \gls{sn} is being used extensively for the training of GANs, it is not
a popular choice when it comes to training standard \gls{nn}s for classification. We suspect that this is because of the decrease in the capacity, which 
we have shown to be increased by the stable rank normalization, proving the worth of \gls{srn} for classification tasks as well.
\subsection{Study of Generalization Behaviour}
\label{sec:empir-eval-gener}
Our last set of experiments established that \gls{srn} provides improved classification accuracies on various \gls{nn}s.
Here we study the generalization behaviour of these models. Quantifying generalization behaviour is non-trivial
and there is no clear answer to it. However, we utilize recent efforts that explore the theoretical understanding of 
generalization and use them to study it in practice.
\begin{figure}%
  \centering
  \begin{subfigure}[!t]{0.19\linewidth}
    \def\svgwidth{0.98\linewidth}
    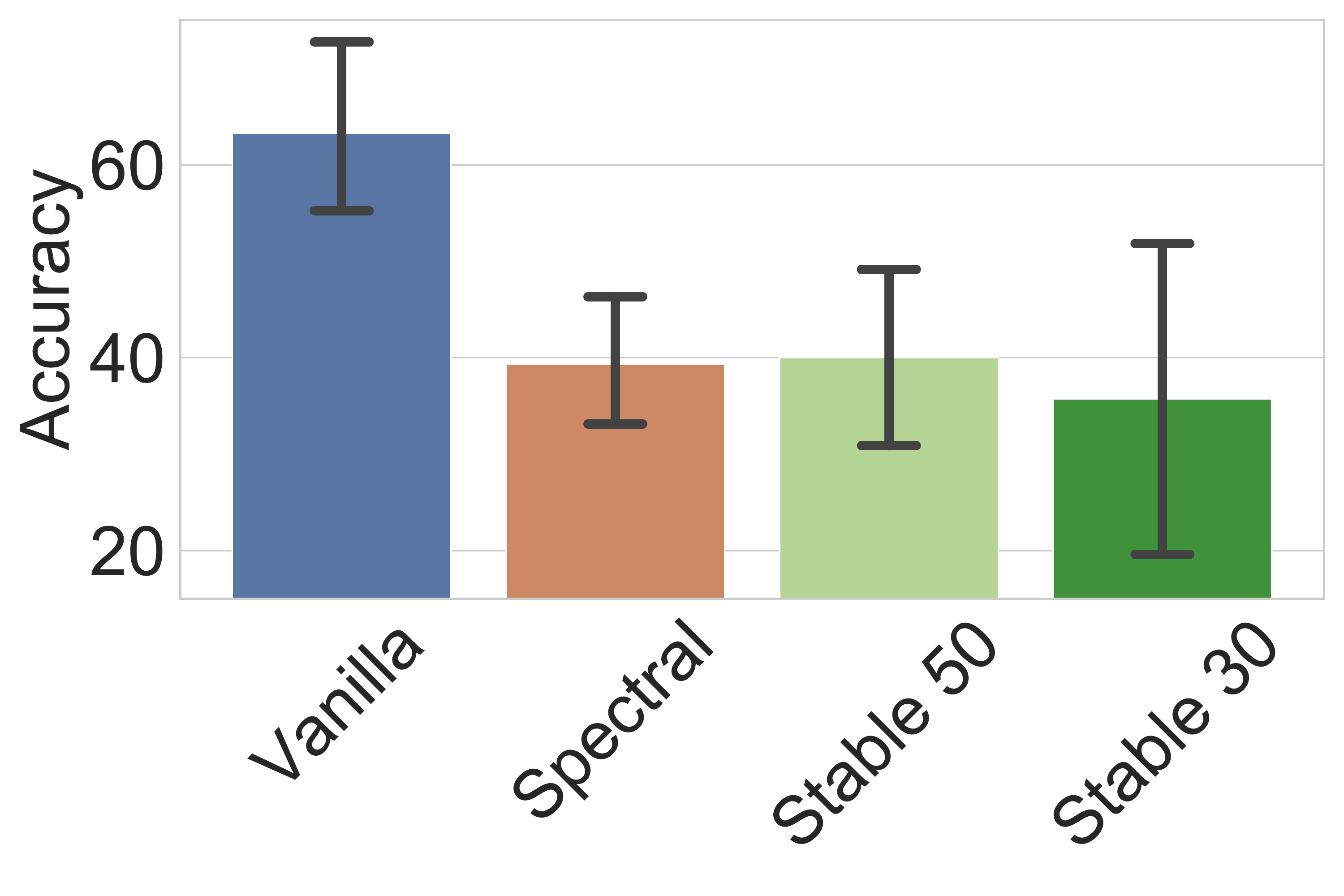
    \subcaption{Resnet110}
  \end{subfigure}
     \begin{subfigure}[!t]{0.19\linewidth}
    \def\svgwidth{0.98\linewidth}
    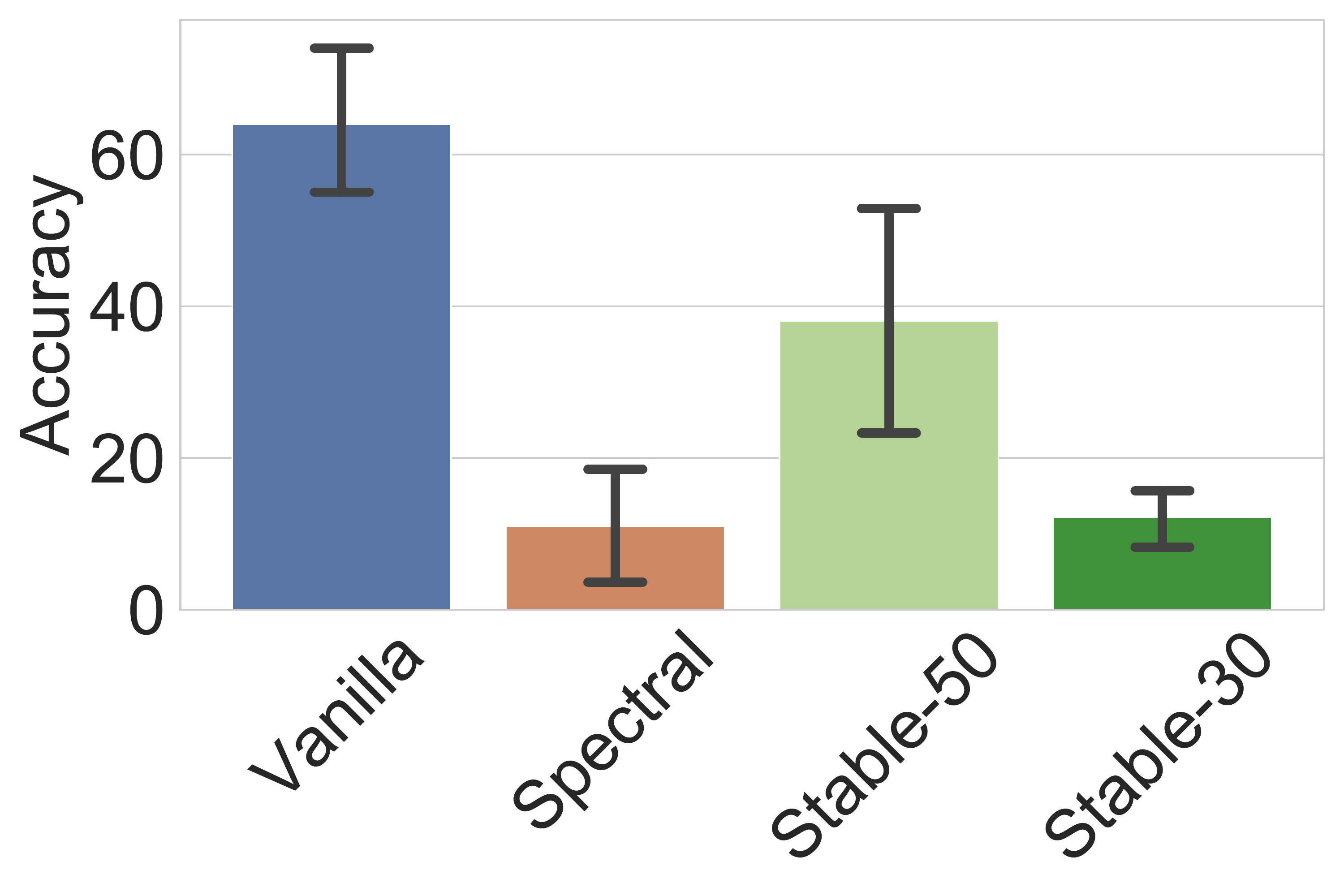
    \subcaption{WideResNet-28}
  \end{subfigure}
   \begin{subfigure}[!t]{0.19\linewidth}
    \def\svgwidth{0.98\linewidth}
    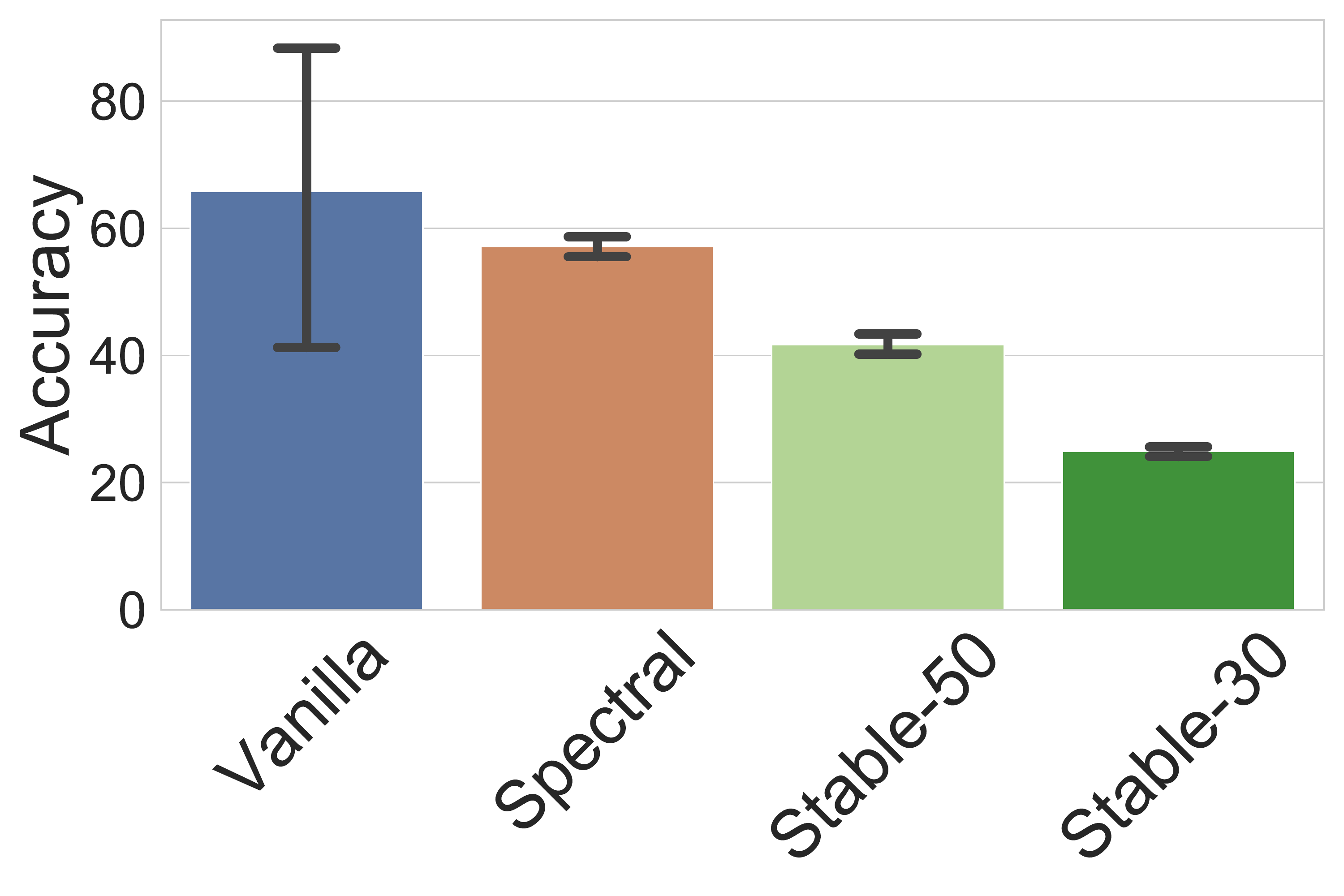
    \subcaption{Alexnet}
  \end{subfigure}
   \begin{subfigure}[!t]{0.19\linewidth}
    \def\svgwidth{0.98\linewidth}
    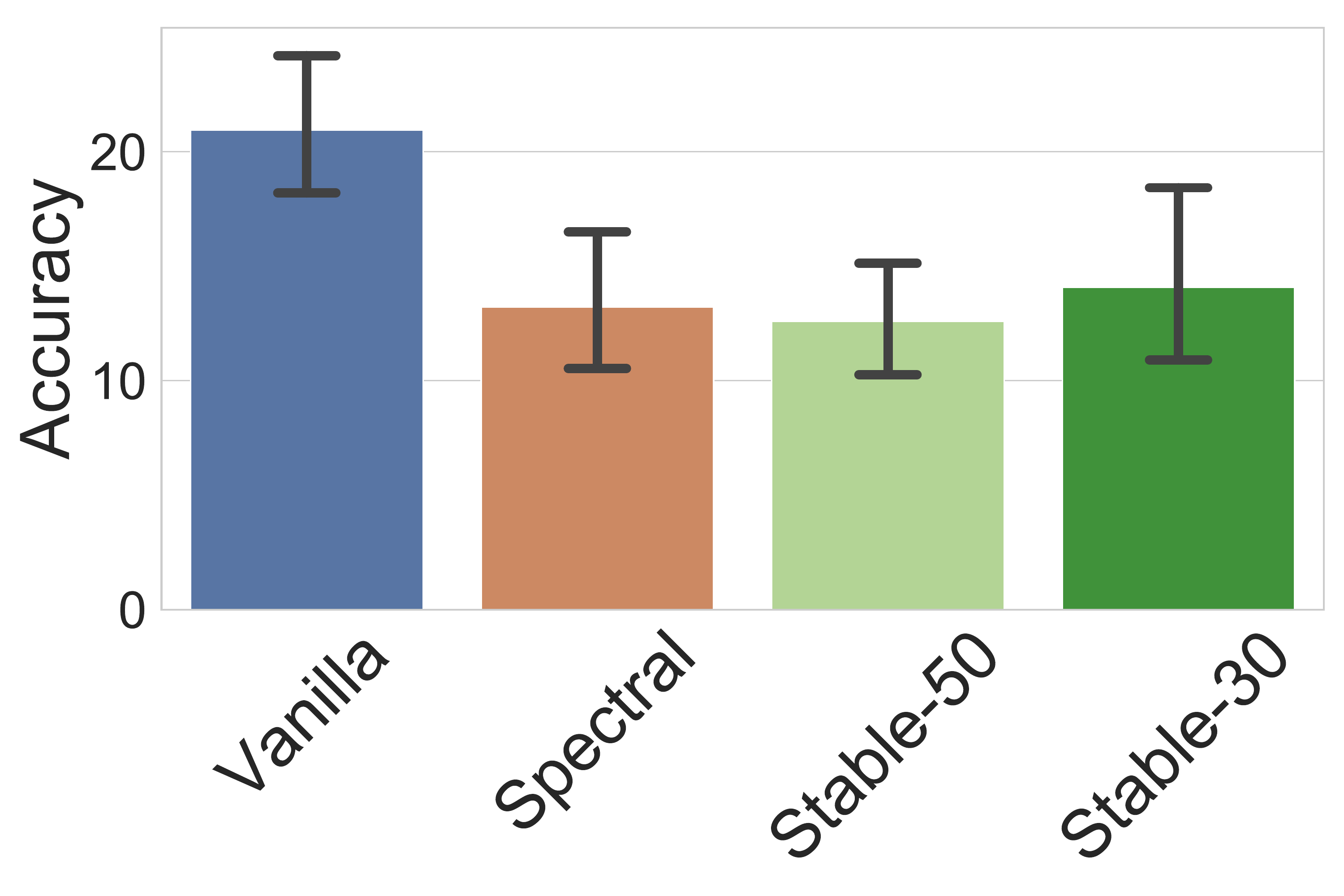
    \subcaption{Densenet-100}
  \end{subfigure}
   \begin{subfigure}[!t]{0.19\linewidth}
    \def\svgwidth{0.98\linewidth}
    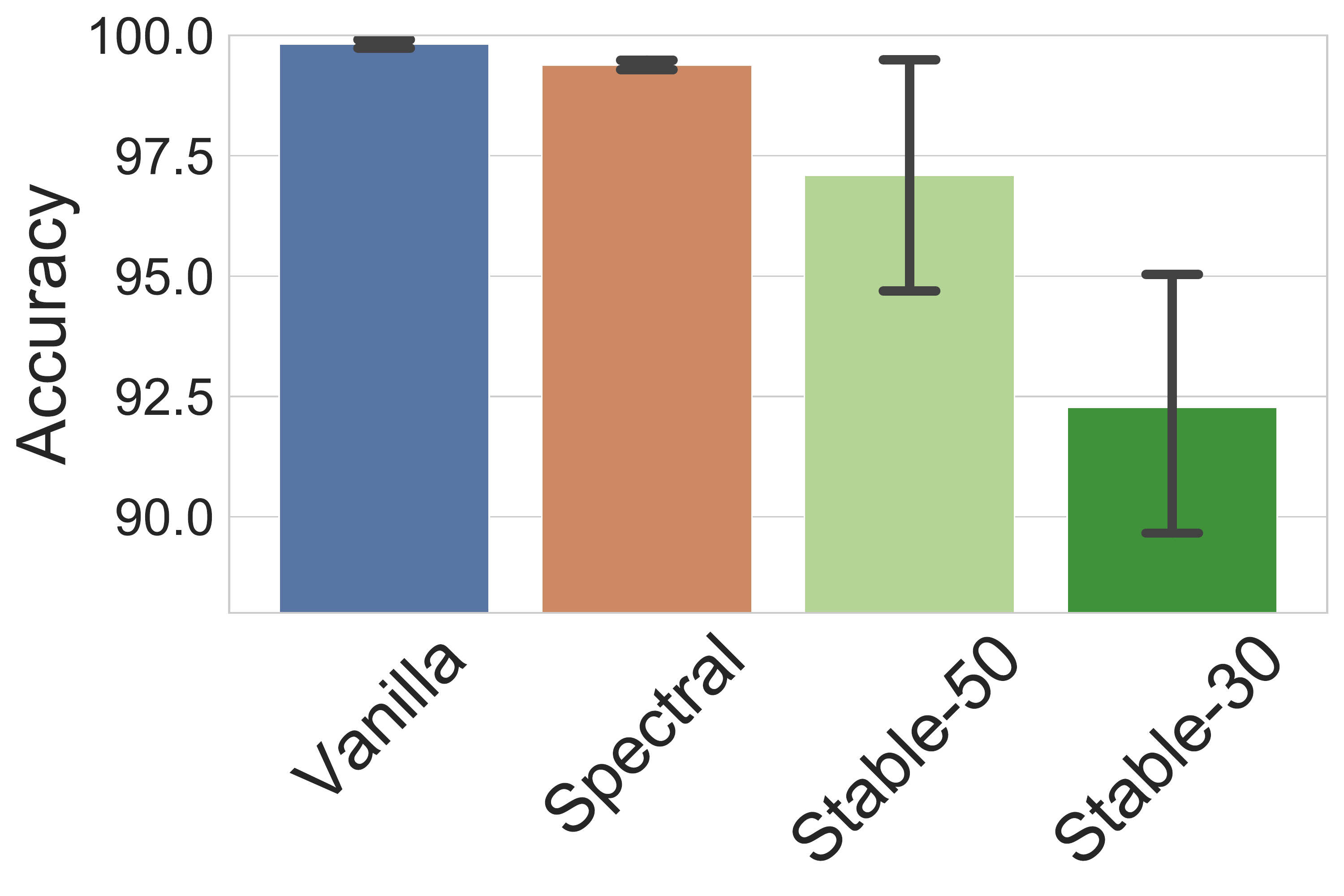
    \subcaption{VGG-19}
  \end{subfigure}
  \caption{Train accuracies on CIFAR100 for shattering experiment. Lower indicate less memorization, thus, better.}
  \label{fig:rand-acc}
\end{figure}

\paragraph{Shattering Experiments}
To inspect the generalization behaviour in \gls{nn}s we begin with the shattering experiment~\citep{Zhang2016}. 
It is a test of whether the network
can fit the training data well but not a label-randomized version of it (each image
of the dataset is assigned a random label). As
there is no correlation of the labels with the data points
$P\br{y\vert \vec{x}}$ is essentially uninformative because it is uniformly random. Thus, the test accuracy on this task is
almost $1\%$. A high training accuracy --- which indicates a high
generalization gap (difference between train and test accuracy) can be achieved only by memorizing the train data
~\footnote{The training
    of all the models of one architecture were stopped  after the same number of epochs
    - double the number of epochs the model were trained on the clean dataset.}.
~\cref{fig:rand-acc}
shows that \gls{srn} reduces memorization on random labels~(thus, reduces the estimate of the Rademacher
complexity~\citep{Zhang2016}). 
Note, as shown in the classification
  experiments, the same model was able to achieve the highest training
  accuracy when the labels were not randomized.

\begin{table}[!htb]
  \centering\footnotesize
  \begin{tabular}{l@{\quad}c@{\quad}c@{\quad}c@{\quad}c@{\quad}}\toprule
    &SRN-50&SRN-30&Spectral (SN) &Vanilla\\\midrule
    WD & $12.02 \pm 1.77$ &$11.87 \pm 0.57$ & $11.13 \pm 2.56$ & $10.56
    \pm 2.32$ \\
    w/o WD & $17.71 \pm 2.30$ &$19.04 \pm 4.53$ & $17.22 \pm 1.94$ & $13.49
    \pm 1.93$ \\\bottomrule
  \end{tabular}
  \caption{{\bf Highly non-generalizable setting}. Training error for ResNet-110 on CIFAR100 with randomized
    labels, low lr$=0.01$, and with and without weight decay.~(Higher
    is better.) The clean test accuracy for this setting is shown in~\cref{tab:clean-c100-lowlr}.}
  \label{tab:low_lr_rnet}
\end{table}

We also look specifically at  highly non-generalizable settings --- {\em low
learning rate and without weight decay}. As shown in~\cref{tab:low_lr_rnet}, \gls{srn} consistently achieves
lower generalization error~(by achieving a low train error) both in
the presence and the absence of weight decay
\footnote{These result are
reported after 200 epochs. It can be looked on as combined with early
stopping, which is a powerful way of avoiding memorizing random
labels~\citep{Li2019}.}.
Similar results are
reported for Alexnet and WideResNet in~\cref{sec:gener-behav}. 

\begin{figure}[t]
  \centering
  \begin{subfigure}[t]{1.0\linewidth}
    \centering
  \begin{subfigure}[t]{0.32\linewidth}
    \def\svgwidth{0.98\linewidth} 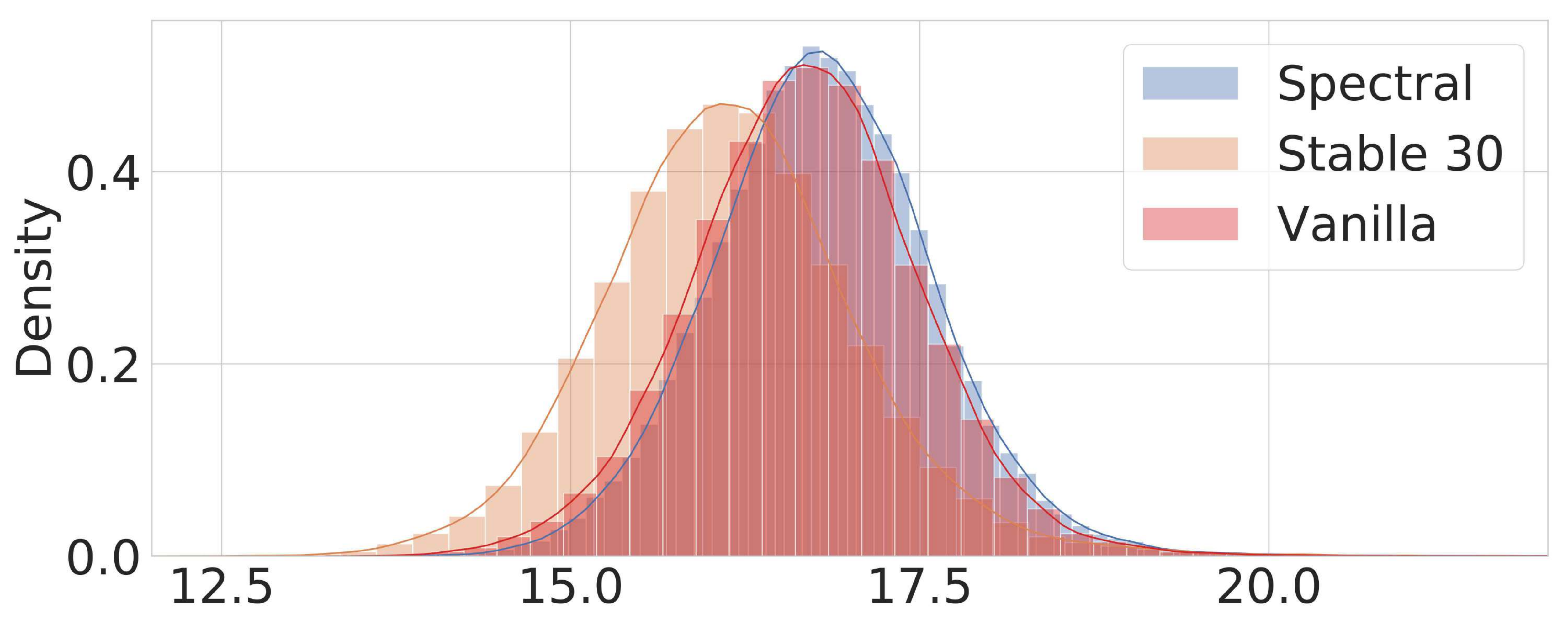
    \subcaption{R110-Jac-Norm}\label{fig:r110-jac-comp}
  \end{subfigure}\begin{subfigure}[t]{0.32\linewidth}
    \def\svgwidth{0.98\linewidth} 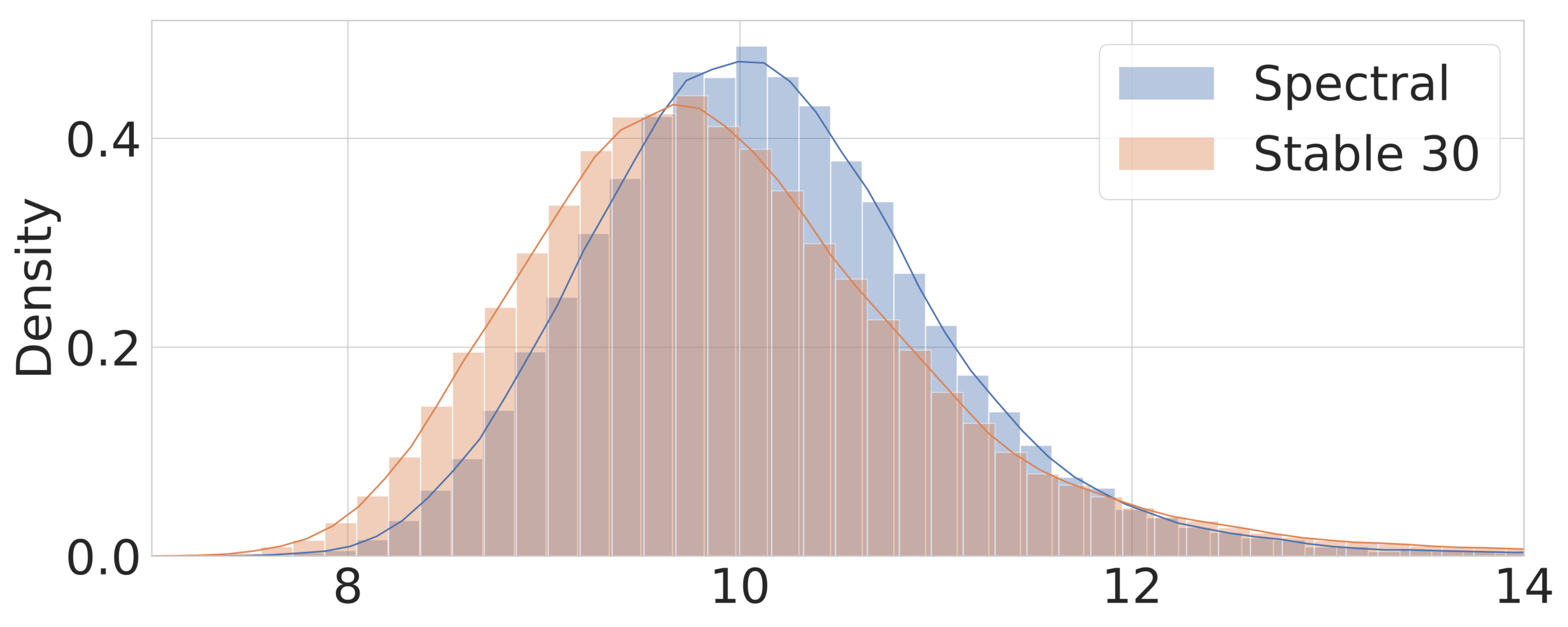
    \subcaption{R110-Spec-$L_1$}\label{fig:r110-spec-l1-comp}
  \end{subfigure}
  \begin{subfigure}[t]{0.32\linewidth}
    \def\svgwidth{0.98\textwidth}
    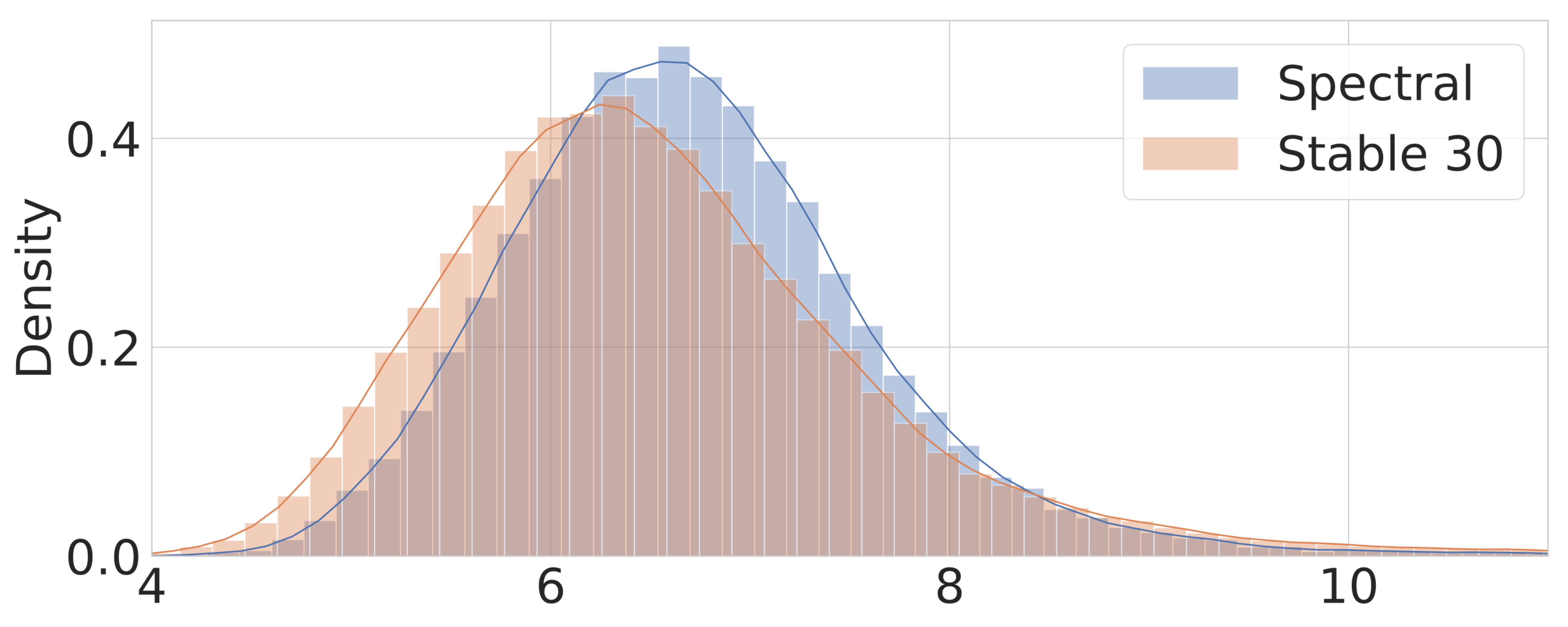
    \subcaption{R110-Spec-Fro}\label{fig:r110-spec-fro-comp}
  \end{subfigure}
\end{subfigure}
\begin{subfigure}[t]{1.0\linewidth}
    \centering
  \begin{subfigure}[t]{0.32\linewidth}
    \def\svgwidth{0.98\linewidth} 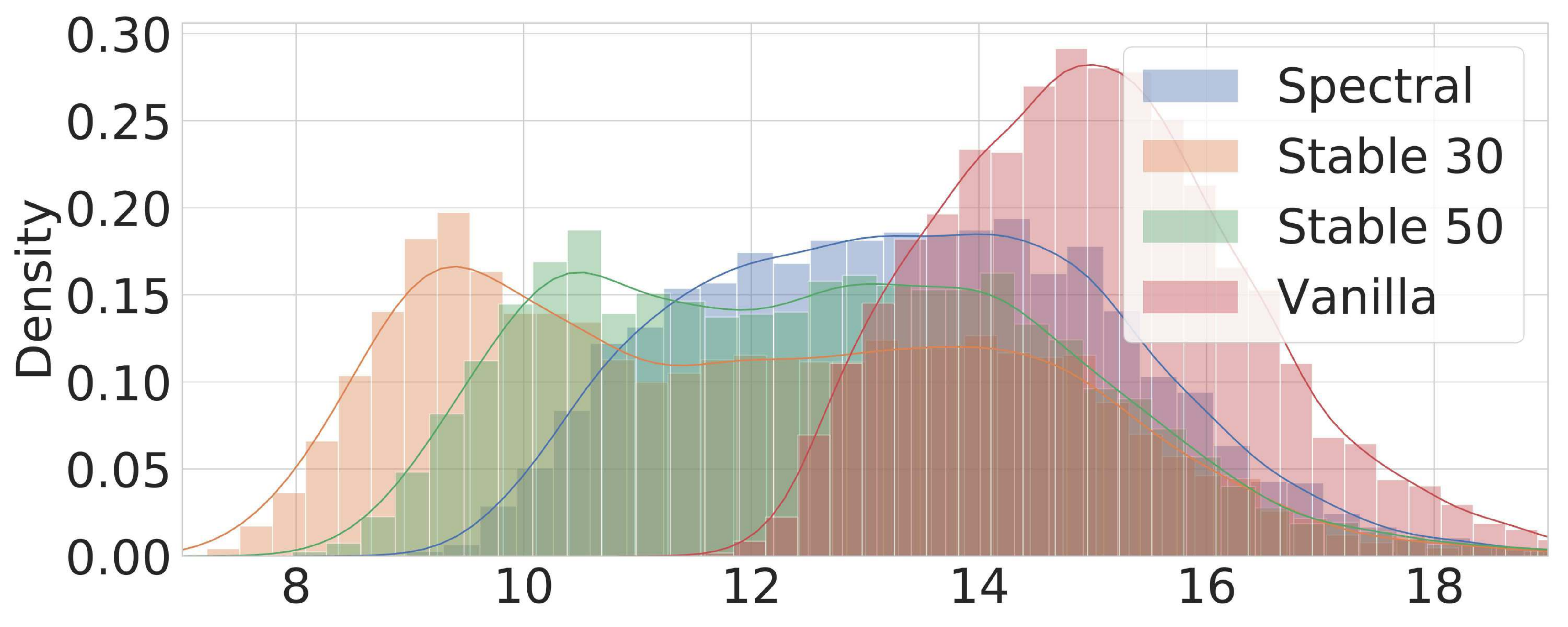
    \subcaption{WRN-Jac-Norm}\label{fig:wrn-jac-comp}
  \end{subfigure}
  \begin{subfigure}[t]{0.32\linewidth}
    \def\svgwidth{0.98\linewidth} 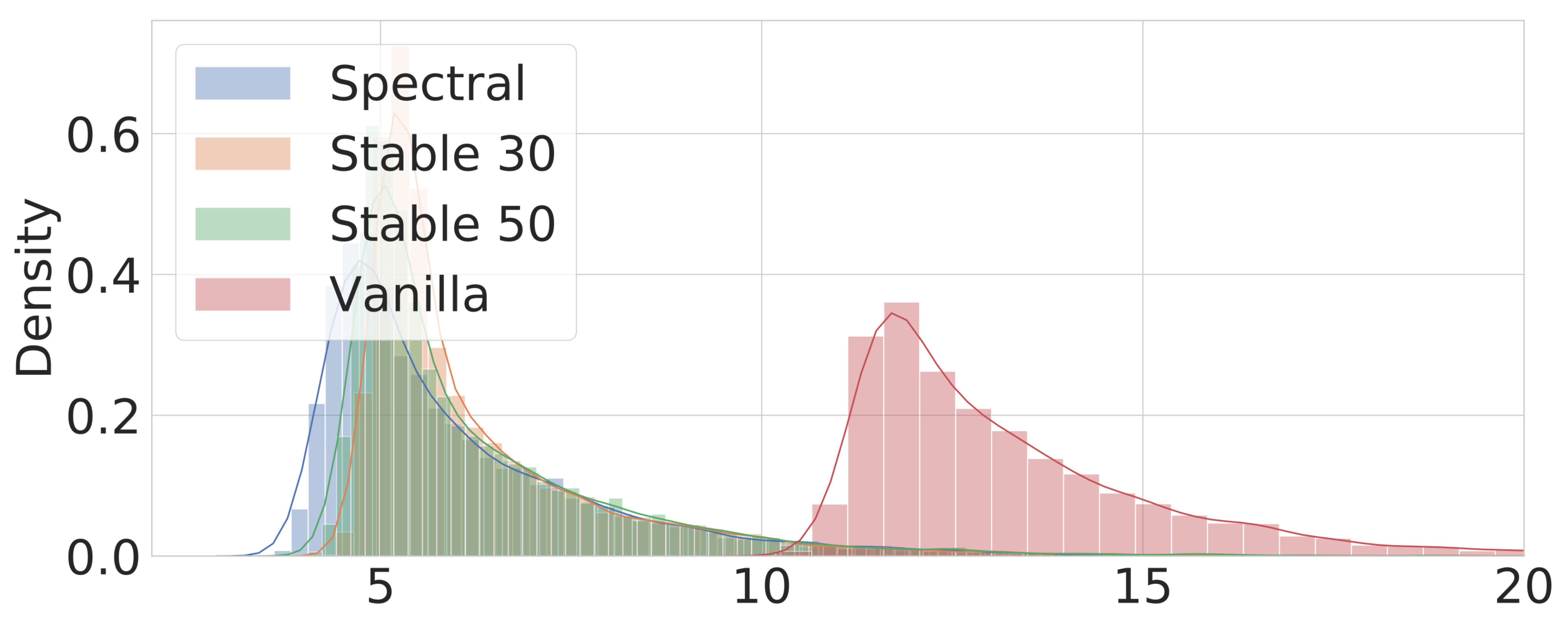
    \subcaption{WRN-Spec-$L_1$}\label{fig:wrn-spec-l1-comp}
  \end{subfigure}
  \begin{subfigure}[t]{0.32\linewidth}
    \def\svgwidth{0.98\textwidth}
    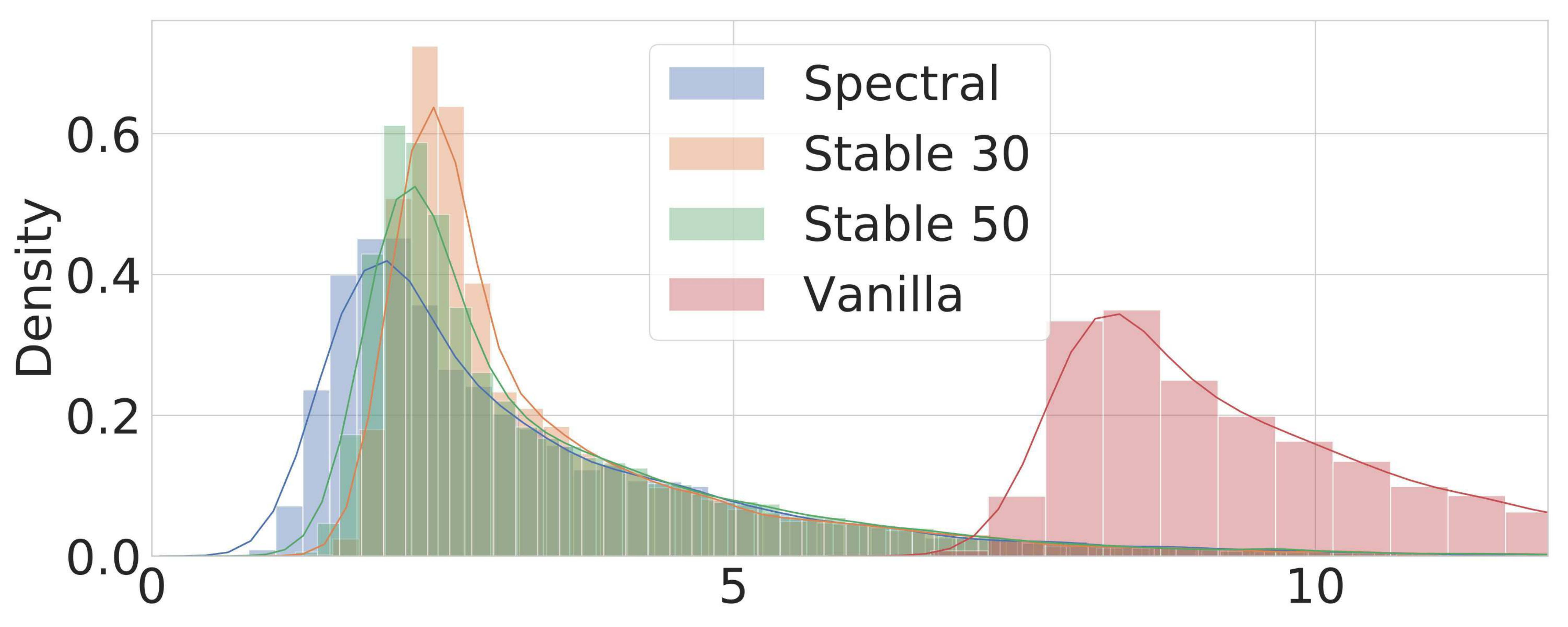
    \subcaption{WRN-Spec-Fro}\label{fig:wrn-spec-fro-comp}
  \end{subfigure}
\end{subfigure}
\begin{subfigure}[t]{1.0\linewidth}
    \centering
  \begin{subfigure}[t]{0.32\linewidth}
    \def\svgwidth{0.98\linewidth} 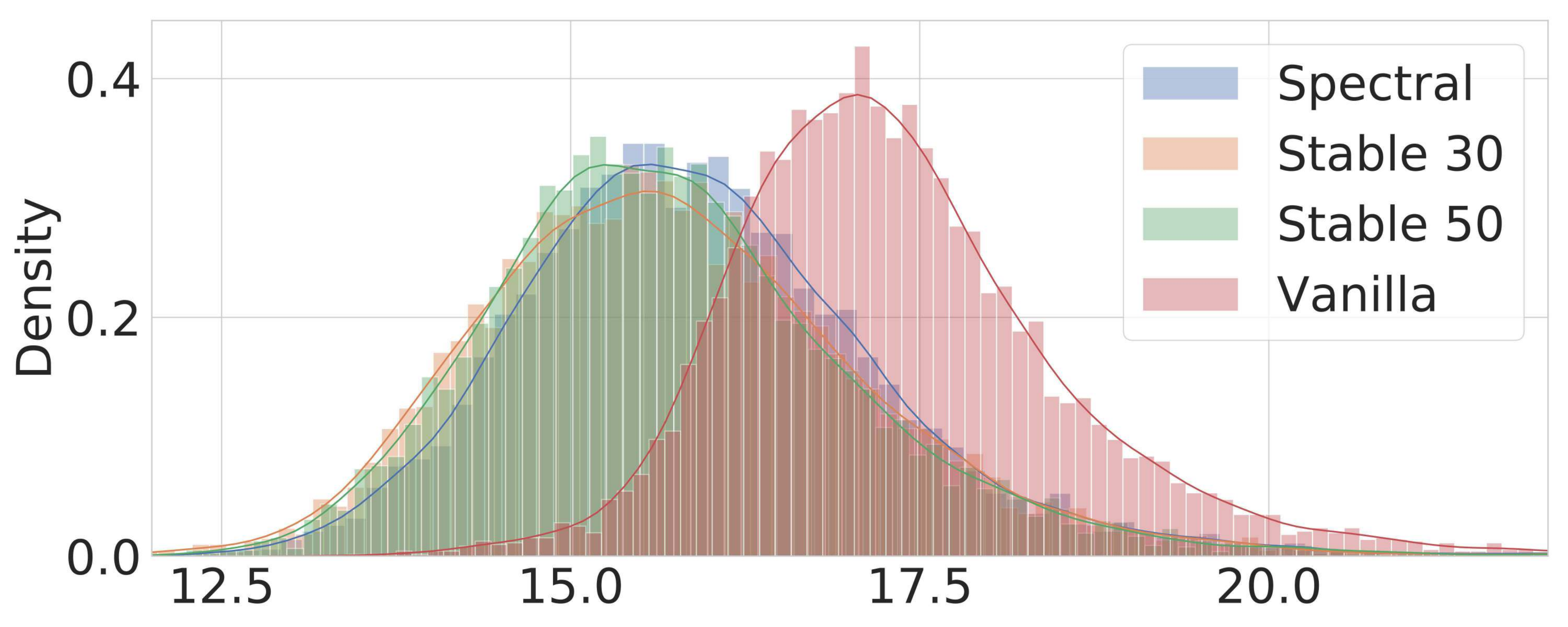
    \subcaption{D100-Jac-Norm}\label{fig:d100-jac-comp}
  \end{subfigure}
  \begin{subfigure}[t]{0.32\linewidth}
    \def\svgwidth{0.98\linewidth} 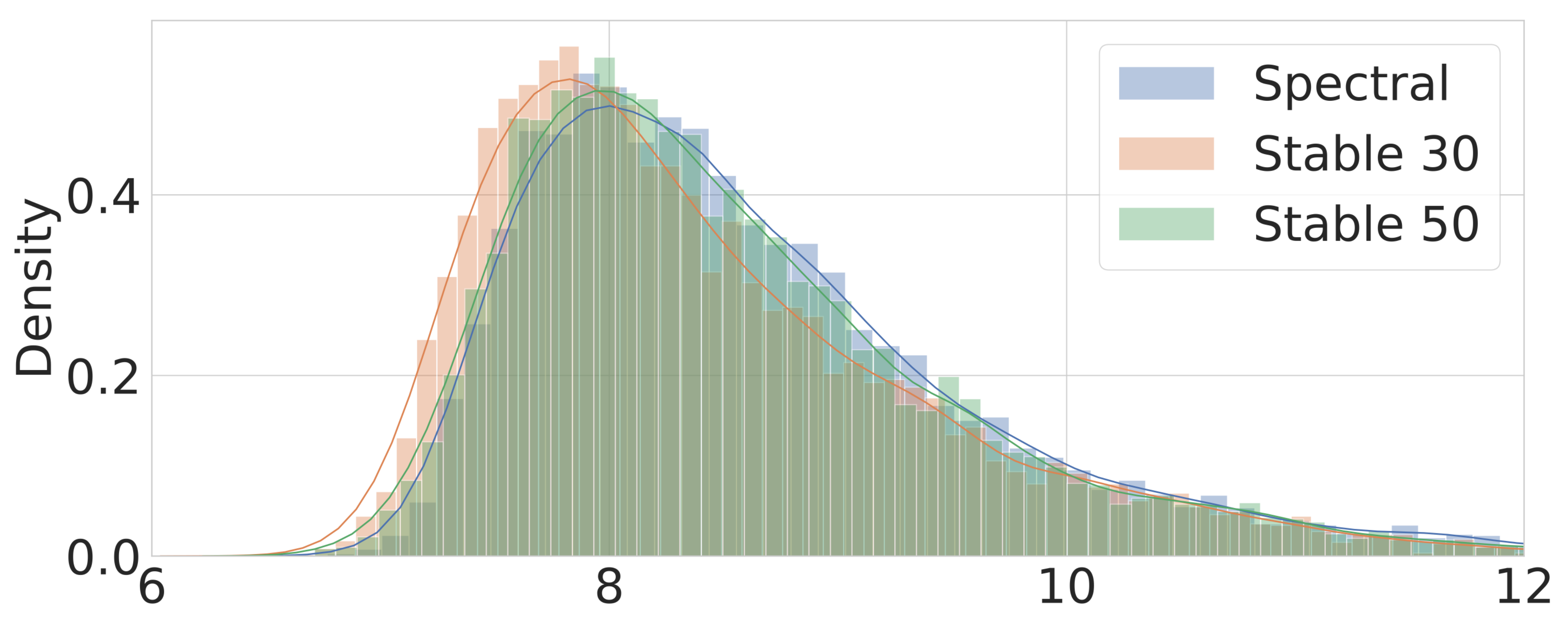
    \subcaption{D100-Spec-$L_1$}\label{fig:d100-spec-l1-comp}
  \end{subfigure}
  \begin{subfigure}[t]{0.32\linewidth}
    \def\svgwidth{0.98\textwidth}
    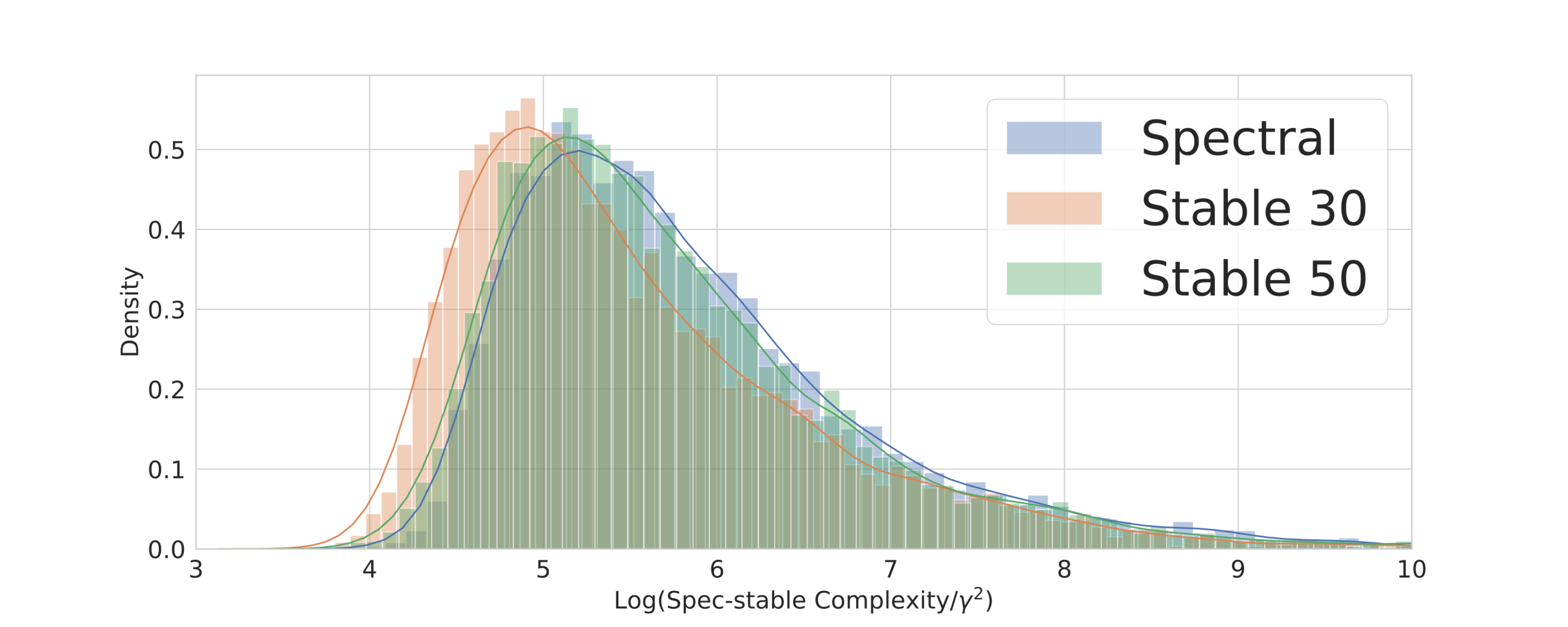
    \subcaption{D100-Spec-Fro}\label{fig:d100-spec-fro-comp}
  \end{subfigure}
\end{subfigure}
  \caption{($\log$) Sample complexity ($C_{\mathrm{alg}}$) of
    ~ResNet-110~(\cref{fig:r110-jac-comp,fig:r110-spec-l1-comp,fig:r110-spec-fro-comp}),
    WideResNet-28-10~(\cref{fig:wrn-jac-comp,fig:wrn-spec-l1-comp,fig:wrn-spec-fro-comp}),
    and Densenet-100~(\cref{fig:d100-jac-comp,fig:d100-spec-l1-comp,fig:d100-spec-fro-comp})
quantified using the three measures discussed in the paper. Left
is better. Vanilla is omitted
from~\cref{fig:r110-spec-l1-comp,fig:r110-spec-fro-comp,fig:d100-spec-l1-comp,fig:d100-spec-fro-comp}
as it is too far to the right. Also, in situations where SRN-50 and SN performed the same, we removed the histogram to avoid clutter.}
  \label{fig:compl}
\end{figure}\vspace{-1ex}
\paragraph{Empirical Evaluation of Generalization Behaviour}
When all the factors in training~(eg. architecture, dataset, optimizer, among other) as in SRN vs SN vs Vanilla, are fixed, and the only variability is in the normalization, the
generalization error can be written as
$ \abs{\mathrm{Train\ Err} - \mathrm{Test\ Err}} \le
\tildeO{\sqrt{\nicefrac{C_{\mathrm{alg}}}{m}}}$   where $\tildeO{\cdot}$ ignores the logarithmic terms,
$m$ is the number of samples in the dataset, and $C_{\mathrm{alg}}$
denotes a measure of {\em sample complexity} for a given algorithm. Lower the value of $C_{\mathrm{alg}}$, the better is the generalization. Before we give various expressions for $C_{\mathrm{alg}}$, we first define 
a common quantity in all these expressions, called the {\em margin} $\gamma =
f_\theta\br{\vec{x}}\bs{y} - \max_{j\neq
  y}f_\theta\br{\vec{x}}\bs{j}$. It measures the gap between the output of the network on the correct
label and the other labels. Now we define three recently proposed
sample complexity measures useful to quantify the generalization
behaviour with further descriptions in~\cref{sec:gener-behav}:
\begin{itemize}[leftmargin=*]
\item \textbf{Spec-Fro:}
  $\nicefrac{\prod_{i=1}^L\norm{\vec{W}_i}_2^2\sum_{i=1}^L\srank{\vec{W}_i}}{\gamma^2}$~\citep{neyshabur2018a}.
\item \textbf{Spec-L1:} $\nicefrac{\prod_{i=1}^L\norm{\vec{W}_i}_2^2\br{\sum_{i=1}^L\frac{\norm{\vec{W}_i}_{2,1}^{\nicefrac{2}{3}}}{\norm{\vec{W}_i}_2^{\nicefrac{2}{3}}}}^3}{\gamma^2}$~\citep{bartlett2017spectrally},
 $\norm{.}_{2,1}$ is the matrix 2-1 norm.
\item \textbf{Jac-Norm:}
  $\sum_{i=1}^L\nicefrac{\norm{\vec{h}_i}_2\norm{\vec{J}_i}_2}{\gamma}$~\citep{wei2019},
  where  $\vec{h}_i$ is the $i^{\it th}$ hidden layer and $\vec{J}_i
  =\frac{\partial \gamma}{\partial h_i}$ 
\end{itemize}
\vspace{-2ex}
\paragraph{Histogram of the Empirical \Gls{lip} Constant (eLhist)}
We evaluate above mentioned sample complexity measures on $10,000$ points from the dataset and plot the distribution of the $\log$ using a histogram shown in~\cref{fig:compl}. 
The more to the left the histogram, the better is the generalization capacity of the
network.

For better clarity, we provide the $90$ percentile for each of these
histograms in ~\cref{tab:perc-compl} in~\cref{sec:gener-behav}. 
 As the plots and the table show, both \gls{srn} and \gls{sn} produces a much smaller
quantity than a Vanilla network and in 7 out of the 9 cases, SRN is
better than SN. The difference between SRN and SN is much more
significant in the case of Jac-Norm. As this depend on the empirical
lipschitzness, it provides the empirical validation of our arguments in~\cref{sec:whyStable}.

{\em Above experiments indicate that SRN, while providing enough capacity for the standard classification task, is remarkably less prone to memorization and provides improved generalization.}

\begin{figure}[t]
  \centering
\begin{subfigure}[c]{0.495\linewidth}
  \centering
  \def\svgwidth{0.99\columnwidth}
  \resizebox{0.95\textwidth}{!}{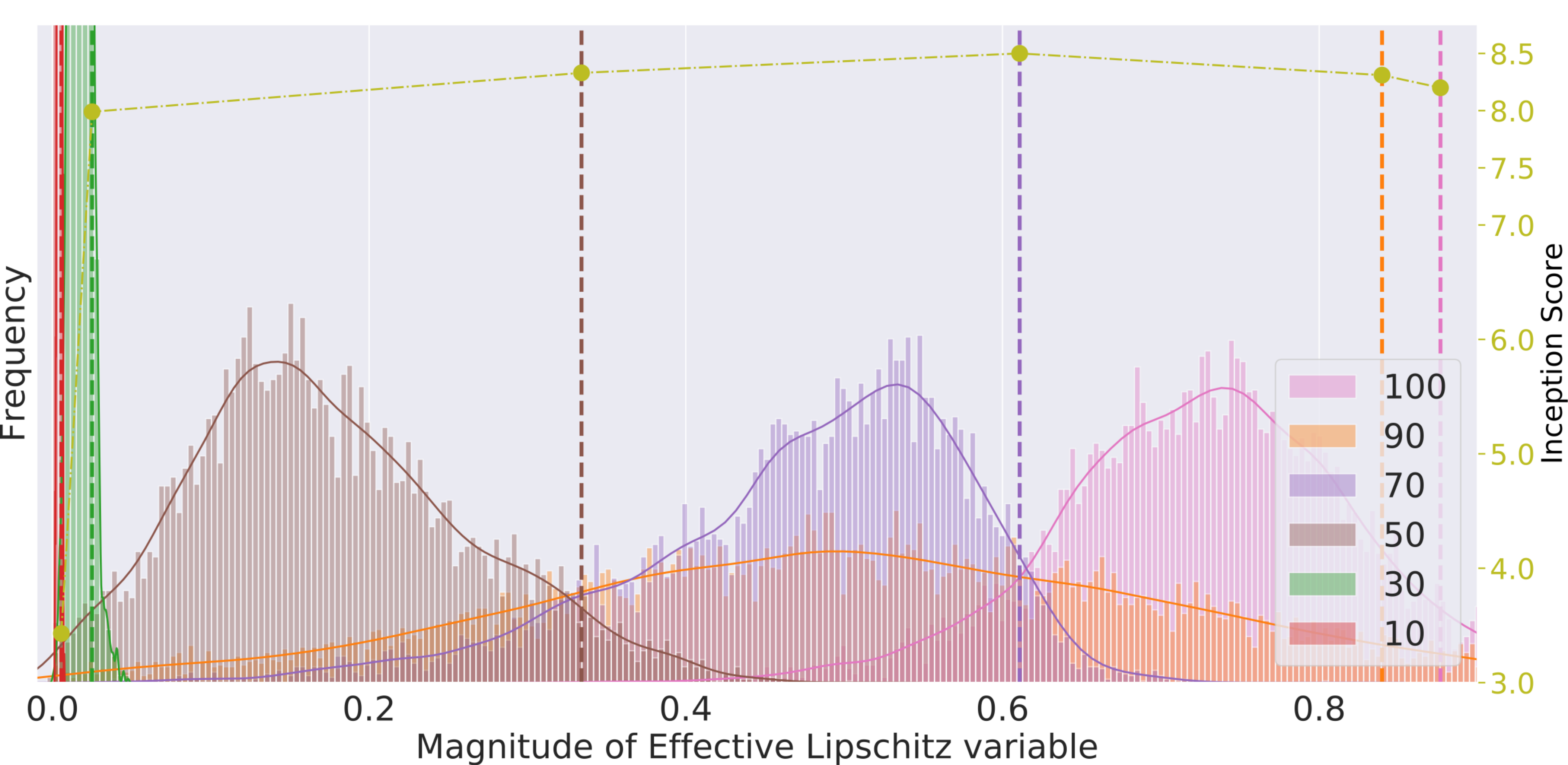} \caption{Varying stable rank constraints} \label{fig:lip_rank_stable}
\end{subfigure}
\begin{subfigure}[c]{0.495\linewidth}
  \centering
  \def\svgwidth{0.99\columnwidth}
  \resizebox{0.95\textwidth}{!}{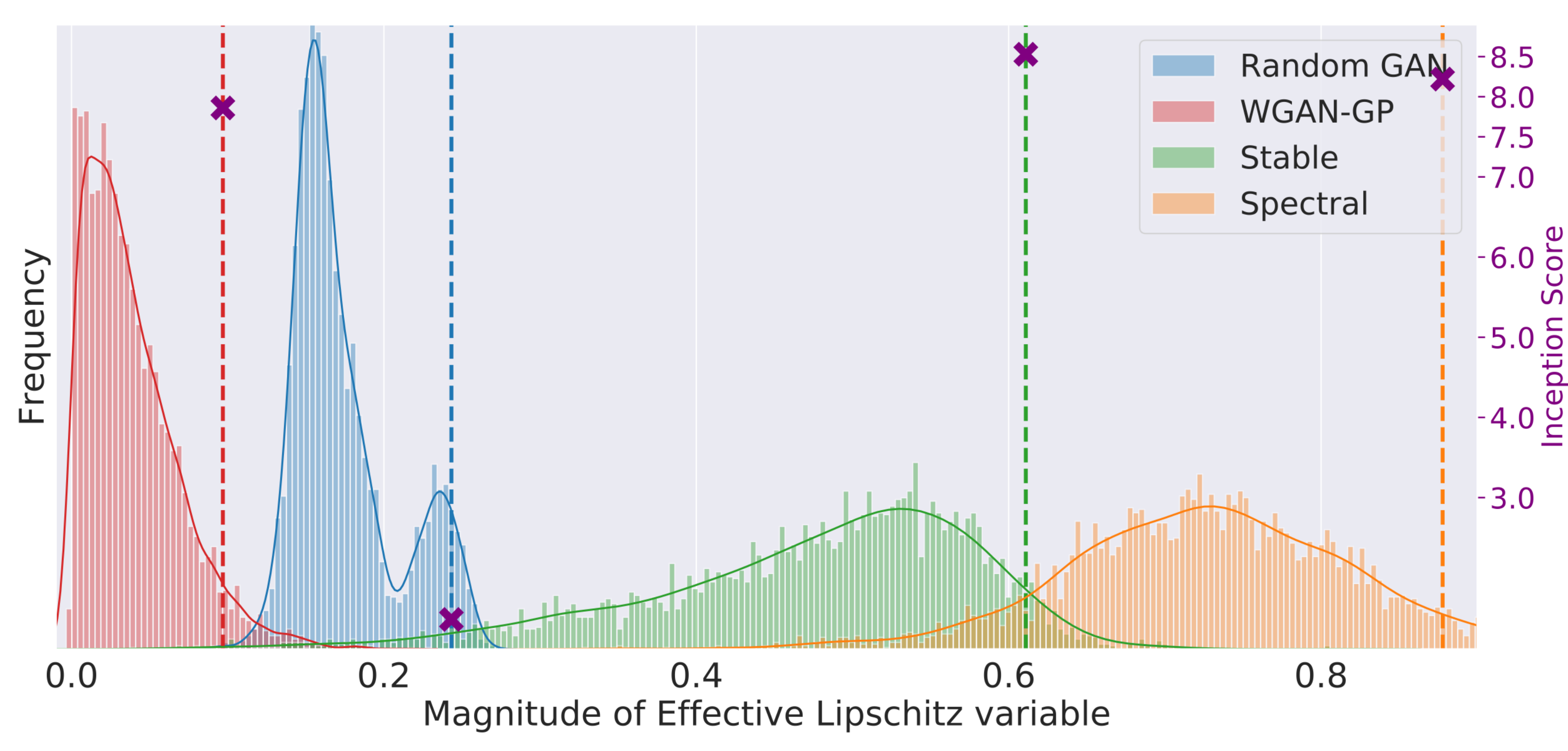}\caption{Comparison against different approaches}  \label{fig:lip_rank_stable_uncond}
\end{subfigure}
\caption{{\bf eLhist} for unconditional GAN on CIFAR10. Dashed
  vertical lines represent  95{\it th} percentile. Solid circles and
  crosses represent the {\em inception score} for each
  histogram. Figure~\ref{fig:lip_rank_stable} shows SRN-GAN for
  different stable rank constraints (\eg $90$ implies $c=0.9$). Figure~\ref{fig:lip_rank_stable_uncond} compares various approaches. Random-GAN represents random initialization
  (no training). For SRN-GAN, we use
  $c=0.7$.}\vspace{-1.5em}
\end{figure}
\footnotetext[1]{Results are taken from ~\citet{miyato2018spectral}. The rest of the results in the tables are generated by us.}
\vspace{-3ex}\subsection{Training of Generative Adversarial Networks (SRN-GAN)}\label{sec:srn-exp}

In GANs, there is a natural tension between the {\em capacity} and the
{\em generalizability} of the discriminator. The capacity ensures that
if the generated distribution and the data distribution are
different, the discriminator has the ability to distinguish them. At
the same time, the discriminator has to be generalizable, implying,
the class of hypothesis should be small enough to ensure that it is
not just memorizing the dataset. Based on these arguments, we use SRN
in the discriminator of GAN which we call SRN-GAN, and compare it
against SN-GAN, WGAN-GP, and orthonormal regularization based GAN
(Ortho-GAN). 

Along with providing results using evaluation metrics such as Inception score (IS)~\citep{salimans2016improved} , FID~\citep{heusel2017gans}, and Neural divergence score (ND)~\citep{gulrajani2018towards}, we use histograms of the empirical \Gls{lip} constant, {\em refered to as eLhist} from now onwards, for the purpose of analyses. For a given trained GAN (unconditional), we create $2,000$ pairs of samples, where each pair $(\bfx_i, \bfx_j)$ consists of $\bfx_i$ (randomly sampled from the `real' dataset) and $\bfx_j$ (randomly sampled from the generator). Each pair is then passed through the discriminator to compute $\nicefrac{\norm{f(\bfx_i) - f(\bfx_j)}_2}{\norm{\bfx_i - \bfx_j}_2}$, which we then use to create the histogram. %
In the conditional setting, we  sample a class from a discrete
uniform distribution, and then follow the same
approach as described for the unconditional setting.
\vspace{-1ex}
\paragraph{Effect of Stable Rank on eLhist and Inception Score}
\begin{wraptable}{r}{0.6\linewidth}\footnotesize
  \begin{tabular}{c@{\quad}cccc@{}}
    \toprule\footnotesize
    & Algorithm &  Inception Score & FID & Intra-FID\\
    \midrule
    \multirow{5}{*}{\rotatebox[origin=c]{90}{\footnotesize Uncond.}}
    & Orthonormal\footnotemark[1]  & $7.92\pm .04$ &$23.8$&-\\ 
    & WGAN-GP & $7.86\pm .07$ &$21.7$& - \\
    & SN-GAN\footnotemark[1] & $8.22\pm .04$ &$20.67$& -\\
    & SRN-70-GAN & $\mathbf{8.53}\pm 0.04$& $\mathbf{19.83}$&- \\
    & SRN-50-GAN & $8.33\pm 0.06$& $\mathbf{19.57}$&- \\
    \cmidrule{2-5}
    \multirow{3}{*}{\rotatebox[origin=c]{90}{\footnotesize Cond.}}
    & SN-GAN & $8.71\pm .04$ &$16.04$\footnote{This is different from what is reported in the original paper.}& $26.24$\\ 
    &SRN-$70$-GAN & $\mathbf{8.93}\pm 0.12$& $\mathbf{15.92}$ & $\mathbf{24.01}$\\
    &SRN-$50$-GAN & $ 8.76\pm 0.09$& $16.89$& $27.3$\\
   \bottomrule
  \end{tabular}
  \caption{Inception and FID score on CIFAR10.}
  \label{tbl:comp_uncond_model1}\vspace{-2ex}
\end{wraptable}
As shown in~\cref{fig:lip_rank_stable}, lowering the value of $c$
(aggressive reduction in the stable rank) moves the histogram towards
zero, implying, lower empirical \Gls{lip} constant. This validates our
arguments provided in~\cref{sec:whyStable}. Lowering $c$ also
improves inception score, however, extreme reduction in the stable
rank ($c=0.1$) dramatically collapses the histogram to zero and also
drops the inception score significantly. This is due to the fact that
at $c=0.1$, the capacity of the discriminator is reduced to the point
that it is not able to learn to differentiate between the real and the
fake samples anymore.

~\cref{tbl:comp_uncond_model1,tab:inc_fid_cifar100} show that SRN-GAN 
consistently provide better FID score and an
extremely competitive inception score on CIFAR10~(both conditional and
unconditional setting) and CIFAR100~(unconditional setting). In
Table~\ref{tab:nn_dis_cifar100}, we  compare the ND loss on CIFAR10 and CelebA datasets. Note, ND
  has been looked as a metric {\em more robust to memorization} than FID and
  IS in recent works~\citep{gulrajani2018towards,arora2017gans}. We
report our exact setting to compute ND in~\cref{sec:gansetup}. 
We essentially report the loss incurred by a \textit{fresh} classifier
trained to discriminate the generator distribution and the data
distribution. Thus higher the loss, the better the generated
images. As evident, SRN-GAN provides better ND scores on both datasets. For a
qualitative analysis of the images, we compare generations
in both conditional and unconditional setting in~\cref{sec:sample-images}.%
\vspace{-2ex}
\paragraph{Comparing different approaches}
\begin{wraptable}{r}{0.38\linewidth}\footnotesize
  \centering
  \begin{tabular}{@{}c@{}cc@{}}
    \toprule
   Model&IS & FID\\
   \midrule
    SN-GAN&$\mathbf{9.04}$&$23.2$\\
    SRN-GAN (Our)&$8.85$&$\mathbf{19.55}$\\
    \bottomrule
  \end{tabular}
  \caption{CIFAR100 experiments.}
  \label{tab:inc_fid_cifar100}
  \centering
  \begin{tabular}{@{}c@{}cc@{}}
  \toprule
   Model&CIFAR10 & CelebA\\
   \midrule
    SN-GAN&$10.69$&$0.36$\\
    SRN-GAN (Our)&$\mathbf{11.97}$&$\mathbf{0.64}$\\
    \bottomrule
  \end{tabular}
  \caption{Neural Discriminator Loss~(Higher the better).}
  \label{tab:nn_dis_cifar100}\vspace{-1.8em}
\end{wraptable}
In addition, in~\cref{fig:lip_rank_stable_uncond}, we provide
eLhist for comparing different approaches. Random-GAN, as expected, has a low
empirical \Gls{lip} constant and extremely poor inception score. Unsurprisingly, WGAN-GP has a lower $L_e$ than Random-GAN, due to its explicit constraint 
on the \Gls{lip} constant, while providing a higher inception score. On the other hand, SRN-GAN, by virtue of its softer constraints on the \Gls{lip} constant,
trades off a higher \Gls{lip} constant for a better inception score---highlighting the
flexibility provided by SRN.
Additional experiments
in~\cref{sec:lipsch-cond-gans} show more detailed behaviour of GANs in
regards to empirical lipschitz in a variety of settings.

\vspace{-1em}\section{Conclusion}
\label{sec:conc}
\vspace{-1em}We proposed a new normalization (SRN) that allows us to
constrain the stable rank of each affine layer of a \gls{nn}, which in
turn learns a mapping with low empirical \Gls{lip} constant. We also
provide optimality guarantees of SRN. On a variety of neural network
architectures, we showed that SRN improves the generalization and
memorization properties of a standard classifier. In addition, we show
that SRN improves the training of GANs and provide better inception,
FID, and ND scores.
\vspace{-1em}
\section{Acknowledgements}
\label{sec:acknowledgements}
\vspace{-1em}The authors would like to thank Leonard Berrada and Pawan Kumar for
helpful discussions. AS acknowledges support from The Alan Turing Institute under the Turing Doctoral Studentship
grant TU/C/000023. PHS and PD are supported by the ERC grant
ERC-2012-AdG 321162-HELIOS, EPSRC grant Seebibyte EP/M013774/1 and
EPSRC/MURI grant EP/N019474/1. PHS and PD also acknowledges the
Royal Academy of Engineering and FiveAI.

\bibliography{stable_rank}
\bibliographystyle{iclr2020_conference}
\clearpage
\appendix
\section{Technical Proofs }
\label{sec:optimalSrank}
Here we provide an extensive proof of~\cref{thm:srankOptimal}
(\cref{sec:srankProof}). We also provide the optimal solution to the spectral norm problem in~\cref{sec:spectralNormOptimal}. Auxiliary lemmas on which our proof depends are provided in~\cref{sec:auxLemmas}.

\subsection{Proof for Optimal Stable Rank Normalization.~(Main Theorem)}
\label{sec:srankProof}
\optimalthm*
\begin{proof} 
Here we provide the proof of~\cref{thm:srankOptimal} (in the main paper) for all the three cases with optimality and uniqueness guarantees. Let $\widehat{\wmat}_k$ be the optimal solution to the problem for any of the two cases. From~\cref{lem:opt-frobenius}, the $\mathrm{SVD}$ of $\wmat$ and $\widehat{\wmat}_k$ can be written as $\wmat=\vec{U}\Sigma\vec{V}^\top$ and $\widehat{\wmat}_k=\vec{U}\Lambda\vec{V}^\top$, respectively. Then, 
   \( L = \norm{\wmat - \widehat{\wmat}_k}_{\forb}^2 = \ip{\Sigma - \Lambda}{\Sigma - \Lambda}_{\forb} \).
From now onwards, we denote $\Sigma$ and $\Lambda$ as vectors consisting of the diagonal entries, and $\ip{.}{.}$ as the vector inner product
\footnote{$\ip{.}{.}_\forb$ represents the Frobenius inner product of two matrices, which in the case of diagonal matrices is the same as the  inner product of the diagonal vectors.}. 
\paragraph{Proof for Case (a):}
In this case, there is no constraint enforced to preserve any of the singular values of the given matrix while obtaining the new one. The only constraint is that the new matrix should have the stable rank of $r$. Let us assume $\Sigma = \br{\sigma_1, \cdots, \sigma_p}$, $\Sigma_2 = \br{\sigma_2, \cdots, \sigma_p}$, $\Lambda = \br{\lambda_1, \cdots, \lambda_p}$ and $\Lambda_2 = \br{\lambda_2, \cdots, \lambda_p}$. Using these notations, we can write $L$ as:
\begin{align}
    L &=\ip{\Sigma}{\Sigma} +  \ip{\Lambda}{\Lambda} - 2\ip{\Sigma}{\Lambda}\nonumber\\
    \label{eq:l1}
    &= \ip{\Sigma}{\Sigma} +  \lambda_1^2 + \ip{\Lambda_2}{\Lambda_2} - 2\sigma_1\lambda_1 - 2\ip{\Sigma_2}{\Lambda_2}
  \end{align}

\noindent Using the stable rank constraint $\srank{\widehat{\wmat}_k}
= r$, which is \(r = 1 + \dfrac{\sum_{j=2}^p\lambda_j^2}{\lambda_1^2}
\).

\paragraph{Case for $\mathbf{r>1}$} If $r>1$ we obtain the following equality constraint, making the
problem non-convex.

\begin{align}
\label{eq:lambda1}
\lambda_1^2 = \dfrac{\ip{\Lambda_2}{\Lambda_2}}{r - 1}
\end{align} 
However, we will show that the solution we obtain is optimal and
unique. Substituting~\cref{eq:lambda1} into~\cref{eq:l1} we get
  \begin{align}
  \label{eq:l2}
    L = \ip{\Sigma}{\Sigma} + \dfrac{\ip{\Lambda_2}{\Lambda_2}}{r - 1}  + \ip{\Lambda_2}{\Lambda_2} -  2\sigma_1\sqrt{\dfrac{\ip{\Lambda_2}{\Lambda_2}}{r - 1}} -  2\ip{\Sigma_2}{\Lambda_2}
  \end{align}
   
Setting $\dfrac{\partial L}{\partial \Lambda_2} = 0$ to get the family of critical points
  \begin{align}
    & \dfrac{2\Lambda_2}{r - 1}  + 2\Lambda_2 - \dfrac{ 4\sigma_1\Lambda_2}{2\sqrt{\br{r - 1}\ip{\Lambda_2}{\Lambda_2}}} -  2\Sigma_2\nonumber = 0\\
    & \implies \Sigma_2 = \Lambda_2\br{\dfrac{1}{r - 1} + 1 - \dfrac{
      \sigma_1}{ 1\sqrt{\br{r - 1}\ip{\Lambda_2}{\Lambda_2}}}
      }\label{eq:scalar_mult}\\
      &\implies \dfrac{\Sigma_2\bs{i}}{\lambda_2\bs{i}} = \br{\dfrac{1}{r - 1} + 1 - \dfrac{ \sigma_1}{ 1\sqrt{\br{r -
        1}\ip{\Lambda_2}{\Lambda_2}}}} = \dfrac{1}{\gamma_2} \quad \forall~1\le i\le p\label{eq:Scalar_mult}
  \end{align}
As the R.H.S. of~\ref{eq:Scalar_mult} is independant of $i$, the above equality implies that all the critical points of~\cref{eq:l2}
are a scalar multiple of $\Sigma_2$, implying,  $\Lambda_2 = \gamma_2
\Sigma_2$. Note that the domain of $\Lambda_2$ are all strictly
positive vectors and thus, we can ignore the critical point at
$\Lambda_2 = \vec{0}$. Substituting this into~\cref{eq:scalar_mult} we obtain
  \begin{align*}
    & \Sigma_2 = \gamma_2\Sigma_2\br{\dfrac{1}{r - 1} + 1 - \dfrac{ \sigma_1}{ \gamma_2\sqrt{\br{r - 1}\ip{\Sigma_2}{\Sigma_2}}} }
  \end{align*}
\noindent Using the fact that $\ip{\Sigma_2}{\Sigma_2} =
\norm{\vec{S}_2}_{\forb}^2$ in the above equality and with some
algebraic manipulations, we obtain \( \gamma_2  = \frac{\gamma + r -
  1}{r} \) where, $\gamma = \frac{\sqrt{r-1}
  \sigma_1}{\norm{\vec{S}_2}_\forb}$. Note, $r \geq 1$, $\gamma \geq
0$, and $\Sigma \geq 0$, implying, $\Lambda_2 = \gamma_2 \Sigma_2 \geq
0$.

\paragraph{Local minima:} Now, we will show that $\Lambda_2$ is indeed a minima
of~\cref{eq:l2}. To show this, we  compute the hessian of  $L$. Recall that
\begin{align*}
  \dfrac{\partial L}{\partial \Lambda} &=    \dfrac{2r}{r - 1}\Lambda
                                         -\dfrac{
                                         2\sigma_1\Lambda}{\sqrt{\br{r
                                         - 1}\norm{\Lambda}_2^2}} -
                                         2\Sigma_2\\ 
  \vec{H}=\dfrac{\partial^2 L}{\partial^2 \Lambda} &=    \dfrac{2r}{r -
                                            1}\vec{I}  -\dfrac{
                                            2\sigma_1}{\sqrt{\br{r -
                                            1}}\norm{\Lambda}_2^2}\br{\norm{\Lambda}_2
                                            \vec{I} - \dfrac{1}{\norm{\Lambda}_2}\Lambda\Lambda^\top}\\
                &=  2\br{\dfrac{r}{r-1} - \dfrac{
                                            \sigma_1\norm{\Lambda}_2}{\sqrt{\br{r -
                                            1}}\norm{\Lambda}_2^2}
                  }\vec{I}+
                  \dfrac{2\sigma_1}{\sqrt{r-1}\norm{\Lambda}_2^3}\br{\Lambda\Lambda^\top}\\
\end{align*}
Now we need to show that $\vec{H}$ at the solution $\Lambda_2$ is PSD i.e.
$\forall~\vec{x}\in\real^{p-1},\enskip  \vec{x}^\top\vec{H\br{\Lambda_2}} \vec{x}\ge
0$
\begin{align*}
  \vec{x}^\top\vec{H} \vec{x} &=  2\br{\dfrac{r}{r-1} - \dfrac{
                                            \sigma_1\norm{\Lambda}_2}{\sqrt{\br{r -
                                            1}}\norm{\Lambda}_2^2}
  }\norm{\vec{x}_2^2}+
  \dfrac{2\sigma_1}{\sqrt{r-1}\norm{\Lambda}_2^3}\vec{x}^\top\br{\Lambda\Lambda^\top}\vec{x}\\
                              &\stackrel{(a)}{\ge} 2\br{\dfrac{r}{r-1} - \dfrac{
                                            \sigma_1\norm{\Lambda}_2}{\sqrt{\br{r -
                                            1}}\norm{\Lambda}_2^2}
     }\norm{\vec{x}}_2^2 \\
                              &\stackrel{(b)}{\ge} 2\br{\dfrac{r}{r-1} - \dfrac{
                                            \sigma_1}{\br{r -
                                            1}\lambda_1}
     }\norm{\vec{x}}_2^2
                              \stackrel{(c)}{=}  \dfrac{2r}{r-1}\br{ 1 - \dfrac{
                                            \gamma }{\br{\gamma+r - 1}}
                                }\norm{\vec{x}}_2^2   \\
  &\stackrel{(d)}{\ge}  \dfrac{2r}{r-1}\br{ 1 - \dfrac{
                                            1 }{\br{1+r - 1}}
     }\norm{\vec{x}}_2^2 = 2\norm{\vec{x}}_2^2\ge 0
\end{align*}
Here $(a)$ is due to the fact that the matrix
$\Lambda\Lambda^\top$ is an outer product matrix and is hence PSD. $(b)$ follows due
to~\cref{eq:lambda1} and $(c)$ follows by substituting $\lambda_1 =
\gamma_1\sigma_1$ and then the value of $\gamma_1$. Finally $(d)$
follows as $\br{ 1 - \dfrac{\gamma }{\br{\gamma+r - 1}}}$ is
decreasing with respect to $\gamma$ and we know that $\gamma<
1$ due to the assumption that $\srank{\wmat}< r$. Thus, we can substitute $\gamma = 1$ to find the minimum value of
the expression. This concludes our proof that $\Lambda_2$ is indeed a
local minima of $L$.

  \paragraph{Uniqueness:}   The uniqueness of $\Lambda_2$  as a
  solution to ~\cref{eq:l2} is shown
in~\cref{lem:uniqueness} and is also guaranteed by the fact that
$\gamma_2$ has a unique value. Using $\Lambda_2 = \gamma_2 \Sigma_2$ and
$\lambda_1 = \gamma_1\sigma_1$ in~\cref{eq:lambda1}, we obtain a
unique solution $\gamma_1 = \frac{\gamma_2}{\gamma}$.

Now, we need to show that it is also an unique solution
to~\cref{thm:srankOptimal}.

For all solutions to~\cref{thm:srankOptimal} that have singular
vectors which are different than that of $\wmat$,
by~\cref{lem:opt-frobenius}, the matrix formed by replacing the
singular vectors of the solution with that of $\wmat$ is also a
solution. Thus, if there were a solution with different singular
values than $\widehat{\wmat}_k$, it should have appeared as a solution
to~\cref{eq:l2}. However, we have shown that~\cref{eq:l2} has a unique
solution.

 Now, we need to show that among all matrices with the same singular
 values as that of $\widehat{\wmat}_k$,  $\widehat{\wmat}_k$ is
 strictly better  in terms of   $\norm{\wmat - \widehat{\wmat}_k}$
 . This requires a further assumption that every non-zero singular
 value of $\Lambda_2$ has a multiplicity   of $1$ i.e. they are all distinctly unique. Intuitively, this
  doesn't allow to create a different matrix by simply interchanging the
  singular vectors associated with the equal singular values. As the elements of $\Sigma_2$ are
  distinct, the elements of  $\Lambda_2 = \gamma_2\Sigma_2$ are also
  distinct and thus by the second part of~\cref{lem:opt-frobenius},
  $\widehat{\wmat}_k$ is strictly better, in terms of
  $\norm{\wmat - \widehat{\wmat}_k}$, than all matrices which have the same
  singular values as that of $\widehat{\wmat_k}$. This concludes our
  discussion on the uniqueness of the solution.

\paragraph{Case for $\mathbf{r=1}$:}
 Substituting $r=1$  in the
constraint  \(r = 1 + \dfrac{\sum_{j=2}^p\lambda_j^2}{\lambda_1^2}
\) we get \[ r  - 1  =
\dfrac{\sum_{j=2}^p\lambda_j^2}{\lambda_1^2}  = 0 \implies
\sum_{j=2}^p\lambda_j^2=0\] As it is a sum of squares, each of the
individual elements is also zero i.e. $\lambda_j=0~\forall 2\le j\le
p$. Substituting this into ~\cref{eq:l1}, we get the following quadratic
equation in $\lambda_1$ 
\begin{equation}
    L  = \ip{\Sigma}{\Sigma} +  \lambda_1^2 - 2\sigma_1\lambda_1\label{eq:ll0}
\end{equation}
which is minimized at
$\lambda_1 = \sigma_1$, thus proving that $\gamma_1 = 1$ and $\gamma_2
= 0$.
\paragraph{Proof for Case (b):} In this case, the constraints are meant to preserve the top $k$ singular values of the given matrix while obtaining the new one. Let $\Sigma_1 = \br{\sigma_1, \cdots, \sigma_k},~\Sigma_2 = \br{\sigma_{k+1}, \cdots, \sigma_p},~\Lambda_1 = \br{\lambda_1, \cdots, \lambda_{k}},~\Lambda_2 = \br{\lambda_{k+1}, \cdots, \lambda_p}$. Since satisfying all the constraints imply $\Sigma_1 = \Lambda_1$, thus, \( L := \norm{\wmat - \widehat{\wmat}_k}_{\forb}^2 = \ip{\Sigma_2 - \Lambda_2}{\Sigma_2 - \Lambda_2}\). From the stable rank constraint $\srank{\widehat{\wmat}_k} = r$, we have
  \begin{align}
    r &= \dfrac{\ip{\Lambda_1}{\Lambda_1}+\ip{\Lambda_2}{\Lambda_2}}{\lambda_1^2}\nonumber\\
    \therefore \; \; \ip{\Lambda_2}{\Lambda_2} &= r\lambda_1^2 - \ip{\Lambda_1}{\Lambda_1} = r\sigma_1^2 - \ip{\Sigma_1}{\Sigma_1}\label{eq:stable_const}
  \end{align}
The above equality constraint makes the problem non-convex. Thus, we relax it to \( \srank{\widehat{\wmat}_k} \leq r \) to make it a convex problem and show that the optimality is achieved with equality. Let \( r\sigma_1^2 - \ip{\Sigma_1}{\Sigma_1} = \eta \). Then, the relaxed problem can be written as
  \begin{align*}
&\min_{\Lambda_2\in\real^{p-k}} L:=\ip{\Sigma_2 - \Lambda_2}{\Sigma_2 - \Lambda_2}  \\
& \mathrm{s.t.} \quad \Lambda_2 \geq 0, \ip{\Lambda_2}{\Lambda_2} \leq \eta. 
\end{align*}
We introduce the Lagrangian dual variables $\Gamma \in \real^{p - k}$ and $\mu$ corresponding to the positivity and the stable rank constraints, respectively. The Lagrangian can then be written as 
\begin{align}
\cL\br{\Lambda_2, \Gamma, \mu}_{\Gamma\ge \vec{0}, \mu \geq 0} =  \ip{\Sigma_2 - \Lambda_2}{\Sigma_2 - \Lambda_2} + \mu\br{\ip{\Lambda_2}{\Lambda_2} -  \eta} - \ip{\Gamma}{\Lambda_2}
\end{align}
Using the primal optimality condition \( \dfrac{\partial \cL}{\partial \Lambda_2} =  \vec{0} \), we obtain
\begin{align}
	& 2\Lambda_2 - 2\Sigma_2 + 2\mu\Lambda_2 - \Gamma = \vec{0} \nonumber\\
  & \implies \Lambda_2 = \dfrac{\Gamma + 2\Sigma_2}{2\br{1 + \mu}}\label{eq:subst_lambda2}
\end{align}
Using the above condition on $\Lambda_2$ with the constraint $\ip{\Lambda_2}{\Lambda_2} \leq \eta$, combined with the stable rank constraint of the given matrix $\wmat$ that comes with the problem definition, $\srank{\wmat} > r$ (which implies $\ip{\Sigma_2}{\Sigma_2} > \eta$), the following inequality must be satisfied for any $\Gamma \geq 0$
\begin{align}
1 <  \frac{\ip{\Sigma_2}{\Sigma_2}}{\eta} \leq  \frac{\ip{\Gamma + \Sigma_2}{\Gamma + \Sigma_2}}{\eta} \leq (1+\mu)^2
\end{align}
For the above inequality to satisfy, the dual variable $\mu$ must be greater than zero, implying, $\ip{\Lambda_2}{\Lambda_2} - \eta$ must be zero for the complementary slackness to satisfy. Using this with the optimality condition~\cref{eq:subst_lambda2} we obtain
\begin{align}
   \br{1 + \mu}^2 &=  \dfrac{\ip{\Gamma + 2\Sigma_2}{\Gamma + 2\Sigma_2}}{4\eta} \nonumber
\end{align}
Substituting the above solution back into the primal optimality condition we get
\begin{equation}
  \label{eq:lambda_gamma}
  \Lambda_2 = \br{\Gamma + 2\Sigma_2} \dfrac{\sqrt{\eta}}{\sqrt{\ip{\Gamma + 2\Sigma_2}{\Gamma + 2\Sigma_2}}}
\end{equation}
Finally, we use the complimentary slackness condition $\Gamma\odot\Lambda_2 = \vec{0}$\footnote{$\odot$ is the hadamard product} to get rid of the dual variable $\Gamma$ as follows
\begin{align*}
  \Gamma \odot\br{\Gamma + 2\Sigma_2}\dfrac{\sqrt{\eta}}{\sqrt{\ip{\Gamma + 2\Sigma_2}{\Gamma + 2\Sigma_2}}} &= \vec{0}
\end{align*}
It is easy to see that the above condition is satisfied only when $\Gamma = \vec{0}$ as $\Sigma_2\ge\vec{0}$ and $\eta > 0$. Therefore, using $\Gamma = \vec{0}$ in~\cref{eq:lambda_gamma} we obtain the optimal solution of $\Lambda_2$ as
\begin{align}
\Lambda_2 = \dfrac{\sqrt{\eta}}{\sqrt{\ip{\Sigma_2}{\Sigma_2}}} \Sigma_2 = \frac{\sqrt{r\sigma_1^2 - \norm{\vec{S}_1}_\forb^2}}{\norm{\vec{S}_2}_\forb^2}\Sigma_2 = \gamma \Sigma_2
\end{align}
\paragraph{Proof for Case (c):} The monotonicity of \(\norm{\wmathw_k - \wmat}_\forb \) for $k\geq 1$ is shown in~\cref{lem:monotonicityK}. 
\end{proof}
Note that by the assumption that $\srank{\wmat}<r$, we can say that $\gamma <1$. Therefore in all the cases $\gamma_2<1$. Let us look at the required conditions for $\gamma_1\ge 1$ to hold. When $k\geq 1$, $\gamma_1 = 1$ holds. When $k=0$, for $\gamma_1> 1$ to be true, $\gamma_2< \gamma$ should hold, implying, $\br{\gamma - 1} < r\br{\gamma -1}$, which is always true as $r >1$ (by the definition of stable rank).

\begin{lem} 
\label{lem:monotonicityK}
For $k\geq 1$, the solution to the optimization problem~\cref{eq:srankProblem} obtained using Theorem~\ref{thm:srankOptimal} is closest to the original matrix  $\wmat$ in terms of Frobenius norm when only the spectral norm is preserved, implying, $k=1$.
\begin{proof}
For a given matrix $\wmat$ and a partitioning index $k \in \{1, \cdots, p\}$, let $\widehat{\wmat}_k = \vec{S}_1^k + \gamma \vec{S}_2^k$ be the matrix obtained using Theorem~\ref{thm:srankOptimal}. We use the superscript $k$ along with $\vec{S}_1$ and $\vec{S}_2$ to denote that this refers to the particular solution of $\widehat{\wmat}_k$. Plugging the value of $\gamma$ and using the fact that $\norm{\vec{S}_2^k}_\forb \neq 0$, we can write
  \begin{align*} \norm{\wmat - \widehat{\wmat}_k}_\forb &= \br{1 - \gamma}\norm{\vec{S}_2^k}_\forb\\
                                                  &= \norm{\vec{S}_2^k}_\forb - \sqrt{r\sigma_1^2 - \norm{\vec{S}_1^k}_\forb^2}\\
                                                  &= \norm{\vec{S}_2^k}_\forb - \sqrt{r\sigma_1^2 - \norm{\wmat}_\forb^2+ \norm{\vec{S}_2^k}_\forb^2}.
  \end{align*}
 Thus, $\norm{\wmat - \widehat{\wmat}_k}_\forb$ can be written in a simplified form as $f(x) = x - \sqrt{a + x^2}$, where $x = \norm{\vec{S}_2^k}_\forb$ and $a = r\sigma_1^2 - \norm{\wmat}_\forb^2$. Note, $a \leq 0$ as $ 1 \leq r \leq  \srank{\wmat}$, and $a + x^2 \geq 0$ because of the condition in \cref{thm:srankOptimal}. Under these settings, it is trivial to verify that $f$ is a monotonically decreasing function of $x$. Using the fact that as the partition index $k$ increases, $x$ decreases, it is straightforward to conclude that the minimum of $f(x)$ is obtained at $k=1$.
\end{proof}
\end{lem}

\subsection{Proof for Optimal Spectral Normalization}
\label{sec:spectralNormOptimal}
The widely used spectral normalization~\citep{miyato2018spectral} where the given matrix $\wmat \in \real^{m \times n}$ is divided by the maximum singular value is an approximation to the optimal solution of the spectral normalization problem defined as
\begin{align}
\label{eq:spectralNormProb}
\argmin_{\widehat{\wmat}} & \norm{\wmat - \widehat{\wmat}}_{\forb}^2  \\
\textit{s.t.} \quad & \sigma(\widehat{\wmat}) \leq s, \nonumber
\end{align}
where $\sigma(\widehat{\wmat})$ denotes the maximum singular value and $s>0$ is a hyperparameter. The optimal solution to this problem is shown in~\cref{alg:spectralNormOptimal}.
\begin{algorithm}[!h]
\caption{Spectral Normalization}
\label{alg:spectralNormOptimal}
\begin{algorithmic}[1]
\Require $\wmat \in \real^{m \times n}$, $s$ 
\State $\wmat_1 \gets \mathbf{0}$, $p \gets \min(m,n)$
\For {$k \in \{1, \cdots, p\}$}
\State $\{\bfu_k, \bfv_k, \sigma_k\} \gets SVD(\wmat, k)$
\Comment{perform power method to get $k$-th singular value}
\If{$\sigma_k \geq s$}
\State $\wmat_1 \gets \wmat_1 + s \; \bfu_k \bfv_k^{\top}$
\State $\wmat \gets \wmat - \sigma_k \; \bfu_k \bfv_k^{\top}$
\Else
\State break
\Comment{exit for loop}
\EndIf
\EndFor
\\
\Return $\wmat \gets \wmat_1 + \wmat$
\end{algorithmic}
\end{algorithm}
\noindent
In what follows we provide the optimality proof of~\cref{alg:spectralNormOptimal} for the sake of completeness. Let $\mathrm{SVD}\br{\wmat} = \vec{U}\Sigma\vec{V}^\top$ and let us assume that $\vec{Z} = \vec{S}\Lambda\vec{T}^\top$ is a solution to the problem~\ref{eq:spectralNormProb}. Trivially, $\vec{X} = \vec{U}\Lambda\vec{V}^\top$ also satisfies $\sigma\br{\vec{X}}\le s$. Now, $\norm{\vec{W}-\vec{X}}_{\forb}^2 = \norm{\vec{U}\br{\Sigma - \Lambda}\vec{V}^\top}_{\forb}^2  = \norm{\br{\Sigma - \Lambda}}_{\forb}^2 \leq \norm{\wmat-\vec{Z}}_{\forb}^2$, where the last inequality directly comes from~\cref{lem:ineq:frob_sing}. Thus the singular vectors of the optimal solution must be the same as that of $\vec{W}$. This boils down to solving the following problem
  \begin{equation}
\label{eq:spectralNormProb_vec}
\argmin_{\Lambda\in\real^{\mathrm{min}\br{m,n}}_{+}}  \norm{\Lambda - \Sigma}_{\forb}^2 \; \textit{s.t.} \;  \Lambda\bs{i} \leq s\enskip \forall i\in\bc{0, {\mathrm{min}\br{m,n}}}.
  \end{equation} 
  Here, without loss of generality, we abuse notations by considering $\Lambda$ and $\Sigma$ to represent the diagonal vectors of the original diagonal matrices $\Lambda$ and $\Sigma$, and $\Lambda\bs{i}$ as its $i$-th index. It is trivial to see that the optimal solution with minimum Frobenius norm is achieved when 
  \[\Lambda\bs{i} = \begin{cases} 
      \Sigma\bs{i}, & \text{if} \; \; \Sigma\bs{i} \le s \\
    s,  & \text{otherwise}.
   \end{cases} \] 
   This is exactly what~\cref{alg:spectralNormOptimal} implements.

\subsection{Auxiliary Lemmas}
\label{sec:auxLemmas}
\begin{lem}\label{lem:ineq:frob_sing}[Reproduced from Theorem~5 in~\citet{mirsky1960symmetric}]
  For any two matrices $\vec{A},{\vec{B}}\in\real^{m\times n}$ with singular values as $\sigma_1 \geq \cdots \geq \sigma_n$ and $\rho_1 \geq \cdots \geq \rho_n$, respectively 
  \[\norm{\vec{A}-\vec{B}}_{\forb}^2 \ge \sum_{i=1}^n\br{\sigma_i - \rho_i}^2\]
\end{lem}
\begin{proof}
  Consider the following symmetric matrices \[
\vec{X}=
  \begin{bmatrix}
    \vec{0} & \vec{A}\\
    \vec{A}^\top & \vec{0}
  \end{bmatrix},
\vec{Y}=
  \begin{bmatrix}
    \vec{0} & \vec{B}\\
    \vec{B}^\top & \vec{0}
  \end{bmatrix},
\vec{Z}=
  \begin{bmatrix}
    \vec{0} & \vec{A-B}\\
    \vec{(A-B)}^\top & \vec{0}
  \end{bmatrix}
\]
Let $\tau_1 \geq \cdots \geq \tau_n$ be the singular values of  $\vec{Z}$. Then the set of characteristic roots of $\vec{X},\vec{Y}$ and $\vec{Z}$ in descending order are \(\bc{\rho_1,\cdots,\rho_n,-\rho_n,\cdots,-\rho_1}\), \(\bc{\sigma_1,\cdots,\sigma_n,-\sigma_n,\cdots,-\sigma_1} \), and \(\bc{\tau_1,\cdots,\tau_n,-\tau_n,\cdots,-\tau_1}\), respectively.
By Lemma~2 in~\citet{Wielandt1955} 
\[\bs{\sigma_1 - \rho_1, \cdots, \sigma_n - \rho_n, \rho_n - \sigma_n, \cdots, \rho_1 - \sigma_1}\preceq \bs{\tau_1,\cdots\tau_n,-\tau_n, -\tau_1},\] 
which implies that \begin{equation}\label{ineq:wielandt}
  \sum_{i=1}^n\br{\sigma_i - \rho_i}^2 \le \sum_{i=1}^n \tau_i^2 =
  \norm{\vec{A}-\vec{B}}_{\forb}^2
\end{equation}
\end{proof}

\begin{lem}
\label{lem:opt-frobenius}
Let $\amat,\bmat \in \real^{m\times n}$ where $ \mathrm{SVD}(\amat)=\vec{U}\Sigma\vec{V}^\top$ and $\bmat$ is the solution to the following problem 
\begin{equation}
\label{eq:stable_rank_opt_copy}
\bmat = \argmin_{\srank{\wmat}= r}\norm{\wmat-\amat}^2_\forb.
\end{equation}
Then, $\mathrm{SVD}\br{\bmat} = \vec{U}\Lambda\vec{V}^\top$ where $\Lambda$ is a diagonal matrix with non-negative entries. Implying, $\amat$ and $\bmat$ will have the same singular vectors.

\begin{proof}
Let us assume that $\vec{Z} = \vec{S}\Lambda\vec{T}^\top$ is
a solution to the problem~\ref{eq:stable_rank_opt_copy} where
$\vec{S}\neq\vec{U}$ and $\vec{T}\neq\vec{V}$. Trivially,
$\vec{X} = \vec{U}\Lambda\vec{V}^\top$ also lies in the feasible set
as it satisfies $\srank{\vec{X}}= r$~(note stable rank only depends on
the singular values). Using the fact that the Frobenius norm is invariant to 
unitary transformations, we can write $\norm{\amat-\vec{X}}_{\forb}^2 = \norm{\vec{U}\br{\Sigma
    - \Lambda}\vec{V}^\top}_{\forb}^2  = \norm{\br{\Sigma -
    \Lambda}}_{\forb}^2$. Combining this with ~\cref{lem:ineq:frob_sing}, we obtain
    \(\norm{\amat-\vec{X}}_{\forb}^2 = \norm{\br{\Sigma -
    \Lambda}}_{\forb}^2 \le \norm{\amat-\vec{Z}}_{\forb}^2
\).
Since,~$\vec{S}\neq\vec{U}$ and $\vec{T}\neq\vec{V}$, we can further 
change $\le$ to a strict inequality $<$. This completes the proof.

Generally speaking, the optimal solution to 
problem~\ref{eq:stable_rank_opt_copy} with constraints depending only
on the singular values (\eg stable rank in this case) will
have the same singular vectors as that of the original matrix.

Further the inequality in~\cref{ineq:wielandt} can be converted into a
strict inequality if neither of  $\vec{A}$ and $\vec{B}$ have repeated
singular values. Using that strict inequality, if both $\Sigma$ and
$\Lambda$ have no repeated values, then $\vec{B}$ is the only solution
to~\cref{eq:stable_rank_opt_copy} that has the singular values of
$\Lambda$.

\end{proof}
\end{lem}

\begin{lem}
\label{lem:uniqueness}
Let $\bfy_1 = a \bfx_1 + b \hat{\bfx}_1$ and $\bfy_2 = a \bfx_2 + b \hat{\bfx}_2$, where $\hat{\bfx}_1$ and $\hat{\bfx}_2$ denotes the unit vectors. Then, $\bfy_1 = \bfy_2$ if $\bfx_1 = \bfx_2$.
\end{lem}

\section{Empirical Lipschitz constant}
\subsection{Relating empirical local and global Lipschitz constants}

\begin{proposition} 
\label{prop:empiricalLocalLip}
Let $f: \real^m \mapsto \real$ be a Fr\'echet differentiable function,
$\data$ the dataset, and $\textit{Conv}\;(\bfx_i, \bfx_j)$ denotes the
convex combination of a pair of samples $\bfx_i$ and $\bfx_j$, then
$\forall p,q \in [1, \infty]$ such that $\frac{1}{p}+\frac{1}{q} = 1$
\begin{align}
\max_{\bfx_i, \bfx_j \in \data}\frac{\abs{f(\bfx_i) - f(\bfx_j)}}{\norm{\bfx_i - \bfx_j}_p} \; \leq \max_{\substack{{\bfx_i, \bfx_j \in \data} \\ \bfx \in \textit{Conv}\;(\bfx_i, \bfx_j) }} \norm{J_f(\bfx)}_q \nonumber
\end{align}
\begin{proof}
Let $f:\real^m\rightarrow \real$ be a differentiable function on an
open set containing $\vec{x}_i$ and $\vec{x}_j$ such that
$\vec{x}_i\neq\vec{x}_j$. By applying fundamental theorem of calculus
\begin{align}
  \abs{f\br{\vec{x}_i} - f\br{\vec{x}_j}} &= \abs{\int_{0}^1\nabla f\br{\vec{x}_i+\theta\br{\vec{x}_j-\vec{x}_i}}^\top\br{\vec{x}_j-\vec{x}_i}\partial\theta}  \nonumber \\
                                    &\le \int_{0}^1\abs{\nabla f\br{\vec{x}_i+\theta\br{\vec{x}_j-\vec{x}_i}}^\top\br{\vec{x}_j-\vec{x}_i}}\partial\theta  \nonumber \\
                                    &\overset{(a)}{\le} \int_{0}^1\norm{\nabla f\br{\vec{x}_i+\theta\br{\vec{x}_j-\vec{x}_i}}}_q\norm{\br{\vec{x}_j-\vec{x}_i}}_p\partial\theta  \nonumber \\
                                    &\le \int_{0}^1\max_{\theta\in\br{0, 1}}\norm{\nabla f\br{\vec{x}_i+\theta\br{\vec{x}_j-\vec{x}_i}}}_q\norm{\br{\vec{x}_j-\vec{x}_i}}_p\partial\theta  \nonumber \\
                                    &=\max_{\theta\in\br{0, 1}}\norm{\nabla f\br{\vec{x}_i+\theta\br{\vec{x}_j-\vec{x}_i}}}_q\norm{\br{\vec{x}_j-\vec{x}_i}}_p \int_{0}^1\partial\theta  \nonumber \\
\therefore \dfrac{\abs{f\br{\vec{x}_i} - f\br{\vec{x}_j}}}{\norm{\br{\vec{x}_j-\vec{x}_i}}_p}&\le\max_{\theta\in\br{0, 1}}\norm{\nabla f\br{\vec{x}_i+\theta\br{\vec{x}_j-\vec{x}_i}}}_q = \max_{\vec{x}\in \textit{Conv}\;(\bfx_i, \bfx_j)}\norm{\nabla f\br{\vec{x}}}_q. \nonumber
\end{align}
The inequality (a) is due to H\"{o}lder's inequality.
\end{proof}
\end{proposition}
\subsection{Effect of Rank on the Empirical \Gls{lip} Constants}
\label{sec:rankEffectLe}
Let $f(\bfx) = \wmat_2 \wmat_1 \bfx$ be a two-layer linear \gls{nn}
with weights $\wmat_1$ and $\wmat_2$. The Jacobian in this case is independent of $\bfx$. Thus, the local \Gls{lip} constant is the same for all $\bfx \in \real^m$, implying, local 
$L_e = L_l(\bfx) = L_l = \norm{\wmat_2 \wmat_1} \leq \norm{\wmat_2} \norm{\wmat_1}$. 
Note, in the case of 2-matrix norm reducing the rank will not affect the upperbound. However, as will be discussed below, rank reduction greatly influences the global $L_e$.

Let $\bfx_i$ and $\bfx_j$ be random pairs from $\data$ and $\Delta
\bfx \neq \bf0$ be the difference  $\bfx_i - \bfx_j$, then, the global
$L_e$ is \( \max_{\{\bfx_i, \bfx_j\} \in \data}
\frac{\norm{\wmat_2 \wmat_1 \Delta \bfx}}{\norm{\Delta \bfx}} \).
Let $k_1$ and $k_2$ be the ranks, and
$\sigma_1 \geq \cdots \geq \sigma_{k_1}$ and $\lambda_1 \geq \cdots
\geq \lambda_{k_2}$ the singular values of the matrices $\wmat_1$ and
$\wmat_2$, respectively. Let $P_i = \bfu_i \bar{\bfu}_i^\top$ be the
orthogonal projection matrix corresponding to $\bfu_i$ and
$\bar{\bfu}_i$, the left and the right singular vectors of
$\wmat_1$. Similarly, we define $Q_i$ for $\wmat_2$ corresponding to
$\bfv_i$ and $\bar{\bfv}_i$. Then, $\wmat_2 \wmat_1  =
\sum_{i=1}^{k_2} \sum_{j = 1}^{k_1} \lambda_i \sigma_j Q_i P_j $. The
upperbound, $\lambda_1 \sigma_1$, can be achieved if and only if
 $\Delta \bfx = \bar{\bfu}_1 \norm{\Delta \bfx}$ and $\bfu_1 = \bar{\bfv}_1$ (a perfect alignment), which is highly unlikely. In practice, not just the
maximum singular values, as is the case with the \Gls{lip} upper-bound,
rather the combination of the projection matrices and the singular
values play a crucial role in providing an estimate of global $L_e$. 
Thus, reducing the singular values, which is equivalent to minimizing the rank (or stable rank), will directly affect $L_e$. For example, assigning $\sigma_j = 0$, which in effect will reduce the rank of $\wmat_1$ by one, will nullify its influence on all projections associated with $P_j$. Implying, all the $k_2$ projections $\sigma_j (\sum_{i=1}^{k_2} \lambda_i Q_i) P_j$ that would propagate the input via $P_j$ will be blocked. This, in effect, will influence $\norm{\wmat_2 \wmat_2 \Delta  \bfx}$; hence the global $L_e$. In a more general setting, let $k_i$ be the rank of the $i$-th linear layer, then, each singular value of a $j$-th layer can influence the maximum of $\prod_{i=1}^{j-1} k_i \prod_{i=j+1}^{l} k_i$ many paths through which an input can be propagated. Thus, mappings with low rank (stable) will greatly reduce the global $L_e$. Similar arguments can be drawn for local $L_e$ in the case of \gls{nn} with non-linearity.
\section{The local \Gls{lip} upper-bound for Neural Networks}
\label{sec:LipNNUB}
As mentioned in~\cref{sec:background}, $L_l(\bfx) = \norm{J_f(\bfx)}_{p,q}$, where, in the case of \gls{nn}, the Jacobian is:

\begin{align}
    \label{eq:jacboianNN}
    J_f(\bfx) = \frac{\partial f\br{\bfx}}{\partial \bfx} := \frac{\partial \bfz_1}{\partial \bfx}\frac{\partial \phi_1 (\bfz_1)}{\partial \bfz_1} \cdots  \frac{\partial \bfz_l}{\partial \bfa_{l-1}}\frac{\partial \phi_l (\bfz_l)}{\partial \bfz_l}.
\end{align}
Using \( \frac{\partial \bfz_l}{\partial \bfa_{l-1}} = \wmat_l \) (affine transformation), and applying submultiplicativity of the matrix norms:
\begin{align}
    \label{eq:lipboundNN}
    \norm{J_f(\bfx)}_{p,q} \leq \norm{\wmat_1}_{p,q} \norm{\frac{\partial \phi_1 (\bfz_1)}{\partial \bfz_1}} \cdots \norm{\wmat_l}_{p,q} \norm{\frac{\partial \phi_l (\bfz_l)}{\partial \bfz_l}}.
\end{align}
Note, most commonly used activation functions $\phi(.)$ such as ReLU, sigmoid, tanh and maxout are known to have \Gls{lip} constant of $1$ (if scaled appropriately)%
\footnote{implying, \( \max_{\bfz} \norm{ \frac{\partial \phi(\bfz)}{\partial \bfz} }_p = 1 \).}%
, thus, the upper bound can further be written only using the operator norms of the intermediate matrices as 
\begin{align}\label{eq:lip_upp_bnd_app}
    L_l(\bfx) \leq \norm{J_f(\bfx)}_{p,q} \leq \norm{\wmat_l}_{p,q} \cdots \norm{\wmat_1}_{p,q}.
\end{align}%
Furthermore $L_l\br{\vec{x}}$ can be substituted by $L_l$, the
local \Gls{lip} constant, as the upper bound~(Eq.~\eqref{eq:lip_upp_bnd_app}) is
independent of $\vec{x}$. Note that this is one of the main reasons
why we consider the empirical \Gls{lip} to better reflect the true
behaviour of the function as the \gls{nn} is never exposed to the
entire domain $\real^m$ but only a small subset dependant on the data
distribution.

The other reason why this upper bound is a bad estimate is that the
inequality in Eq~\eqref{eq:lipboundNN} is tight only when the partial
derivatives are aligned, implying, $ \norm{\frac{\partial \bfz_\ell}{\partial
  \bfz_{\ell-1}}  \frac{\partial \bfz_{\ell+1}}{\partial
  \bfz_{\ell}}}_2 =  \norm{\frac{\partial \bfz_\ell}{\partial
  \bfz_{\ell-1}}}_2  \norm{\frac{\partial \bfz_{\ell+1}}{\partial
  \bfz_{\ell}}}_2~\quad \forall l -  2\le \ell\le l$. This problem has been referred to as the problem of
mis-alignment and is similar to  quantities like layer cushion in~\citet{arora18b}.

\section{Experimental details}
\label{sec:experimental-details}
\textbf{WideResNet-28-10} We use a standard WideResNet with $28$ layers
and a growth factor of $10$. In total, the network has  36,539,124
trainable parameters.  The network is the standard configuration with
batchnorm and ReLU activations and is trained with a weight decay of
$1e-4$. The learning rate was multiplied by $0.2$ after $60,120$, and
$160$ epochs respectively.

\textbf{ResNet-110} The ResNet-110 is a standard $110$ layered ResNet
with batch Norm and ReLU and has $1,973,236$ parameters. The network
is trained with SGD, an initial learning rate of $0.1$, which is
multiplied $0.1$ after $150$ and $250$ epochs respectively, a weight
decay of $5e-4$ and a momentum of $0.9$.

\textbf{Densenet-100} The DenseNet-100 is a standard $100$-layered
densenet with Batchnorm and ReLU and has a total of $800,032$
trainable parameters. The network is trained with SGD, an initial
learning rate of $0.1$, which is multiplied by $0.1$ after $150$ and
$250$ epochs respectively, a weight decay of $1e-4$, and a momentum of
$0.9$.

\textbf{VGG19} The VGG19 model is the standard 19-layered VGG model
with Batchnorm and ReLU. It has a total of $20,548,392$ trainable
parameters and is trained with SGD with a momentum of $0.9$ and a
weight decay of $5e-4$. The initial learning rate is $0.1$ and is
multiplied by $0.1$ after $150$ and $250$ epochs respectively. For the
shattering experiments, we used the same architecture and the same
training recipe except the initial learning rate, which was deceased
to $0.01$ as the model failed to learn the random labels with a large
learning rate.

\textbf{AlexNet} The Alexnet model is the standard ALexNet model with
$4,965,092$ trainable parameters. It was trained with SGD, with a
momentum of $0.9$, with an initial learning rate is $0.01$, which is
multiplied by $0.1$ after $150$ and $250$ epochs respectively. The
optimizer was further augmented with a weight decay rate of $5e-4$. Please refer
to the next section for results on different learning rates and with and without weight decay.

\subsection{Additional Experiments on Generalization}
\label{sec:gener-behav}

\paragraph{Complexity measures}

In this section, we provide more details about the various complexity
measures we used in~\cref{fig:compl}. 
\begin{itemize}
\item \textbf{Spec-Fro:}
  $\dfrac{\prod_{i=1}^L\norm{\vec{W}_i}_2^2\sum_{i=1}^L\srank{\vec{W}_i}}{\gamma^2}$~\citep{neyshabur2018a}. This
  bound is the main motivation of this paper; the
  two quantities used to normalize the margin~($\gamma$) is the
  product of spectral norm i.e. $\prod_{i=1}^L\norm{\vec{W}_i}_2^2$~(or worst case lipschitzness) and sum of
  stable rank {\em i.e.}, $\sum_{i=1}^L\srank{\vec{W}_i}$~(or an approximate parameter count like rank of a matrix).
\item \textbf{Spec-L1:} $\dfrac{\prod_{i=1}^L\norm{\vec{W}_i}_2^2
    \br{\sum_{i=1}^L\frac{\norm{\vec{W}_i}_{2,1}^{\nicefrac{2}{3}}}{\norm{\vec{W}_i}_2^{\nicefrac{2}{3}}}}^3}{\gamma^2}$,
where $\norm{.}_{2,1}$ is the matrix 2-1 norm. As showed
by~\citet{bartlett2017spectrally}, Spec-L1 is the spectrally
normalized margin, and unlike just the margin, is a good indicator of the
generalization properties of a network. 
\item \textbf{Jac-Norm:}
  $\sum_{i=1}^L\dfrac{\norm{\vec{h}_i}_2\norm{\vec{J}_i}_2}{\gamma}$~\citep{wei2019},
  where  $\vec{h}_i$ is the $i^{\it th}$ hidden layer and $\vec{J}_i
  =\frac{\partial \gamma}{\partial h_i}$ {\em i.e.}, the Jacobian of
  the margin with respect to the $i^{\it th}$ hidden layer (thus, a
  vector). Note, Jac-Norm depends on the norm of the Jacobian (local
  empirical \Gls{lip}) and norm of the hidden layers -  additional data-dependent terms compared to
  Spec-Fro and Spec-L1, thus captures a more realistic (and
  optimistic) generalization behaviour. 
\end{itemize}

For better clarity regarding~\cref{fig:compl}, we provide the $90$ percentile for each of these histograms in ~\cref{tab:perc-compl}. 
 As the plots and the table show, both \gls{srn} and \gls{sn} produces a much smaller
quantity than a Vanilla network and in 7 out of the 9 cases, SRN is
better than SN. The difference between SRN and SN is much more
significant in the case of Jac-Norm. As this depend on the empirical
lipschitzness, it provides the empirical validation of our arguments
in~\cref{sec:whyStable}. 
\begin{table}[!htb]\footnotesize
  \centering\renewcommand{\arraystretch}{1.1}
  \begin{tabular}[t]{C{3cm}C{3cm}c@{\quad}c@{\quad}c@{\quad}}\toprule
    Model & Algorithm & Jac-Norm & Spec-$L_1$ & Spec-Fro\\\hline
     \multirow{3}{*}{ResNet-110} & Vanilla & 17.7& $\infty$&$\infty$\\
         & Spectral (SN) & 17.8&10.8 &7.4\\
         & SRN-30 & {\bf 17.2}& {\bf 10.7} & {\bf 7.2}\\\cmidrule{2-5}
     \multirow{4}{*}{WideResNet-28-10} & Vanilla & 16.2& 14.60&11.18\\
         & Spectral (SN) & 16.13& 7.23&4.5\\
         & SRN-50 & 15.8& 7.3&4.5\\
         & SRN-30 & {\bf 15.7}& {\bf 7.20}&{\bf 4.4}\\\cmidrule{2-5}
     \multirow{4}{*}{Densenet-100} & Vanilla &19.2 &$\infty$ &$\infty$\\
         & Spectral (SN) &17.8 & 12.2&9.4\\
         & SRN-50 &17.6 &12 &9.2\\
         & SRN-30 & {\bf 17.7}&{\bf 11.8} &{\bf 9.0}\\\hline
  \end{tabular}
  \caption{Values of 90 percentile of $\log$ complexity measures
    from~\cref{fig:compl}. Here $\infty$ refers to the situations
    where the product of spectral norm blows up. This is the case in deep networks like ResNet-110 and Densenet-100 where the absence of spectral normalization (Vanilla)
    allows the product of spectral norm to grow arbitrarily large with increasing number of layers. Lower is better.}
  \label{tab:perc-compl}
\end{table}

\paragraph{Alexnet experiments:} Figure~\ref{fig:alex-test} shows the test error and generalization
error of Alexnet trained with a large learning rate of $0.1$. Note
that, the model fails to learn completely without weight
decay. Generalisation Error decreases monotonically with decreasing $c$
in the stable rank constraint. Test error is the lowest for
$c=0.5$. The constraint becomes too aggressive for even  $c$ lower
than that. The slightly more interesting observation is that having a
weight decay actually hurts generalization error while it has a
slightly positive effect on test error.

\begin{figure}[!htb]
  \centering
   \begin{subfigure}[!t]{0.49\linewidth}
    \def\svgwidth{0.98\linewidth}
    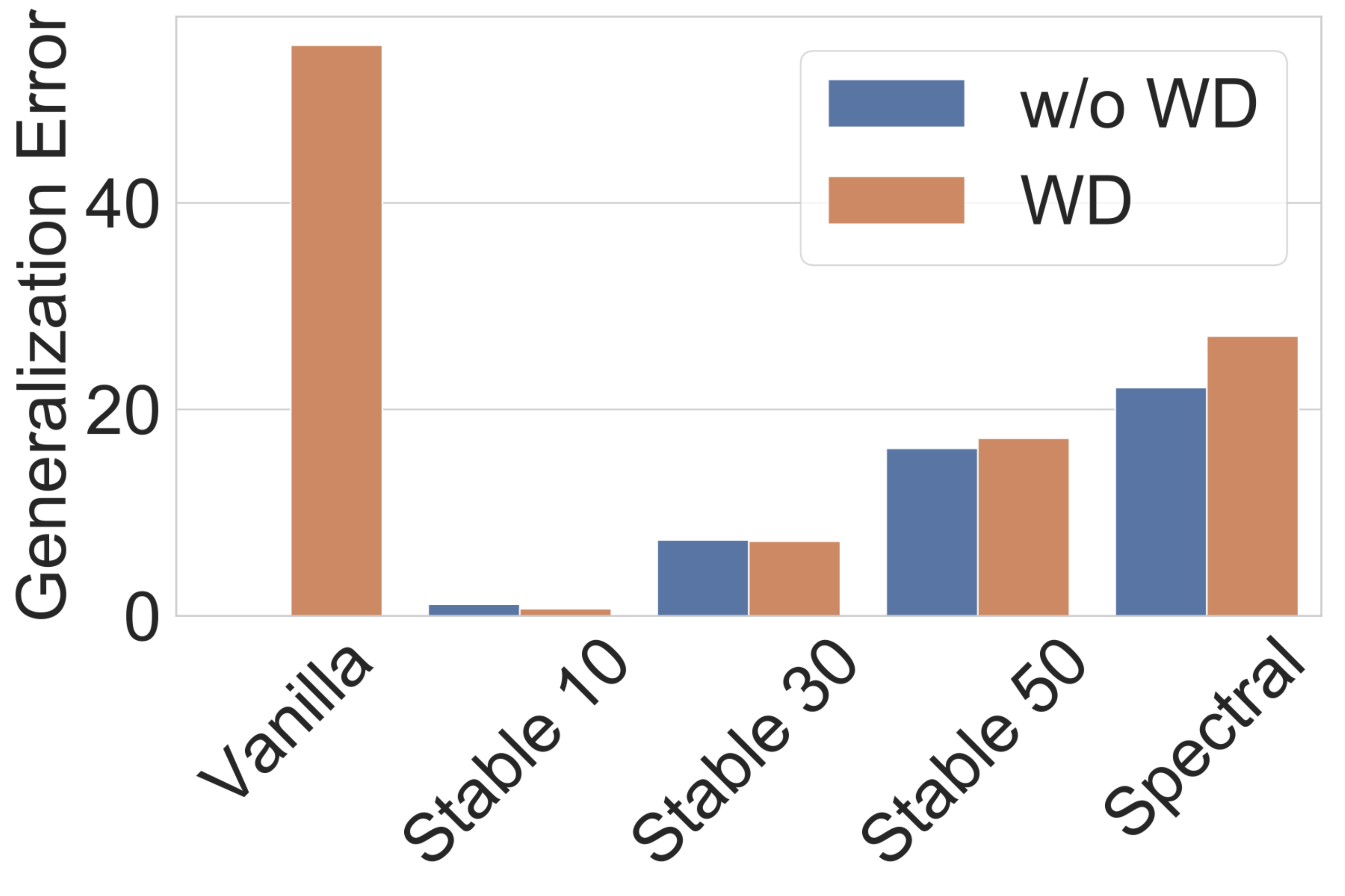
    \subcaption{Generalization Error}
  \end{subfigure}
   \begin{subfigure}[!t]{0.49\linewidth}
    \def\svgwidth{0.98\linewidth}
    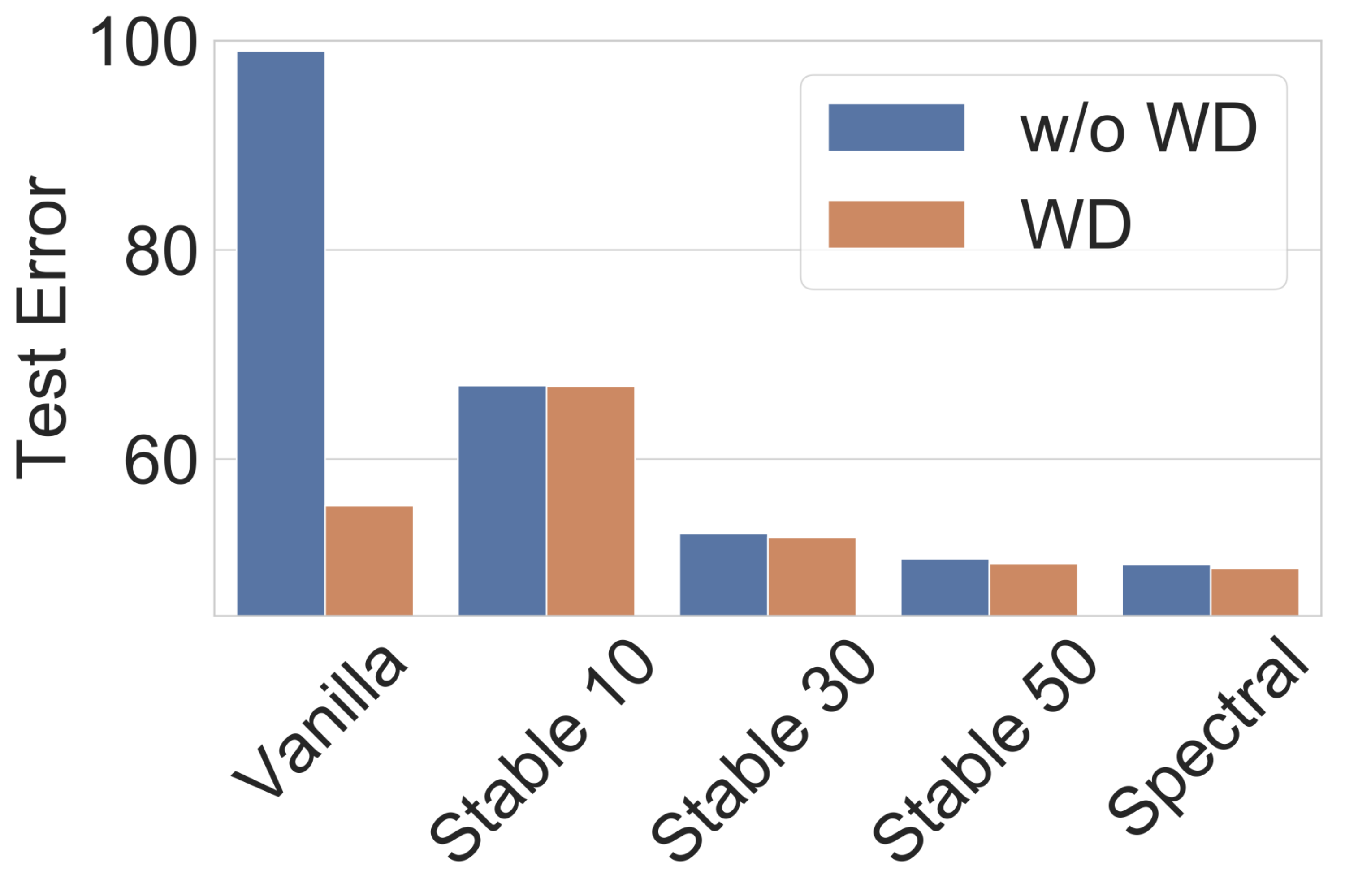
    \subcaption{Test Error}
  \end{subfigure}
  \caption{Test Error and Generalization Error of AlexNet trained with
    SGD with $lr=0.1$ on ~(clean) CIFAR-100.~(Lower is better}
  \label{fig:alex-test}
\end{figure}

\paragraph{Low Learning Rate}
Here, we train a WideResnet-28-10 with SRN, SN, and vanilla methods
with an $lr=0.01$ and weight decay of $5\times 10^{-4}$ on randomly
labelled CIFAR100. for $50$ epochs. The results are shown
in~\cref{tab:low_lr_wrn} and it further supports that SRN is more
robust to random noise than SN or vanilla methods.

\begin{table}[!htb]
  \centering\small
  \begin{tabular}{||c|c|c||}\hline
    Stable-30&Spectral&Vanilla\\\hline
    $29.04$& $17.24 $ & $1.22$ \\\hline
  \end{tabular}
  \caption{Training Error for WideResNet-28-10 on CIFAR100 with randomized
    labels, low lr$=0.01$, and with weight decay.~(Higher
    is better.)}
  \label{tab:low_lr_wrn}
\end{table}

\paragraph{With and without weight decay}
In~\cref{fig:alex-rand-high-train}, we show the training error of
Alexnet trained with SGD with and without weight decay~($=5e-4$) with
a learning rate of $0.01$. Again, we see that a more aggressive stable
rank constraint decreases fitting the random data . Similar results
are seen for ResNet-110 in~\cref{fig:r100-rand-high-train}.

\begin{figure}[!htb]
  \centering
  \begin{subfigure}[t]{0.47\linewidth}
    \def\svgwidth{0.98\linewidth}
    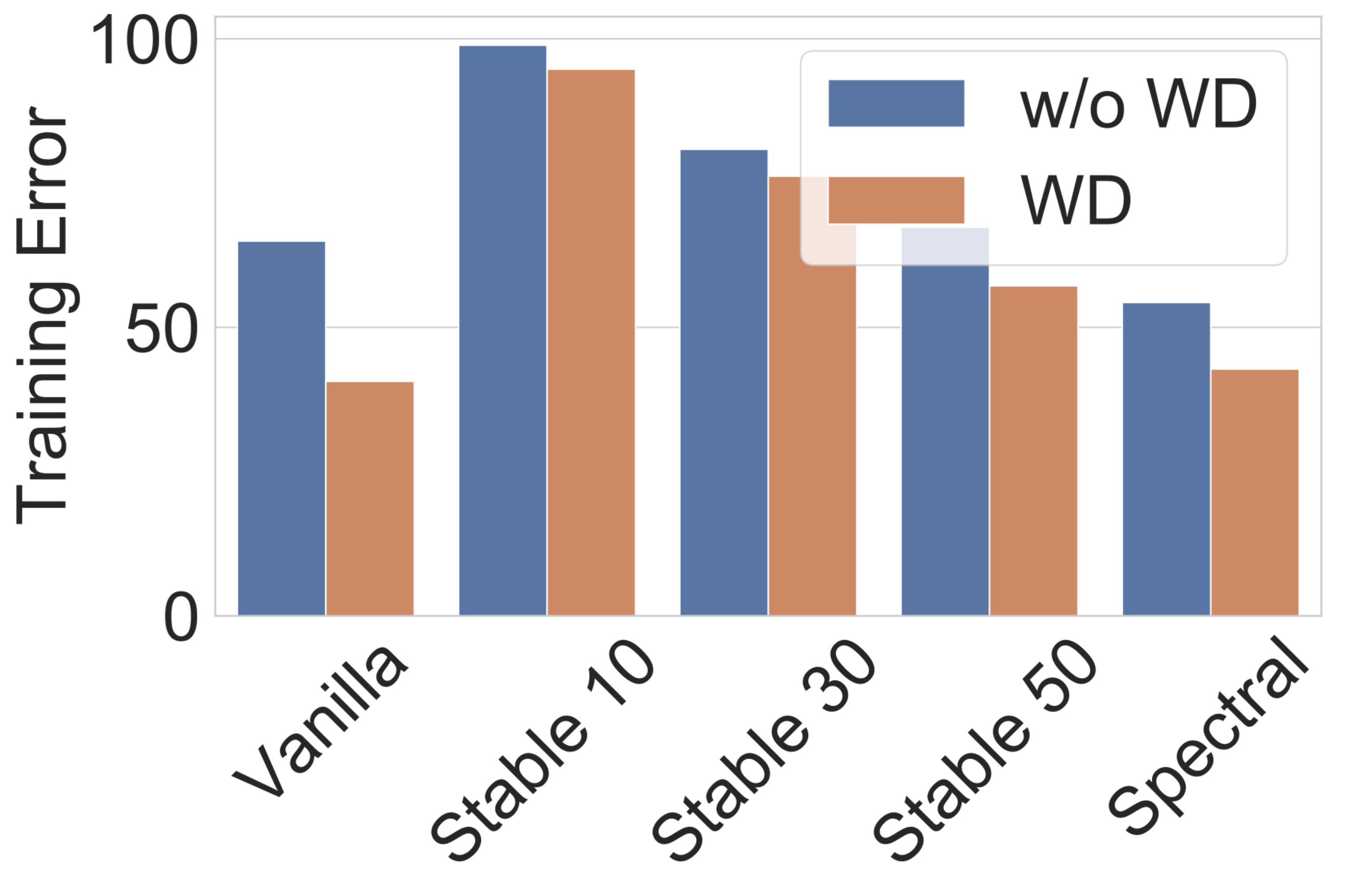
    \caption{AlexNet}\label{fig:alex-rand-high-train}
  \end{subfigure}\hfill
  \begin{subfigure}[t]{0.47\linewidth}
    \def\svgwidth{0.98\linewidth}
    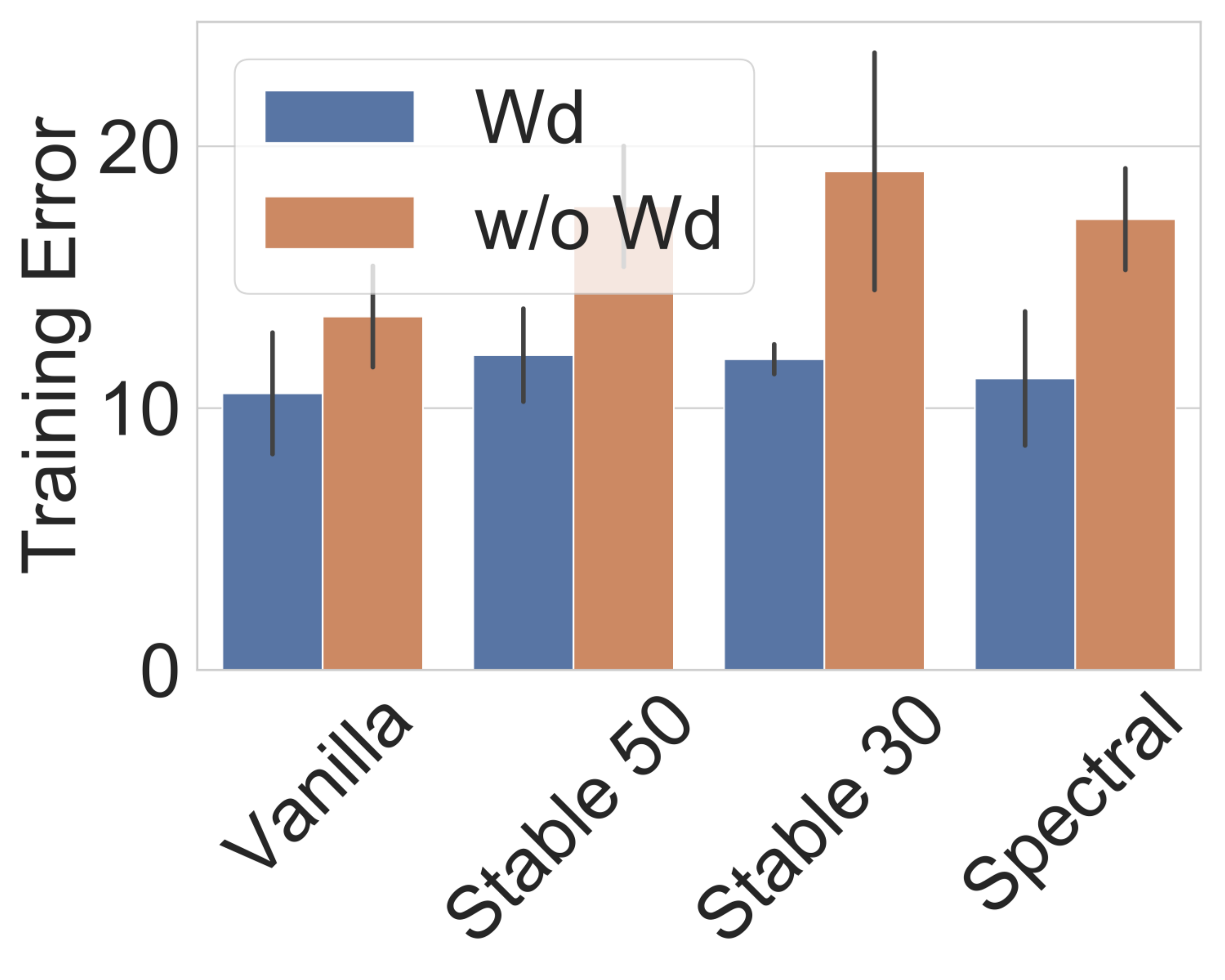
    \caption{ResNet110}\label{fig:r100-rand-high-train}
  \end{subfigure}
  \caption{Training error on randomly labelled CIFAR-100 with a
    learning rate of $0.01$ and with/ with out weight decay.~(Higher
    is better.}
  
\end{figure}

\paragraph{Low Learning Rate, with and without weight decay on clean
  CIFAR100}

In~\cref{tab:clean-c100-lowlr}, we show the test accuracies for the
clean data with the same confifuration as
in~\cref{tab:low_lr_rnet}. This corresponds to the hihgly
non-generelizable learning setting.

\begin{table}[!htb]
  \centering
  \begin{tabular}{ccccc}\toprule
    & Vanilla & Spectral & Stable-50 & Stable-30 \\\midrule
    W/o WD & 69.2 $\pm$ 0.5 & 69$\pm$0.1 & 69.1 $\pm$0.85 & 69.3 $\pm$0.4 \\
    With WD & 70.4$\pm$ 0.3 & 71.35 $\pm$0.25 & 70.6 $\pm$0.1 & 70.6 $\pm$0.1 \\\bottomrule
  \end{tabular}
  \label{tab:clean-c100-lowlr}
  \caption{Clean Test Accuracy on CIFAR10. The learning configuration
    corresponds to the non-generelizable settings with high learning
    rate. The corresponding shattering experiments for this setting
    are shown in~\cref{tab:low_lr_rnet}.}
\end{table}

\paragraph{Training Accuracy as Stopping Criterion}

In this section we show that our regularizor performs consistently for
a different stopping criterion. In particular, we use the train
accuracy as a stopping criterion. For Resnet110,
WideResnet-28,Densenet-100, and VGG-19 we use a train accuracy of
$99\%$ as a stopping criterion and report the test accuracy when that
train accuracy was achieved for the first time. For Alexnet, as SRN-30
never achieves a train accuracy higher than $55\%$, we use $55\%$ as
the stopping criterion and plot the test accuracies
in~\cref{fig:test-acc-stop}. Our results show that SRN-30 and SRN-50
outperform SN and vanilla consistently. In
Figure~\ref{fig:test-acc-stop-c10}, we show similar plots for CIFAR10.

\begin{figure}[t]
  \centering\small
  \begin{subfigure}[!t]{0.19\linewidth}
    \def\svgwidth{0.98\linewidth}
    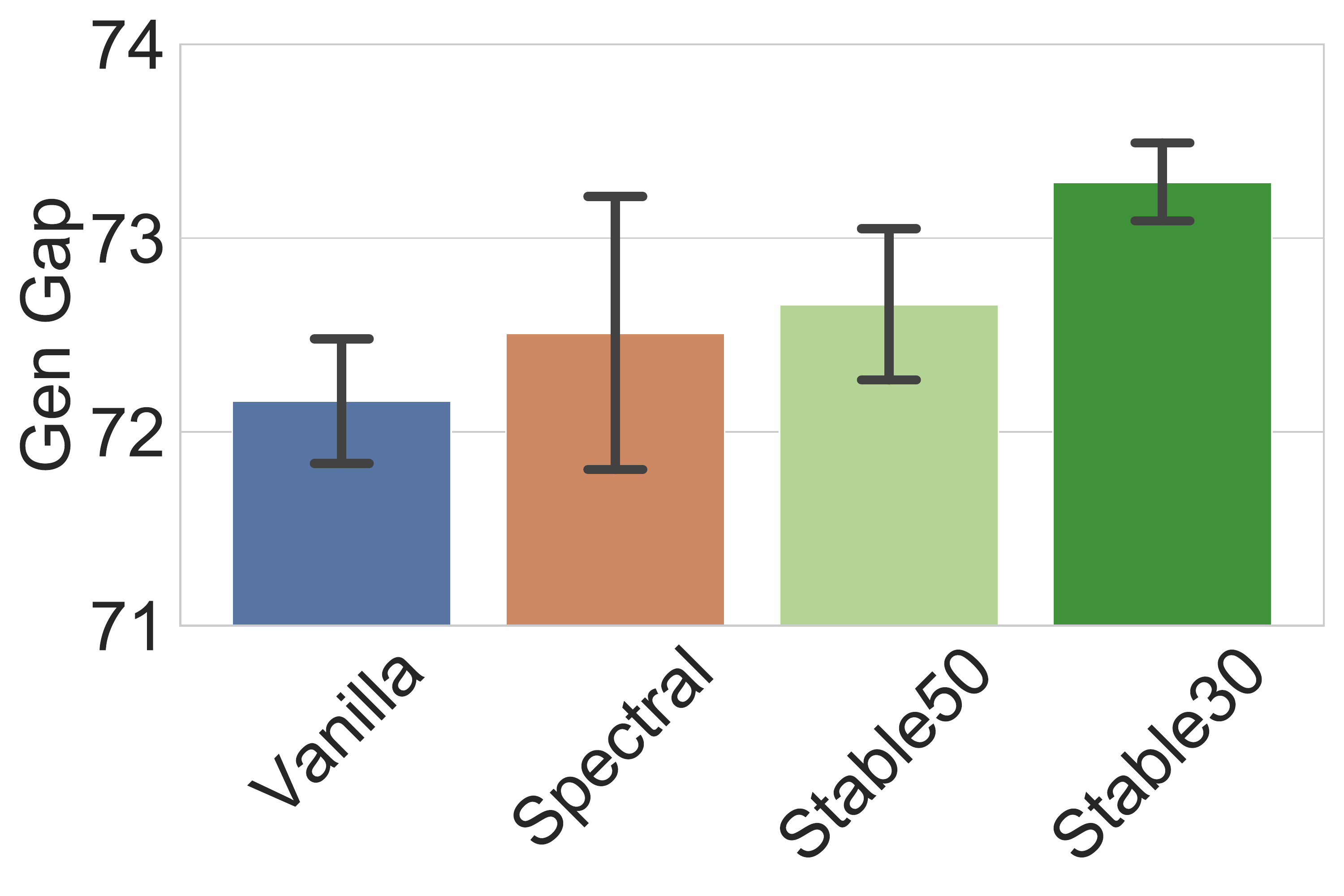
    \subcaption{ Resnet110}
  \end{subfigure}
  \begin{subfigure}[!t]{0.19\linewidth}
    \def\svgwidth{0.98\linewidth}
    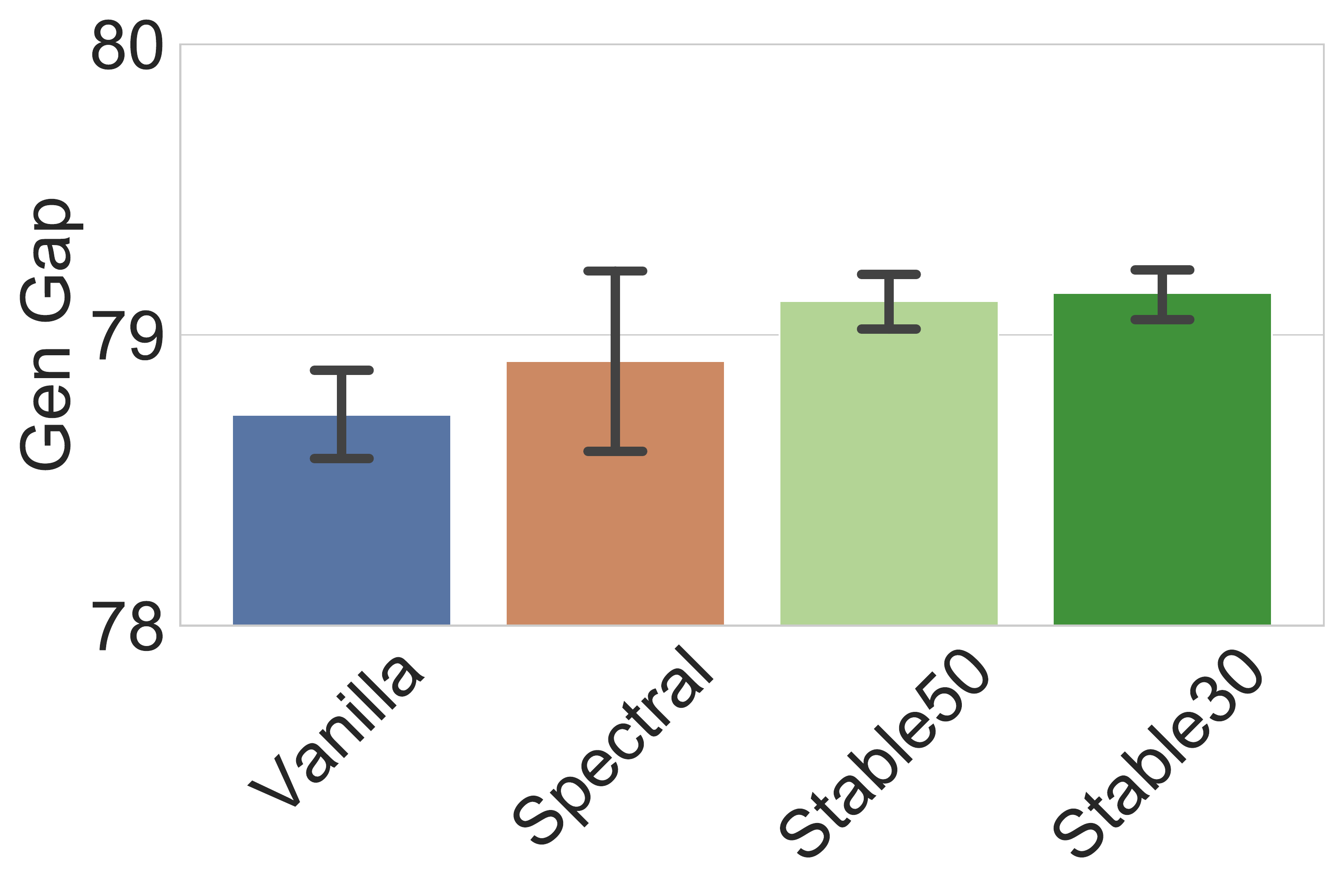
    \subcaption{ WideResnet-28}
  \end{subfigure}
   \begin{subfigure}[!t]{0.19\linewidth}
    \def\svgwidth{0.98\linewidth}
    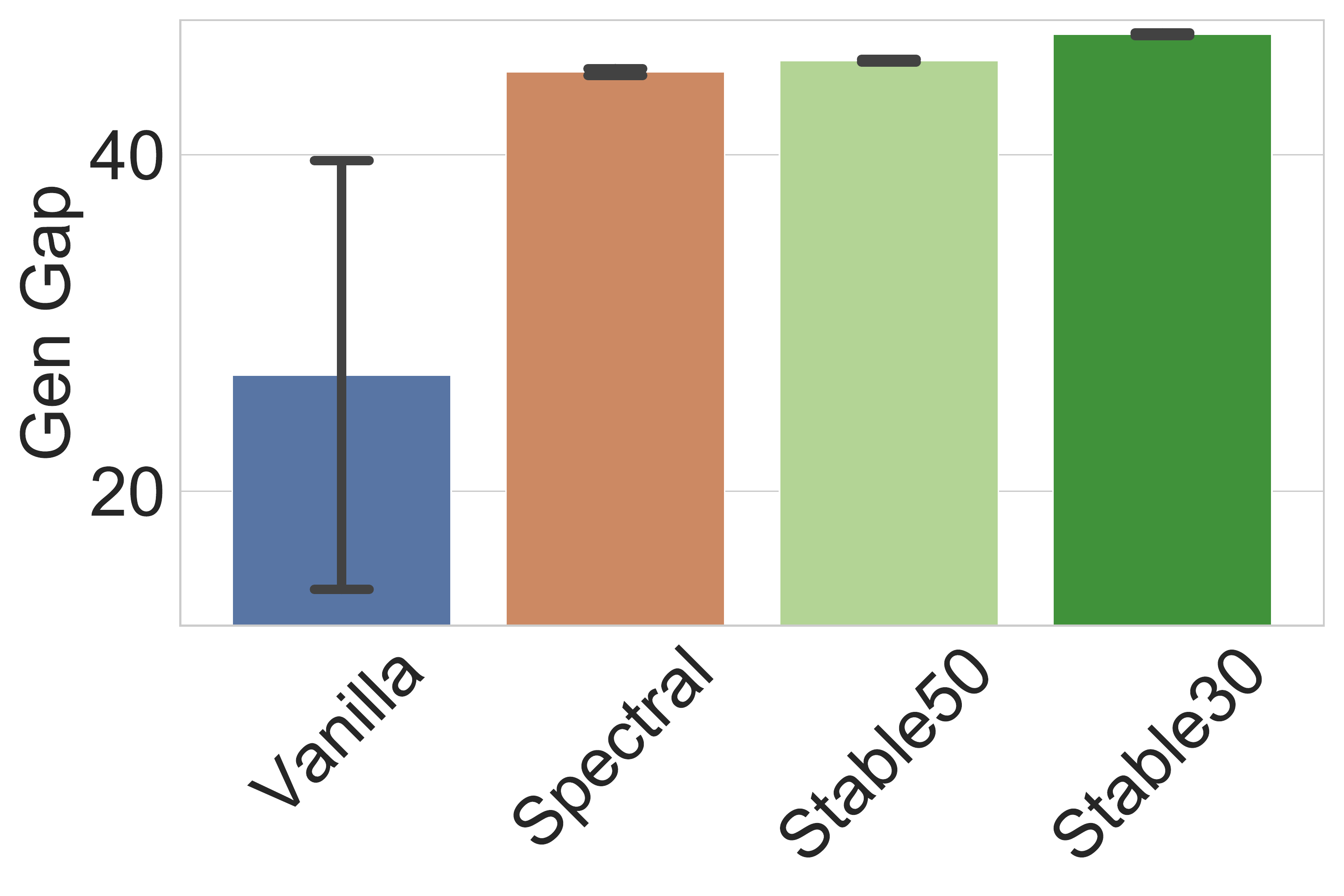
    \subcaption{ Alexnet}
  \end{subfigure}
  \begin{subfigure}[!t]{0.19\linewidth}
    \def\svgwidth{0.98\linewidth}
    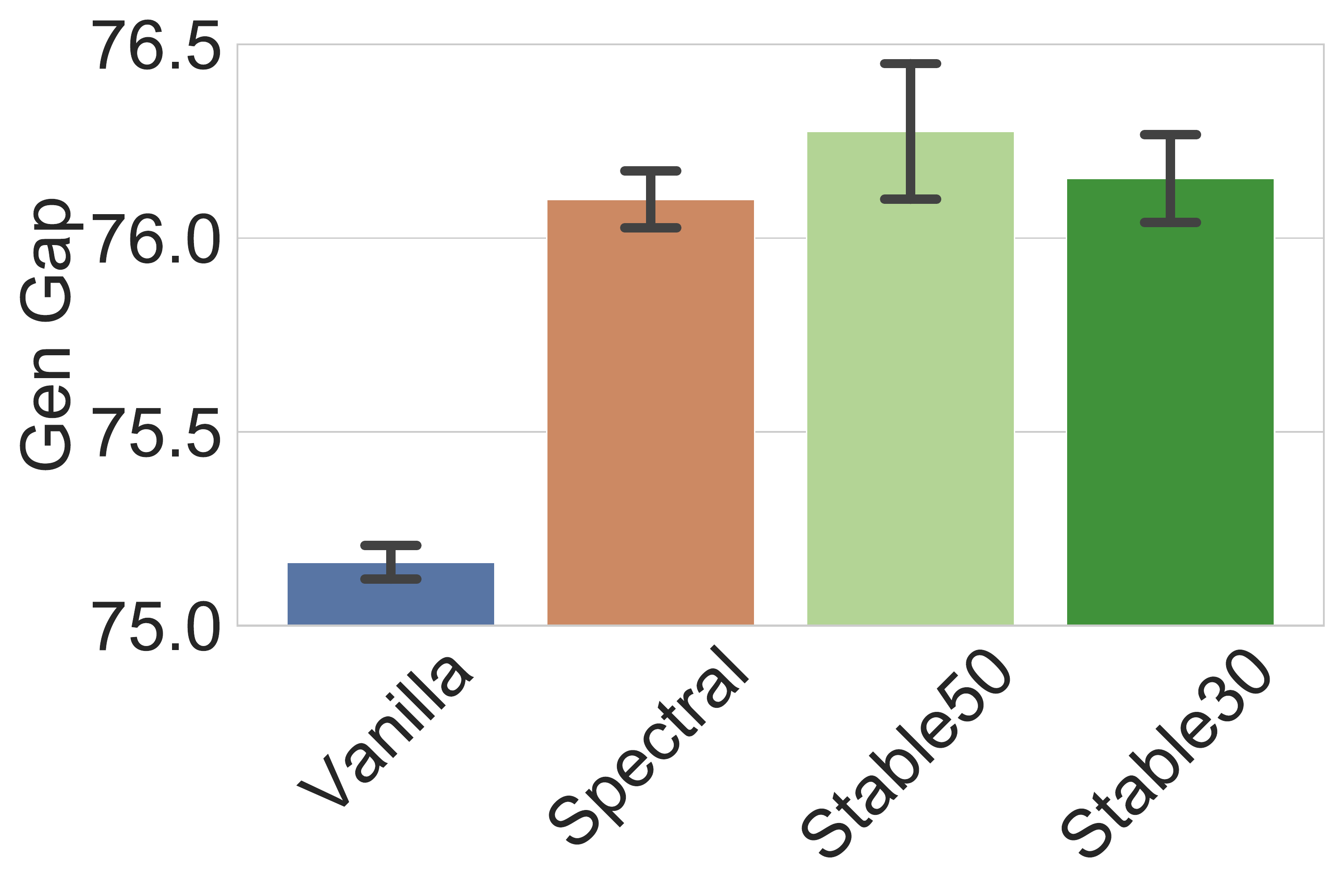
    \subcaption{ Densenet-100}
  \end{subfigure}
   \begin{subfigure}[!t]{0.19\linewidth}
    \def\svgwidth{0.98\linewidth}
    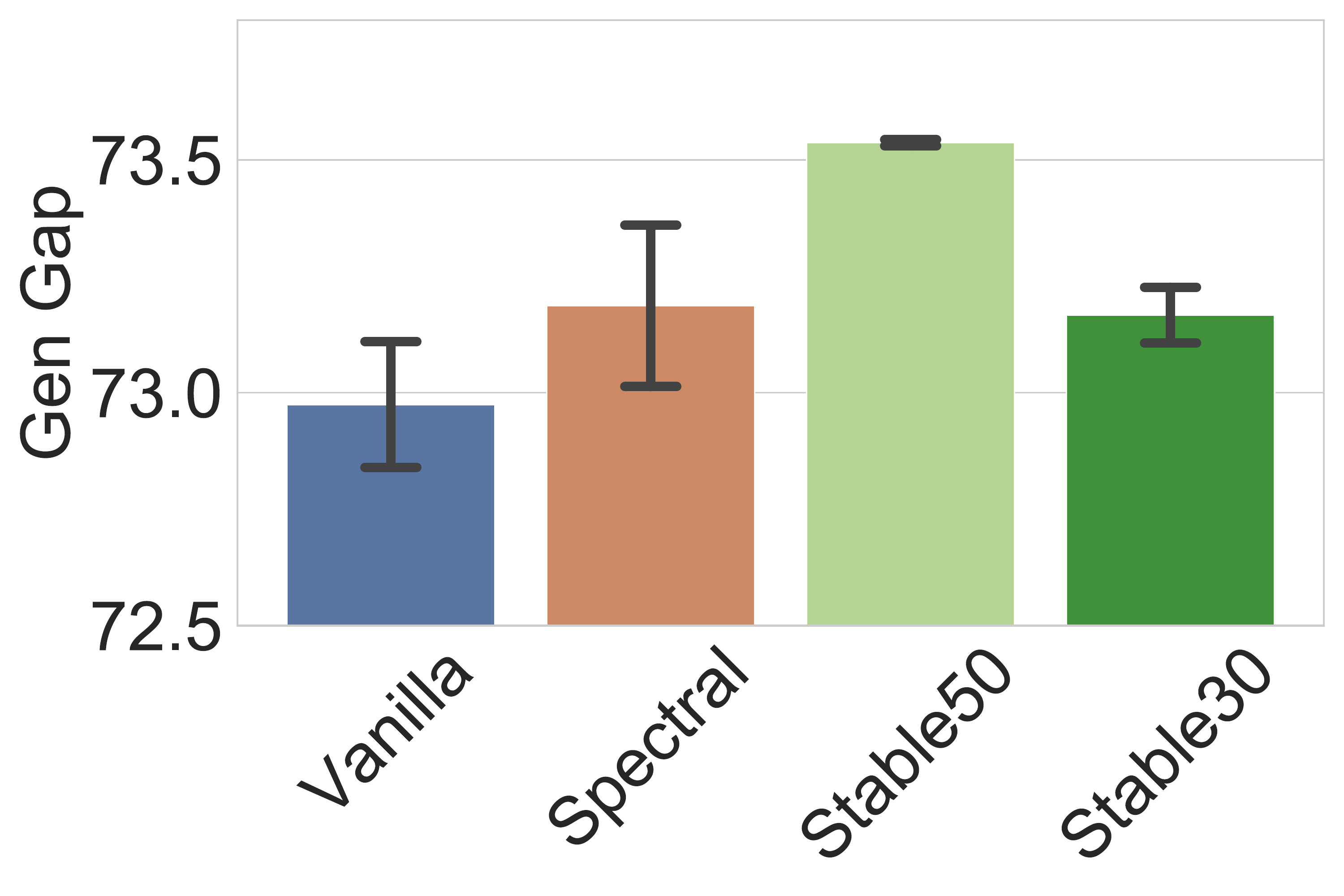
    \subcaption{ VGG-19}
  \end{subfigure}
  \caption{Test accuracies on CIFAR100 for clean data using a stopping
    criterion based on train accuracy. Higher is better.}
  \label{fig:test-acc-stop}
\end{figure}

\paragraph{CIFAR10 experiments}
In this section, we plot results on CIFAR10 trained using ResNet-110,
Desnenet100, WideResNet-28, and Alexnet. In~\cref{fig:test-acc-stop-c10}, we plot the test accuracy on clean
CIFAR-10 with the training accuracy as the stopping criterion. For all
models other than Alexnet, we use $99\%$ training accuracy as the
criterion and for Alexnet we use
$85\%$. In~\cref{fig:test-acc-epoch-c10}, we plot the test accuracy on
clean CIFAR10 using the number of epochs as the stopping criterion.The results here are consistent with those in the main
paper in that SRN outperforms the vanilla and SN.

\begin{figure}[t]
  \centering\small
  \begin{subfigure}[!t]{0.24\linewidth}
    \def\svgwidth{0.98\linewidth}
    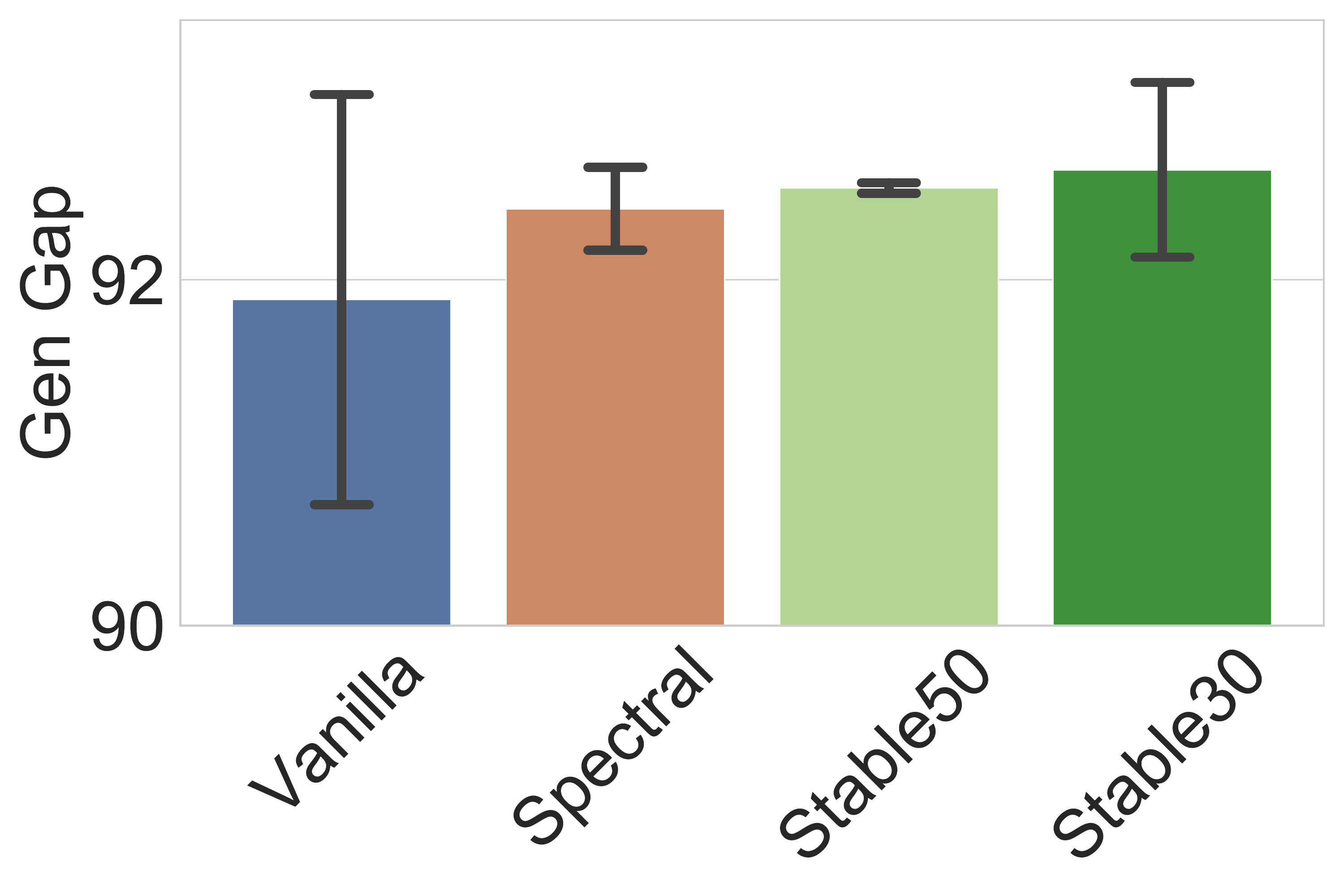
    \subcaption{ Resnet110}
  \end{subfigure}
  \begin{subfigure}[!t]{0.24\linewidth}
    \def\svgwidth{0.98\linewidth}
    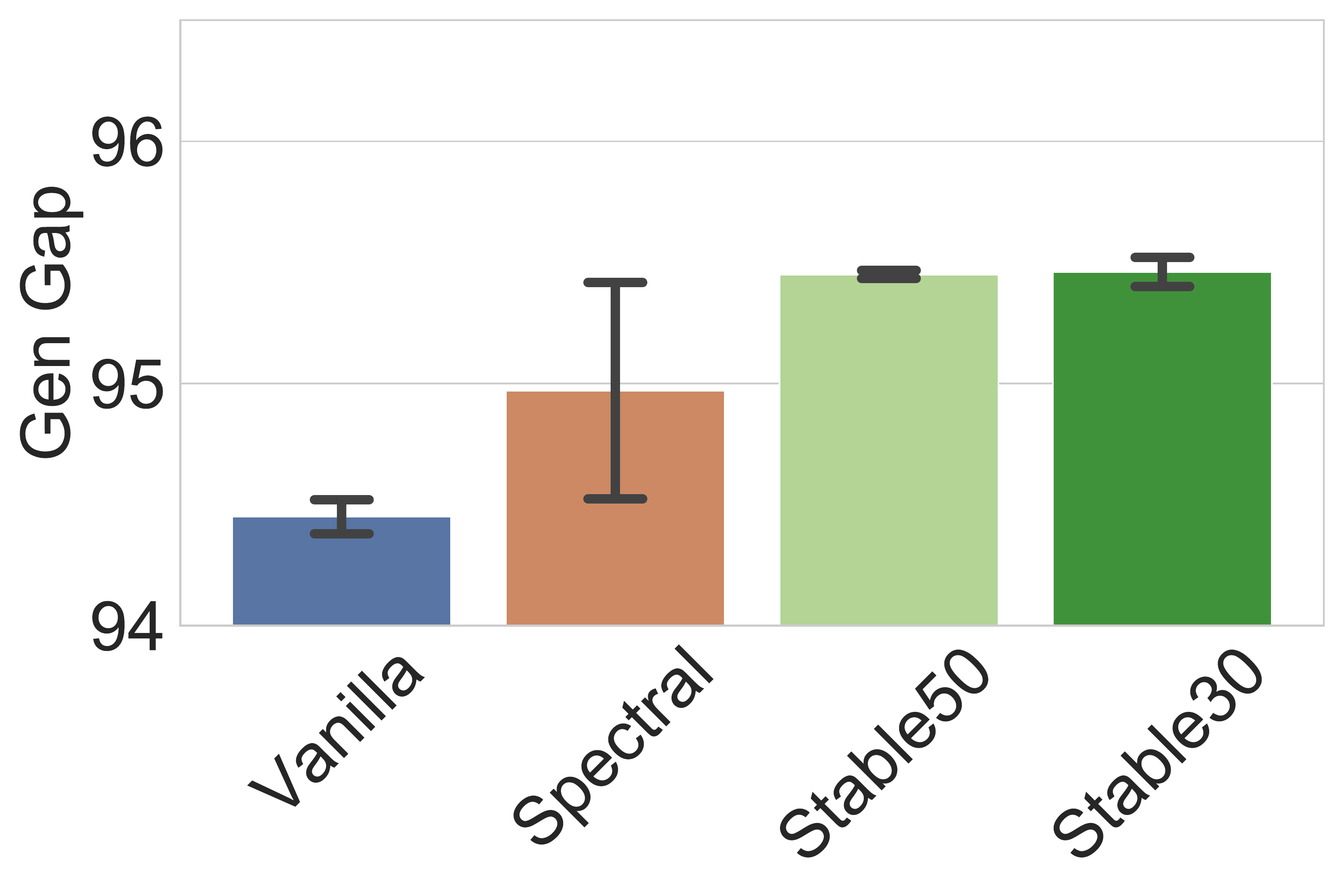
    \subcaption{ WideResnet-28}
  \end{subfigure}
   \begin{subfigure}[!t]{0.24\linewidth}
    \def\svgwidth{0.98\linewidth}
    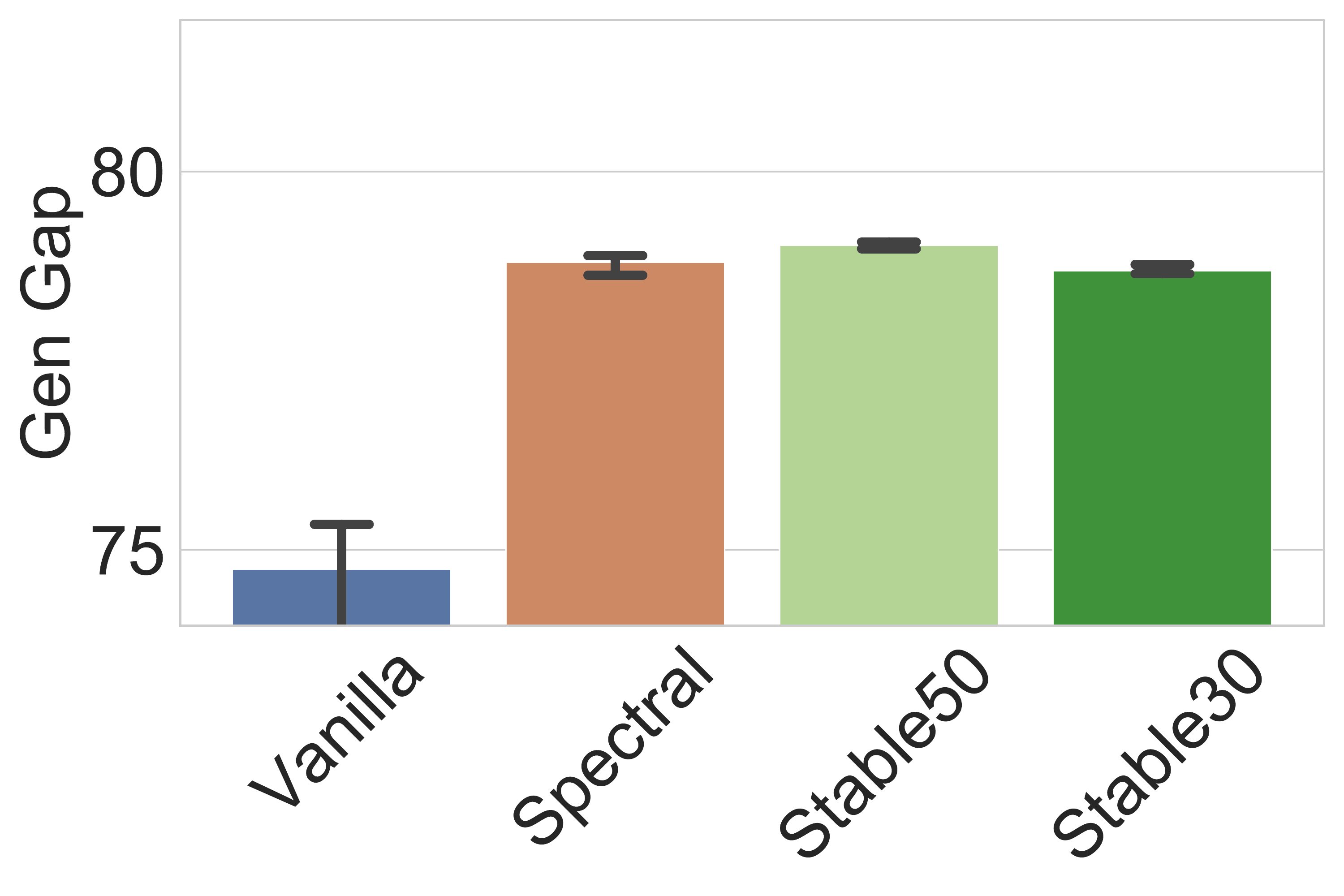
    \subcaption{ Alexnet}
  \end{subfigure}
  \begin{subfigure}[!t]{0.24\linewidth}
    \def\svgwidth{0.98\linewidth}
    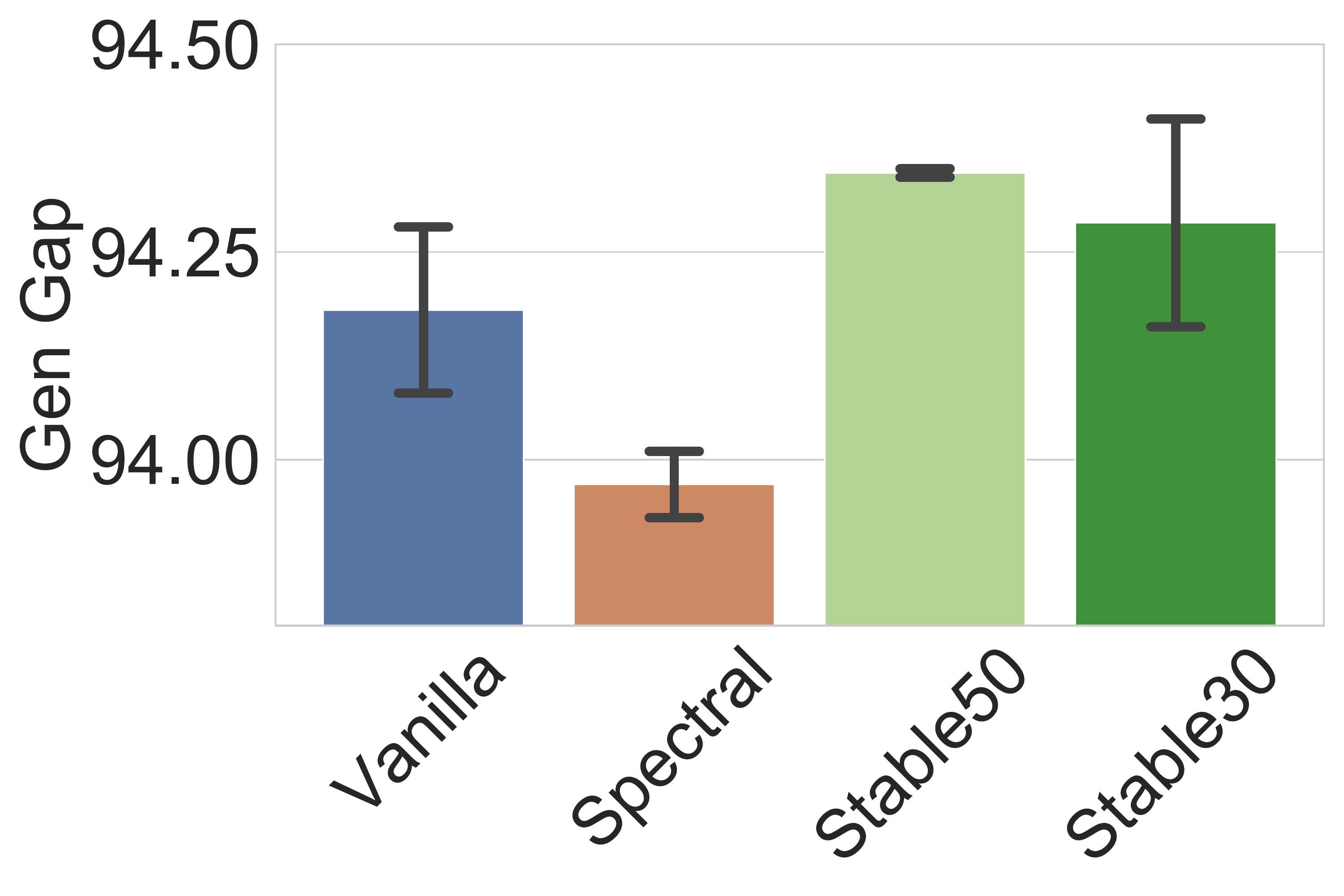
    \subcaption{ Densenet-100}
  \end{subfigure}
  \caption{Test accuracies on CIFAR10 for clean data using a stopping
    criterion based on train accuracy. Higher is better.}
  \label{fig:test-acc-stop-c10}
\end{figure}

\begin{figure}[t]
  \centering\small
  \begin{subfigure}[!t]{0.24\linewidth}
    \def\svgwidth{0.98\linewidth}
    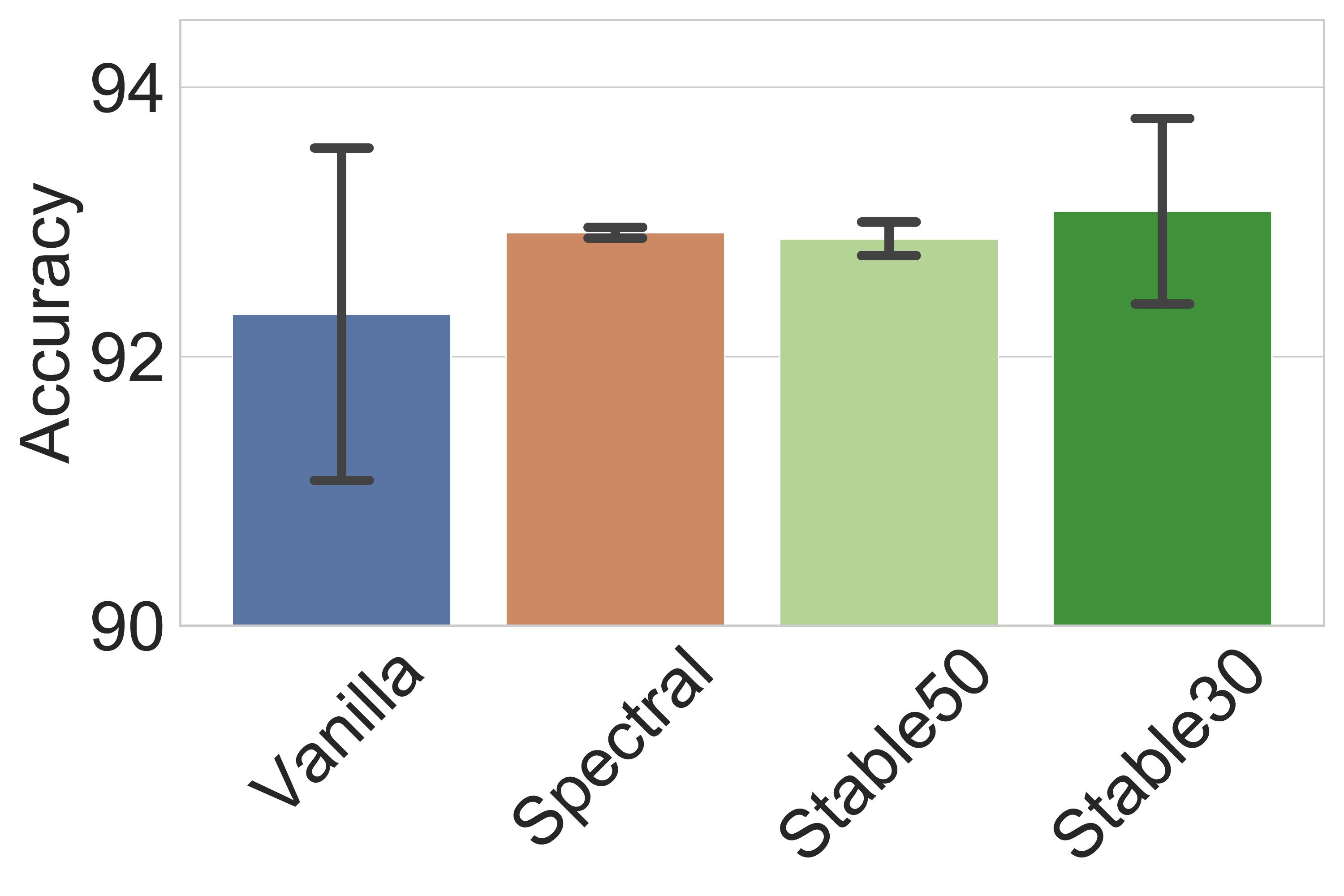
    \subcaption{ Resnet110}
  \end{subfigure}
  \begin{subfigure}[!t]{0.24\linewidth}
    \def\svgwidth{0.98\linewidth}
    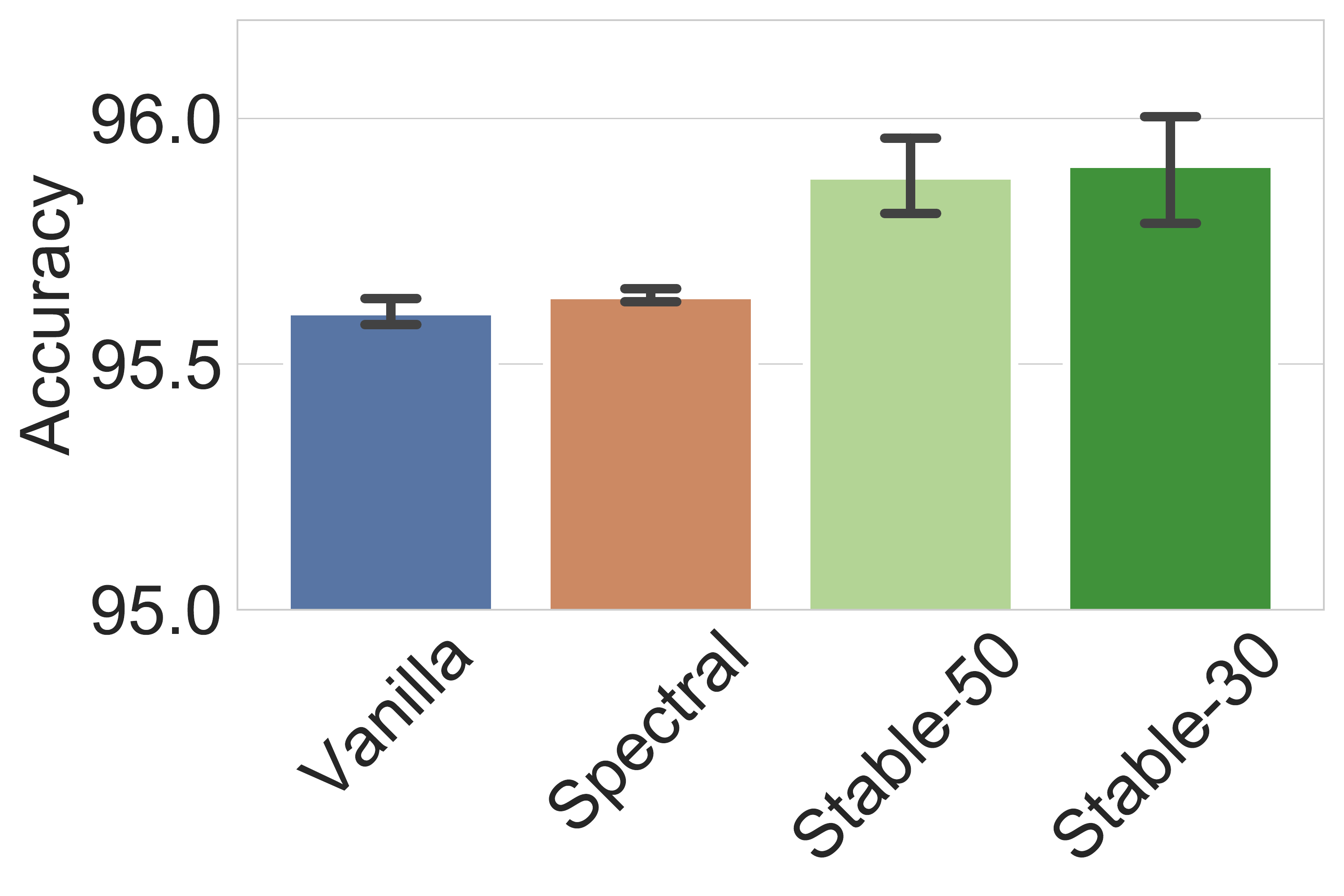
    \subcaption{ WideResnet-28}
  \end{subfigure}
   \begin{subfigure}[!t]{0.24\linewidth}
    \def\svgwidth{0.98\linewidth}
    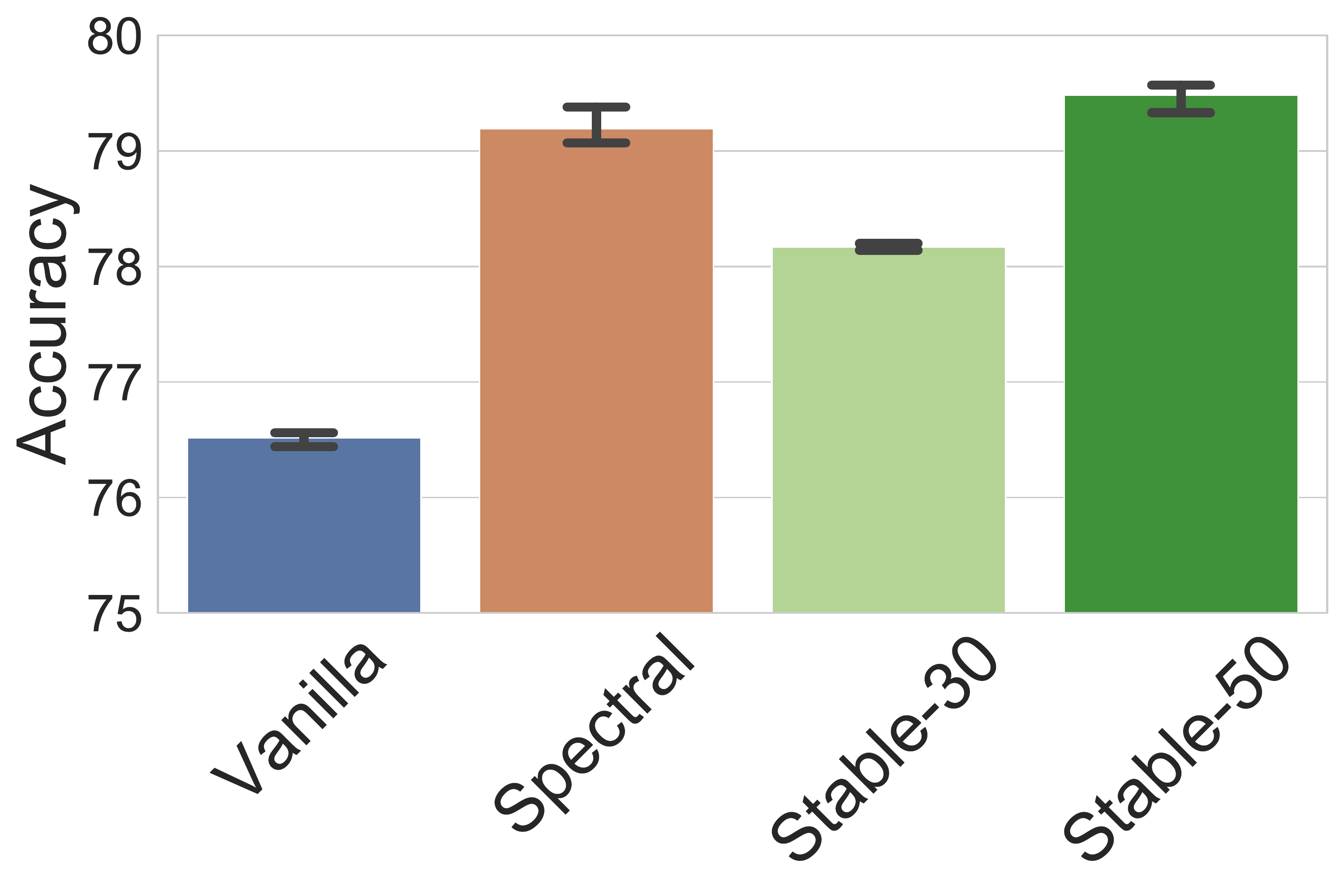
    \subcaption{ Alexnet}
  \end{subfigure}
  \begin{subfigure}[!t]{0.24\linewidth}
    \def\svgwidth{0.98\linewidth}
    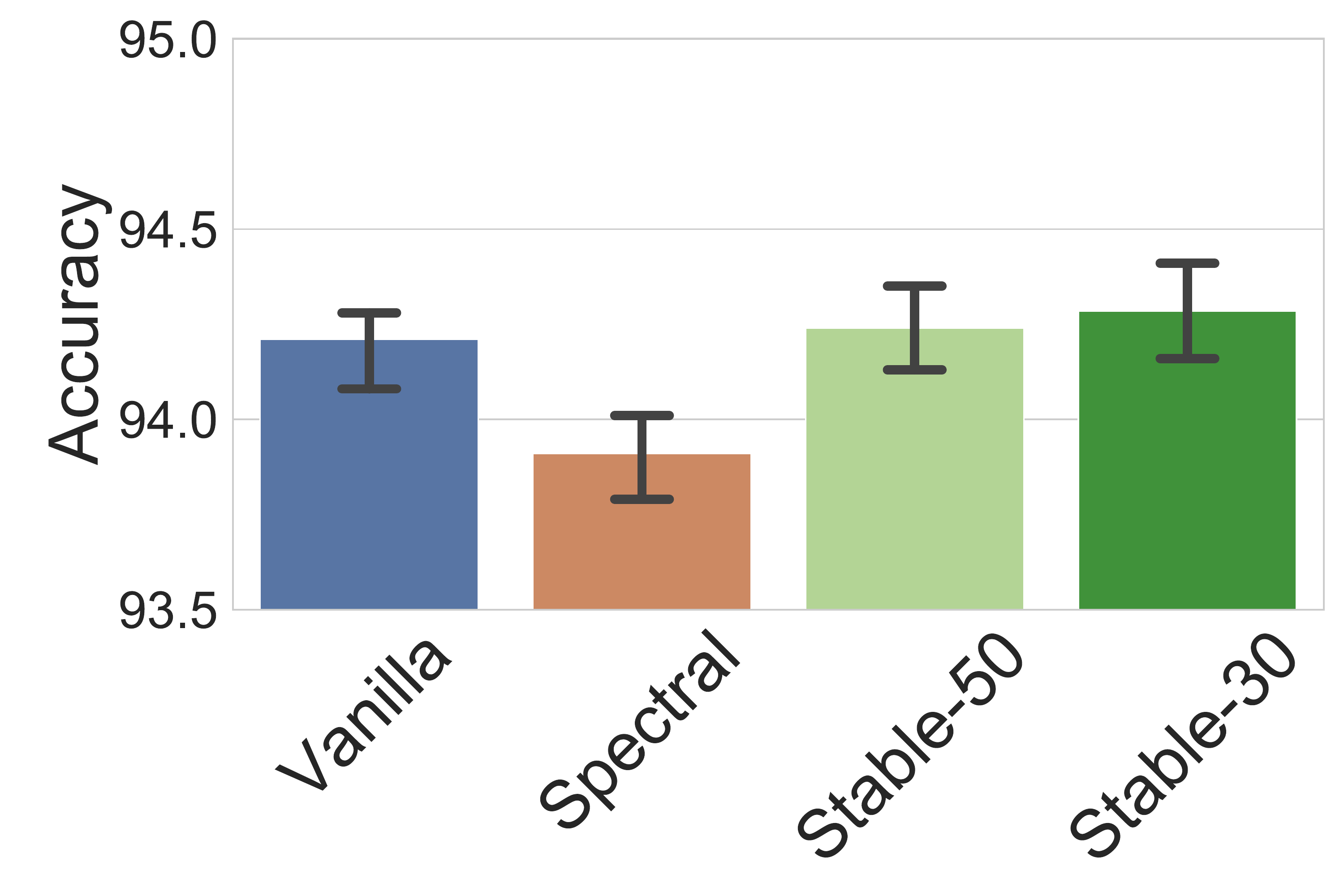
    \subcaption{ Densenet-100}
  \end{subfigure}
  \caption{Test accuracies on CIFAR10 for clean data using the number
    of epochs as a stopping
    criterion. Higher is better.}
  \label{fig:test-acc-epoch-c10}
\end{figure}

In~\cref{fig:cifar-10-gen}, we plot the training accuracy on CIFAR10
when the labels are randomized for Resnet100, and Alexnet. SRN-50 and
SRN-30 are much better than Vanilla and SN in this case.

\begin{figure}[H]
  \centering
  \begin{subfigure}[t]{0.33\linewidth}
    \def\svgwidth{0.98\linewidth}
    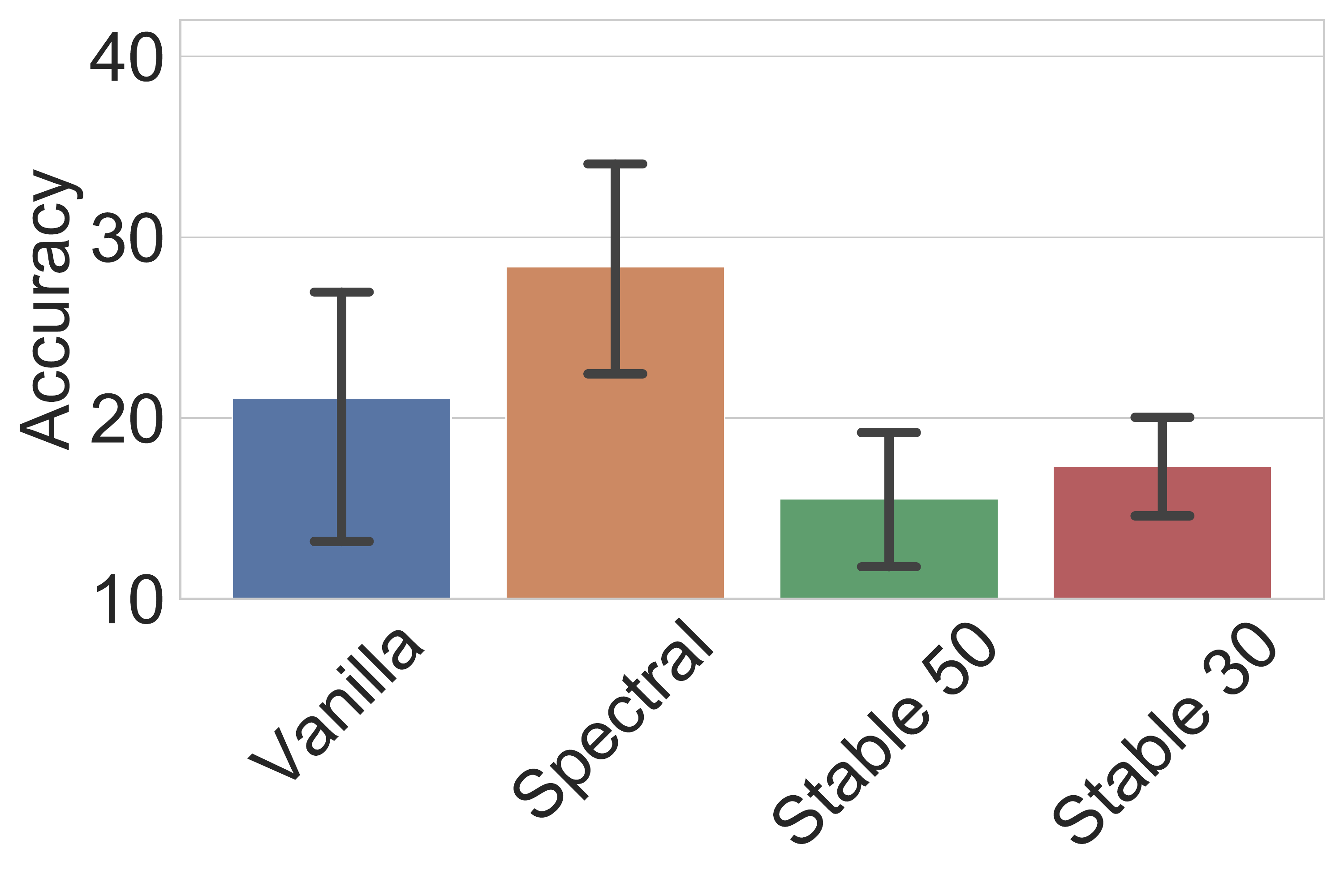
    \caption{Resnet110}\label{fig:r100-rand-train-c10}
  \end{subfigure}
  \begin{subfigure}[t]{0.33\linewidth}
    \def\svgwidth{0.98\linewidth}
    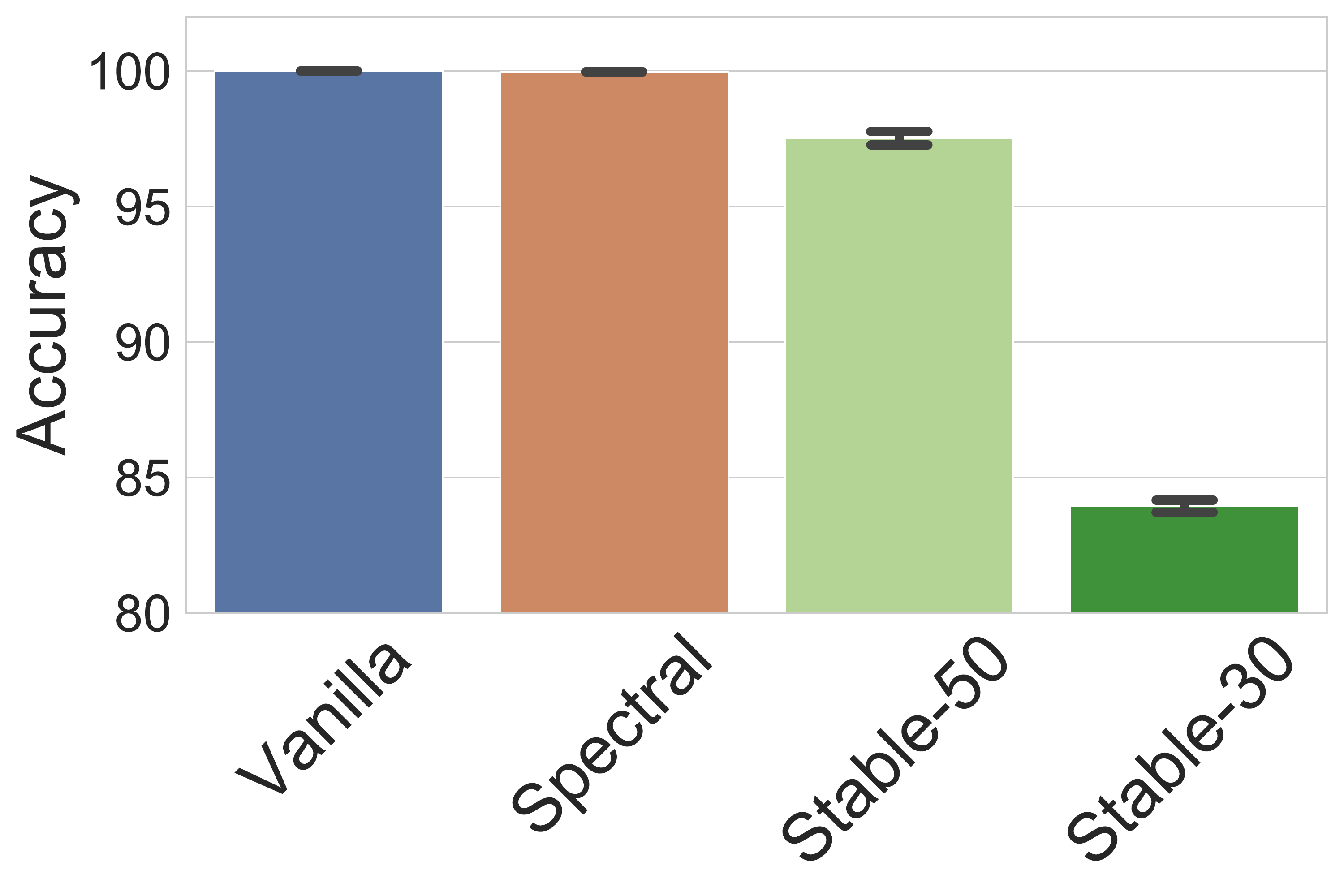
    \caption{Alexnet}\label{fig:alex-rand-train-c10}
  \end{subfigure}
  \caption{Training accuracy on randomly labelled CIFAR-10~(Lower
    is better).}
  \label{fig:cifar-10-gen}
\end{figure}

\section{Additional Experiments on \gls{gan}s}
\subsection{GAN experimental setup}
\label{sec:gansetup}
\paragraph{Datasets and Network Architectures}
Each of the CIFAR datasets contain a total of $50,000$ RGB images in the training set, where each image is of size  $32\times 32$,  and a further $10,000$ RGB images of the same dimension in the test set. The CelebA dataset contains more than 200K images scaled to a size of $64\times64$. The model architecture for  both the generator and the discriminator was chosen to be a 32 layered ResNet~\citep{HZRS:2016} due to its previous superior performance in other works~\citep{miyato2018spectral}. %
We use Adam optimizer~\citep{kingma2014adam} which depends on three main hyper-parameters $\alpha$- the initial learning rate, $\beta_1$- the first order moment decay rate and $\beta_2$- the second order moment decay rate. We cross-validate these parameters in the set $\alpha\in\{0.0002, 0.0005\},~\beta_1\in\{0, 0.5\},~\beta_2\in\{0.9, 0.999\}$ and chose $\alpha=0.0002$, $\beta_1=0.0$ and $\beta_2=0.999$ which performed consistently well in all of the experiments.

\paragraph{GAN objective functions}
In the case of conditional GANs~\citep{Mirza2014}, we used the conditional batch normalization~\citep{dumoulin2017learned} to condition the generator and the projection discriminator~\citep{Miyato2018} to condition the discriminator. The dimension of the latent variable for the generator was set to $128$ and was sampled from a zero mean and unit variance Gaussian distribution.  For training the model, we used the hinge loss version of the adversarial loss~\citep{Lim2017,Tran2017} in all experiments except the experiments with WGAN-GP. The hinge loss version was chosen  as it has  been shown to give consistently better performance in previous works~\citep{Zhang2018, miyato2018spectral}. For training the WGAN-GP model, we used the original loss function as described in~\citet{Gulrajani2017}. 

\paragraph{Evaluation Metrics}
We use \textit{Inception}~\citep{salimans2016improved}  and \textit{Frechet Inception Distance}~(FID)~\citep{heusel2017gans} scores for the evaluation of the generated samples. For measuring the inception score, we generate $50,000$ samples, as was recommended in~\citet{salimans2016improved}. For measuring FID, we use the same setting as~\citet{miyato2018spectral} where we sample $10,000$ data points from the training set and compare its statistics with that of $5,000$ generated samples. In addition, we use a recent evaluation metric called Neural divergence score~\citet{gulrajani2018towards} which is more robust to memorization. The exact set-up for the same is discussed below. In the case of conditional image generation, we also measure Intra-FID~\citep{miyato2018spectral}, which is the mean of the FID of the generator, when it is conditioned over different classes. Let $\mathrm{FID}(\mathcal{G}, c)$ be the FID of the generator $\mathcal{G}$ when it is conditioned on the class $c \in \mathcal{C}$ (where $\mathcal{C}$ is the set of classes), then, $\mathrm{Intra~FID}(\mathcal{G}) = \frac{1}{|\mathcal{C}|}\mathrm{FID}(\mathcal{G}, c)$

\paragraph{Neural Divergence Setup}
We train a new classifier inline with the architecture
in~\citet{gulrajani2018towards}. It includes three convolution layers
with 16, 32 and 64 channels, a kernel size of $5\times~5$ and a stride
of $2$. Each of these layers are followed by a Swish
activation~\citep{ramachandran2018searching} and then finally a linear
layer that gives a single output. The network is initialized using normal distribution with zero mean and the standard
deviation of $0.02$, and trained using Adam optimizer with $\alpha=0.0002,~\beta_1=0.,~\beta_2=0.9$ for a total of $100,000$ iterations with minibatch of $128$ generated samples and $128$ samples from the test set\footnote{For CelebA, we used the training set.}. We use the standard WGAN-GP loss function, $\log\br{1 + \exp\br{f\br{\vec{x}_{\mathrm{fake}}}}} + \log{\br{1 + \exp\br{-\vec{x}_{\mathrm{real}}}}}$, where $f$ represents the network described above. Finally, we generate $1~\mathrm{Million}$ samples from the generator and report the average $\log\br{1 + \exp\br{f\br{\vec{x}_{\mathrm{fake}}}}}$ over these samples. Higher average value implies better generation as the network in this case is unable to distinguish the generated and the real samples. %
\subsection{More Empirical Lipschitz plots}
\label{sec:lipsch-cond-gans}

For the purpose of analysis,~\cref{fig:lip_rank_stable_uncond_only_fake,fig:lip_rank_stable_uncond_only_real_uncond}
shows eLhist for pairs where each sample either comes from the true data or from the generator, and we
observe a similar trend.
To verify that same results hold in the conditional
setup, we show comparisons for~\gls{gan}s with projection
discriminator~\citep{Miyato2018}
in~\cref{fig:lip_rank_stable_cond,fig:lip_rank_stable_cond_only_fake,fig:lip_rank_stable_cond_only_real},
and observe a similar trend. Further, to see the value of the local Lipschitzness in the
vicinity of real and generated  samples we also plot the norm of the
Jacobian in~\cref{fig:lip_rank_stable_only_fake_vicinity,fig:lip_rank_stable_only_real_vicinity} in
Appendix~\ref{sec:lipsch-cond-gans} and  observe mostly a similar
trend. In~\cref{sec:noise-stability} (\cref{fig:disc_training_stability}), we also show that the
discriminator training of SRN-GAN is more stable than SN-GAN.

\paragraph{Conditional GANs}~\cref{fig:lip_rank_stable_cond} shows the eLihst of conditional \gls{gan}s with projection discriminator~\citep{Miyato2018}.
\begin{figure}[H]
  \centering
  \def\svgwidth{0.99\columnwidth}
  \resizebox{0.7\textwidth}{!}{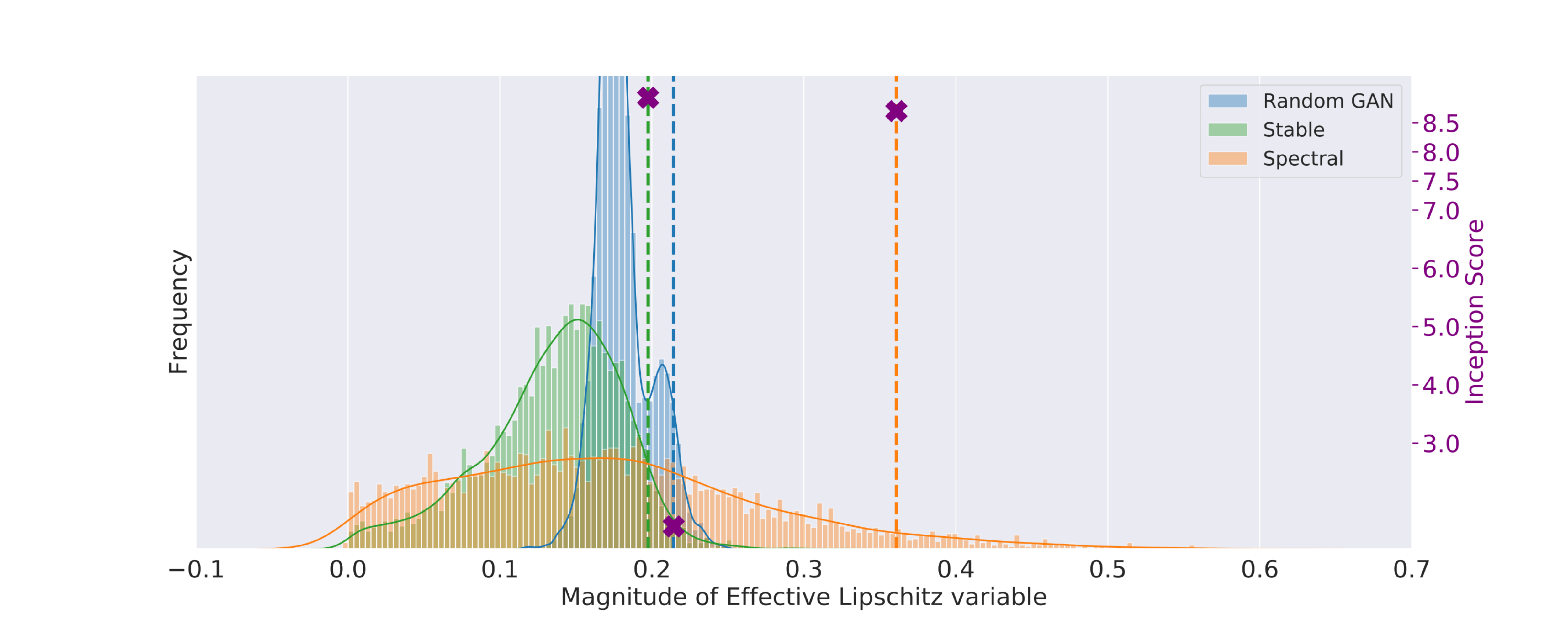} \caption{\footnotesize
    Comparison: {\bf eLhist} of the discriminator in the conditional GAN setting with projection discriminator on CIFAR100.}  \label{fig:lip_rank_stable_cond}
\end{figure}

\paragraph{Empirical Lipschitzness between real samples and between fake samples.}
Figure~\ref{fig:emp_lipschitz_fake_samples_only} shows the histogram
of eLhist of the discriminator for pairs of fake samples i.e. samples
generated by the
generator. Figure~\ref{fig:lip_rank_stable_uncond_only_real} shows
eLhist of the discriminator when samples came from the dataset.

\begin{figure}[!h]\vspace{-1em}
  \begin{subfigure}[t]{0.49\linewidth}
    \centering
  \def\svgwidth{0.99\columnwidth}
  \resizebox{0.95\textwidth}{!}{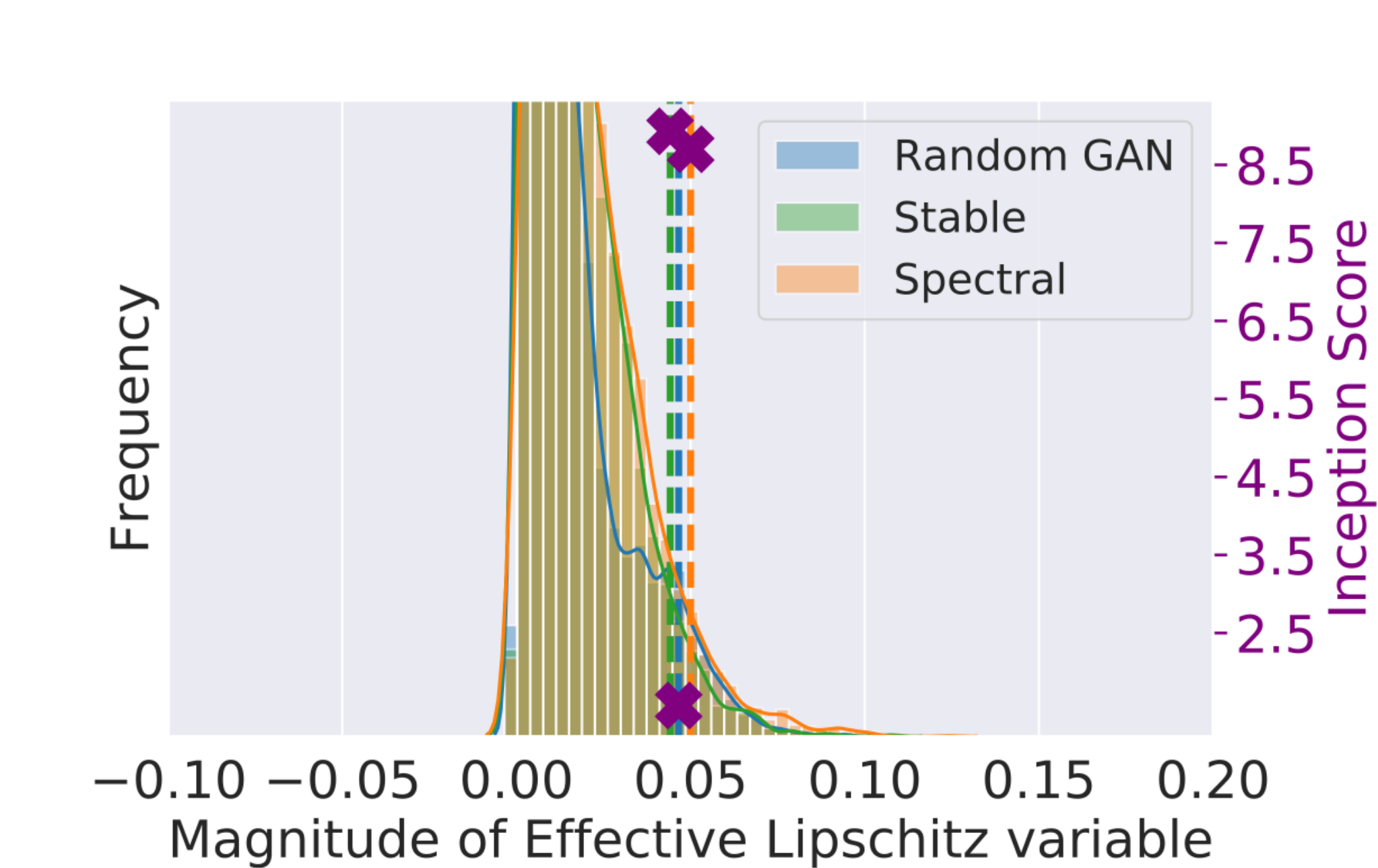} \caption{Conditional GAN  with projection discriminator.}  \label{fig:lip_rank_stable_cond_only_fake}
\end{subfigure}\hfill
\begin{subfigure}[t]{0.49\linewidth}
  \centering 
  \def\svgwidth{0.99\columnwidth}
  \resizebox{0.95\textwidth}{!}{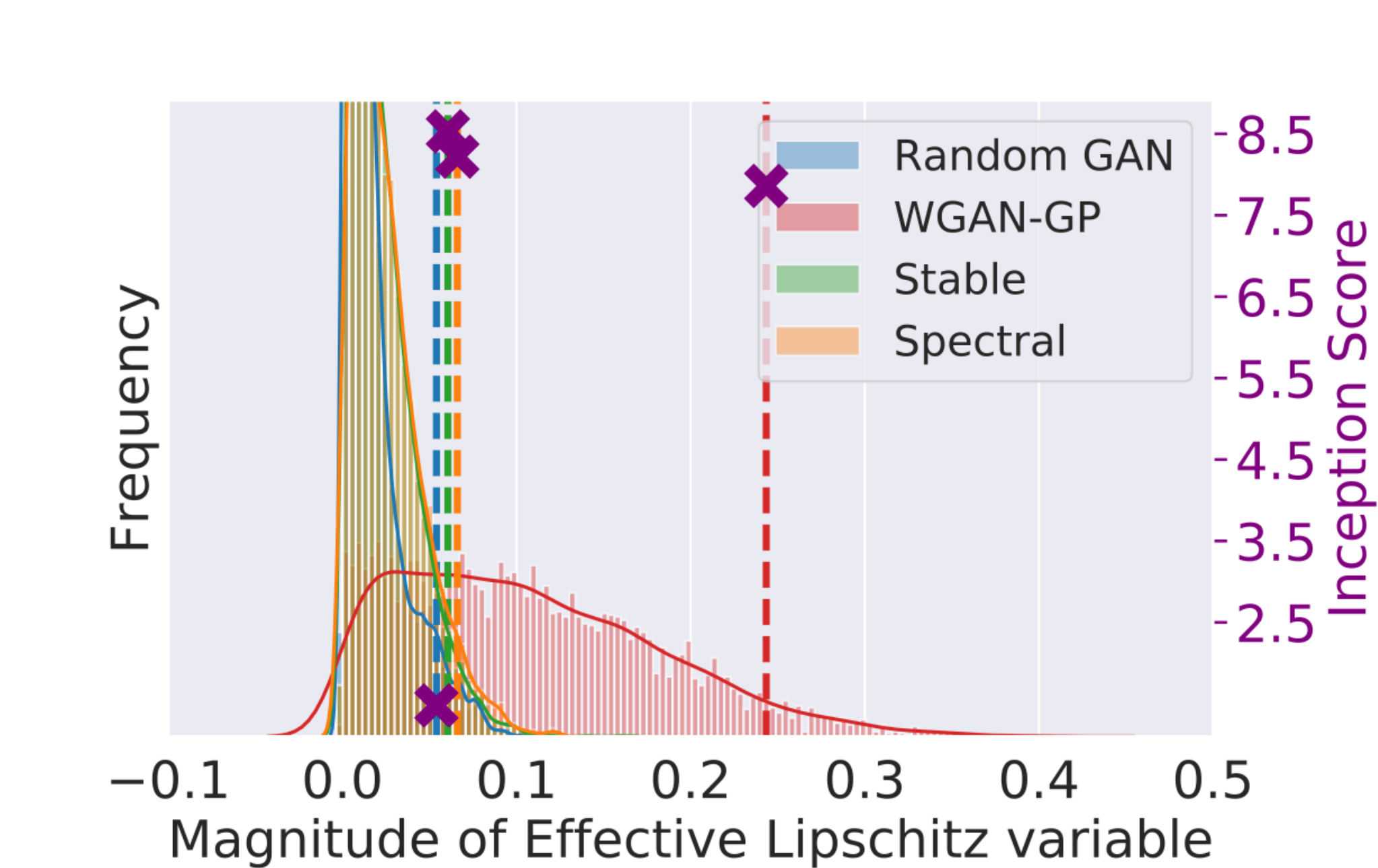}  \caption{Unconditional GAN setting.} \label{fig:lip_rank_stable_uncond_only_fake}
\end{subfigure}
\caption{\footnotesize Comparison: {\bf eLhist} of the discriminator for pairs of samples selected from the generator on CIFAR10}
\label{fig:emp_lipschitz_fake_samples_only}
\end{figure}

\begin{figure}[!h]\vspace{-1em}
\begin{subfigure}[!h]{0.49\linewidth}
  \centering
  \def\svgwidth{0.99\columnwidth}
  \resizebox{0.95\textwidth}{!}{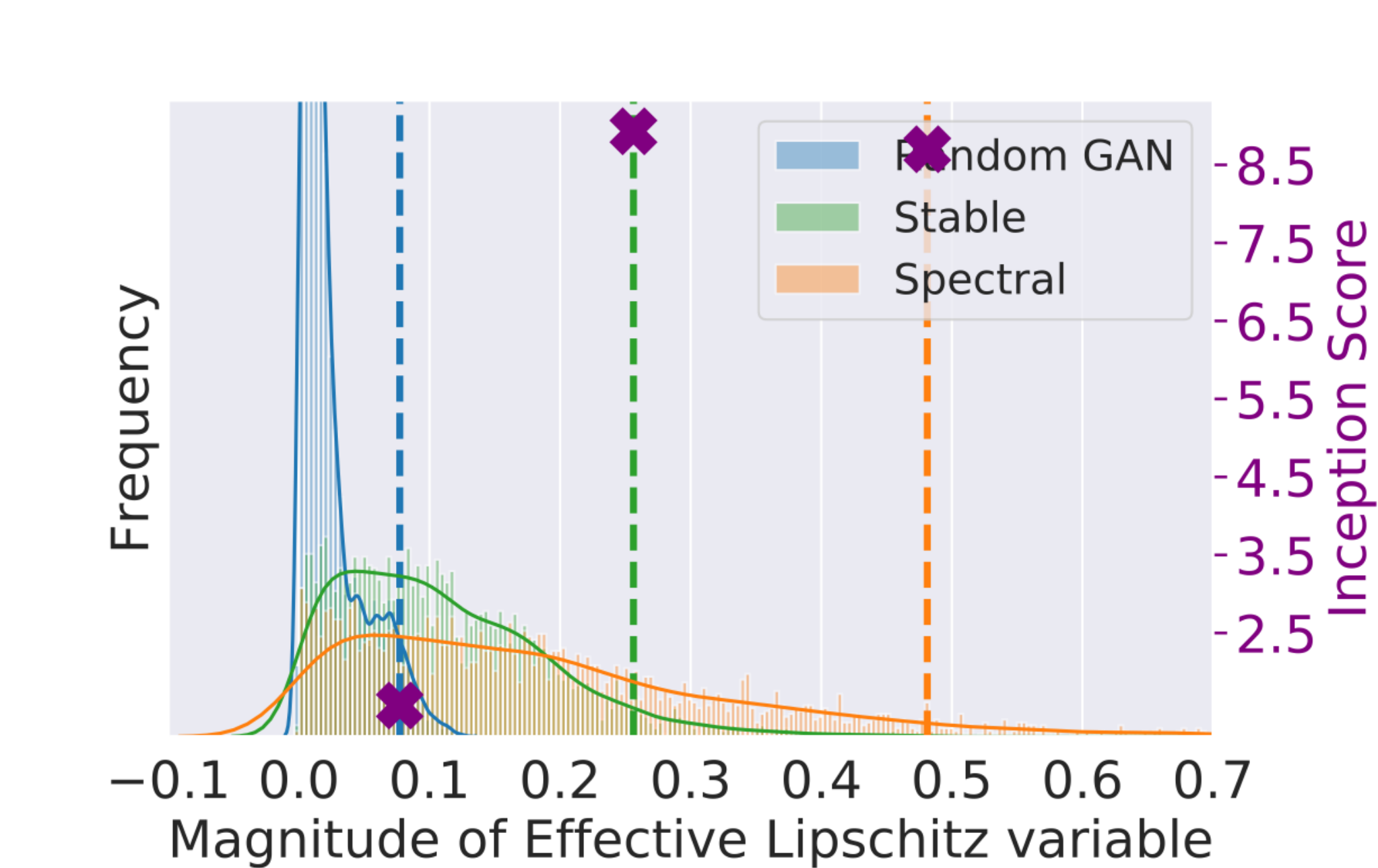} \caption{Conditional GAN with projection discriminator}  \label{fig:lip_rank_stable_cond_only_real}
\end{subfigure}
\begin{subfigure}[!h]{0.49\linewidth}
  \centering
  \def\svgwidth{0.99\columnwidth}
  \resizebox{0.95\textwidth}{!}{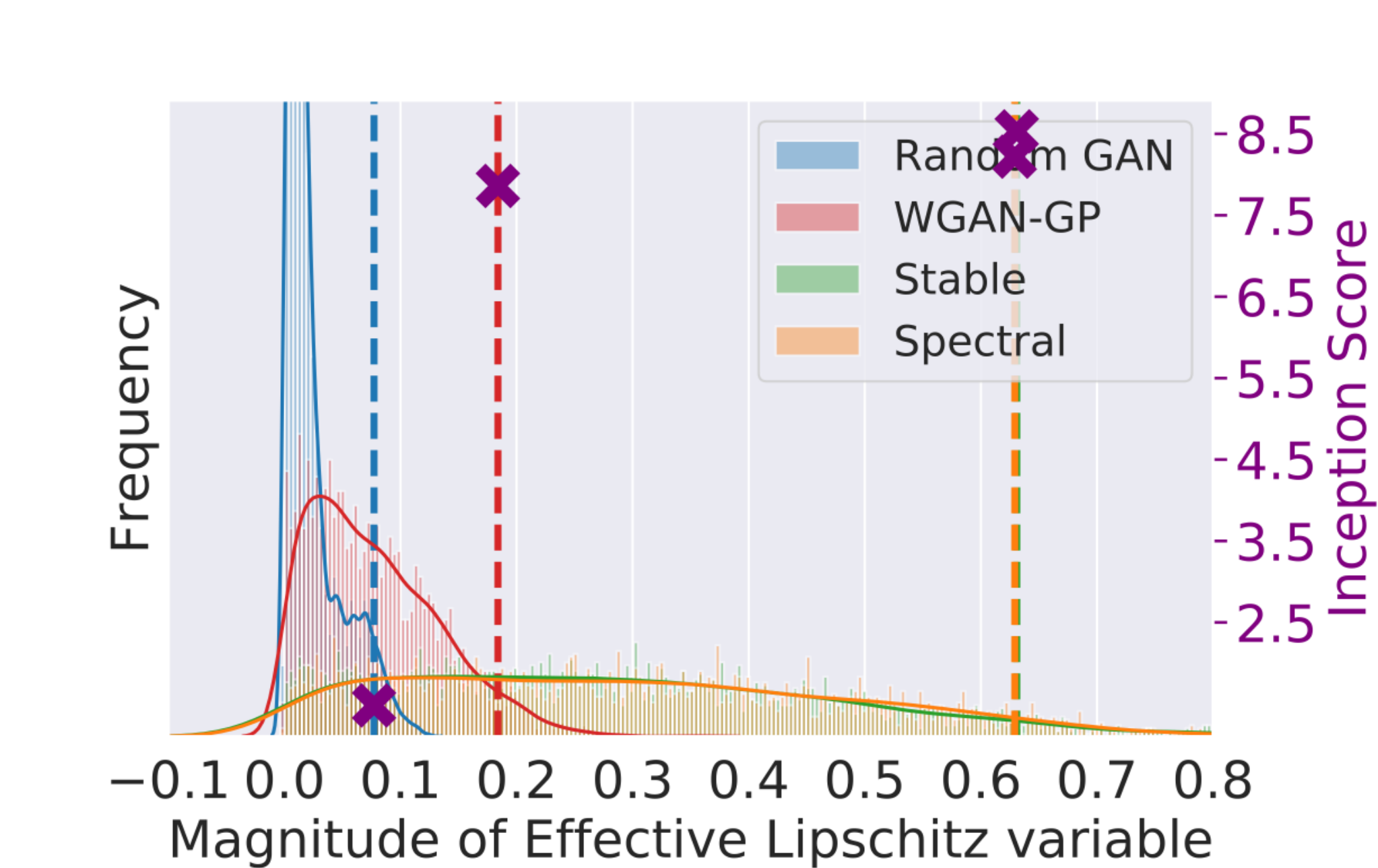} \caption{Unconditional GAN setting.}  \label{fig:lip_rank_stable_uncond_only_real_uncond}
\end{subfigure}
\caption{\footnotesize Comparison: {\bf  eLhist}  of the discriminator for pairs of samples from the real distribution on CIFAR10.} \label{fig:lip_rank_stable_uncond_only_real}
\end{figure}

\paragraph{Jacobian norm in the vicinity of the points}
Here we  compare the Jacobian of the discriminator of the trained
models in the vicinity of the samples from the generator and the
real dataset. This is a penalized measure in various
algorithms~\citet{Gulrajani2017,petzka2018on}~(often referred to as
local perturbations) and was independently proposed by
~\citet{kodali2018on}.~\cref{fig:lip_rank_stable_only_fake_vicinity} and~\cref{fig:lip_rank_stable_only_real_vicinity}
show the histogram of the norm of the Jacobian of the discriminator in
the vicinity of the generated and the real samples, respectively. To
generate these plots, $2,000$  samples were used from the
respective distributions. It is interesting to note that the norm is the same for the points in the vicinity of the
real data points and the generated data points for the \gls{srngan} as
well for WGAN-GP whereas it varies between fake and real samples for \gls{sngan}.

\begin{figure}[!htb]\vspace{-1em}

\begin{subfigure}[!h]{0.5\linewidth}
  \centering
  \def\svgwidth{0.99\columnwidth}
  \resizebox{0.95\textwidth}{!}{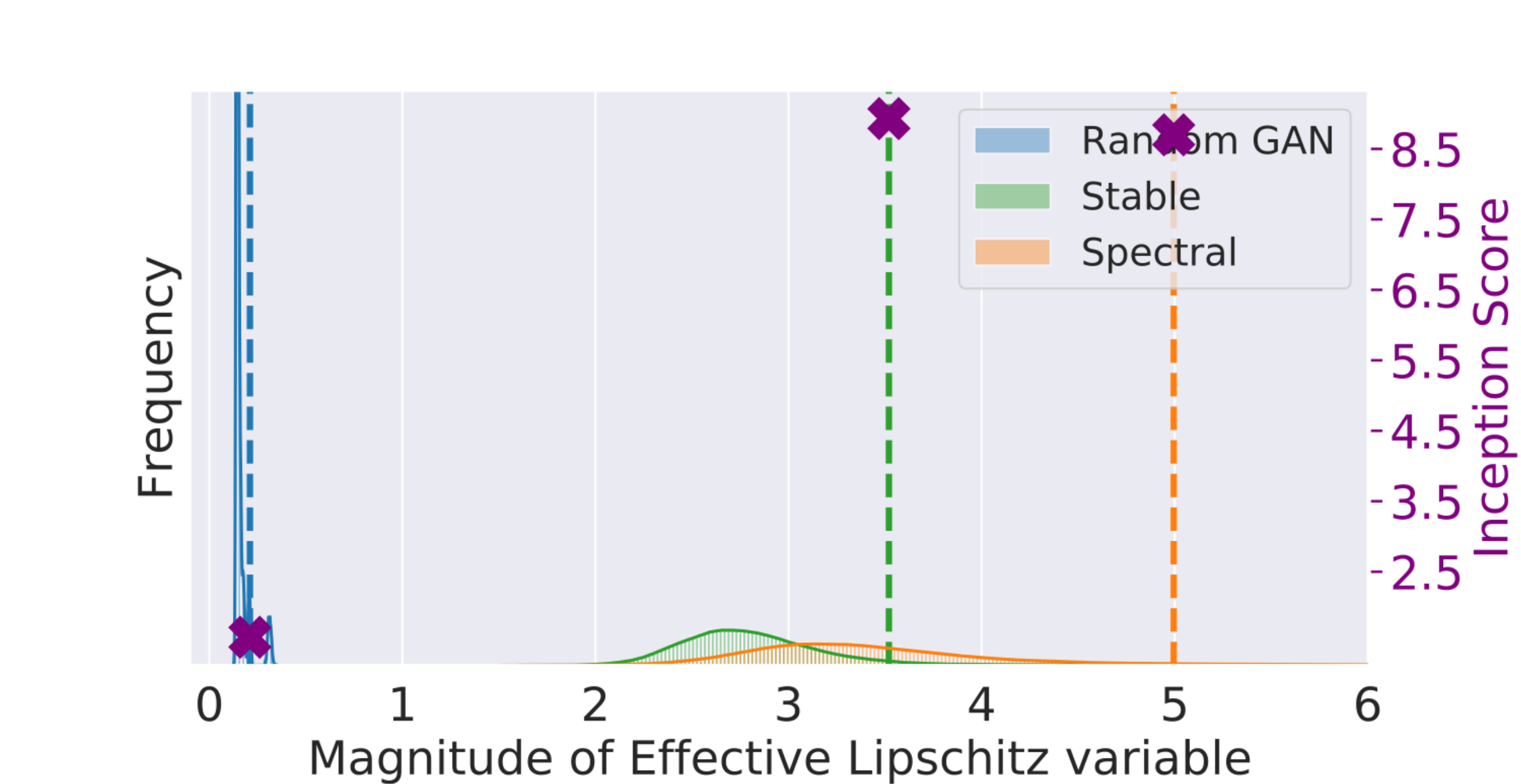}\caption{Conditional GAN with projection discriminator.}  \label{fig:lip_rank_stable_cond_only_fake_vicinity}
\end{subfigure}
\begin{subfigure}[!h]{0.5\linewidth}
  \centering
  \def\svgwidth{0.99\columnwidth}
  \resizebox{0.95\textwidth}{!}{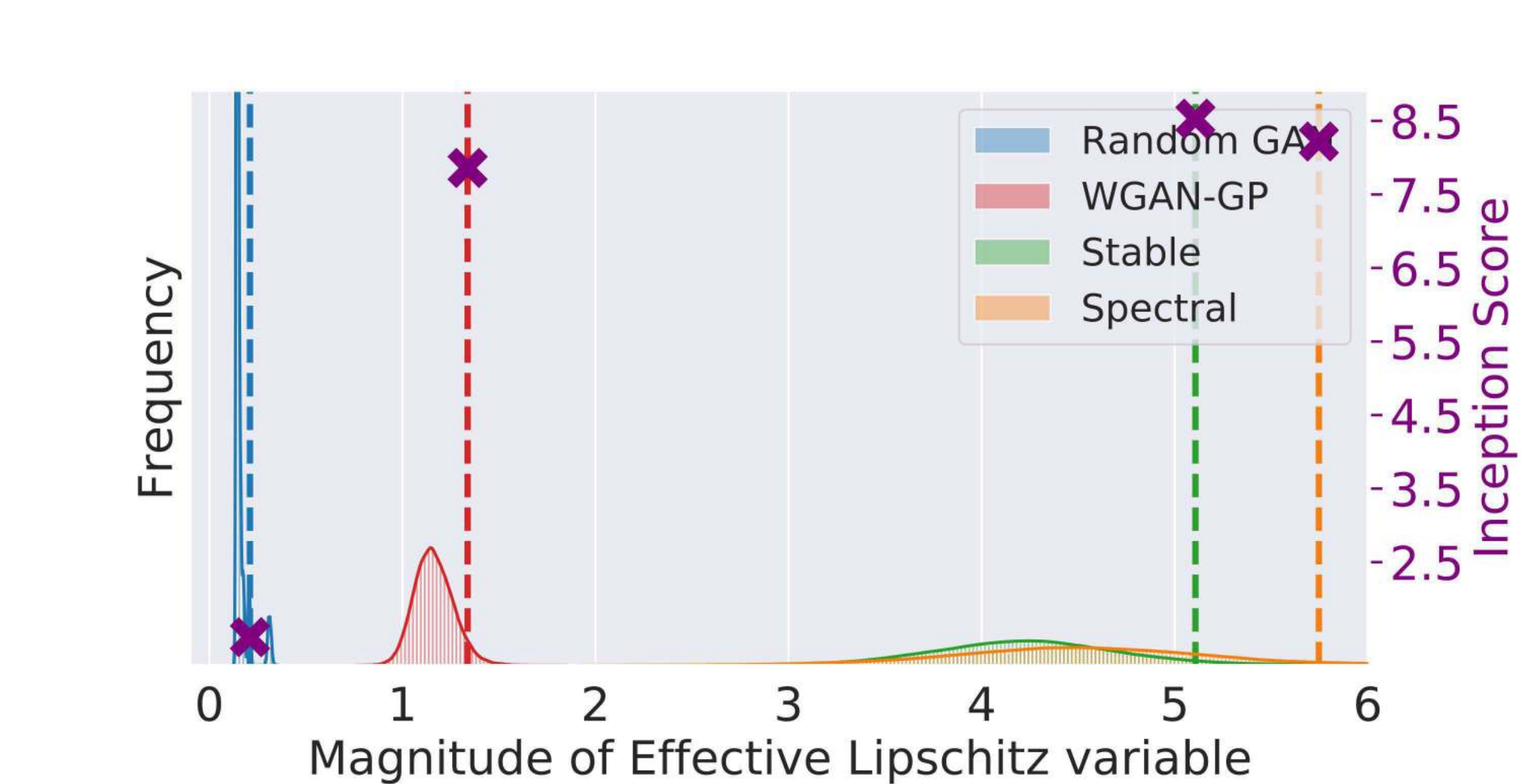} \caption{Unconditional GAN setting.}  \label{fig:lip_rank_stable_uncond_only_fake_vicinity}
\end{subfigure}
\caption{\footnotesize Jacobian norm of the discriminator in the neighbourhood of the samples from the \emph{generator} trained on CIFAR10.}\label{fig:lip_rank_stable_only_fake_vicinity}
\end{figure}

\begin{figure}[!htb]\vspace{-1em}

\begin{subfigure}[!h]{0.5\linewidth}
  \centering
  \def\svgwidth{0.99\columnwidth}
  \resizebox{0.95\textwidth}{!}{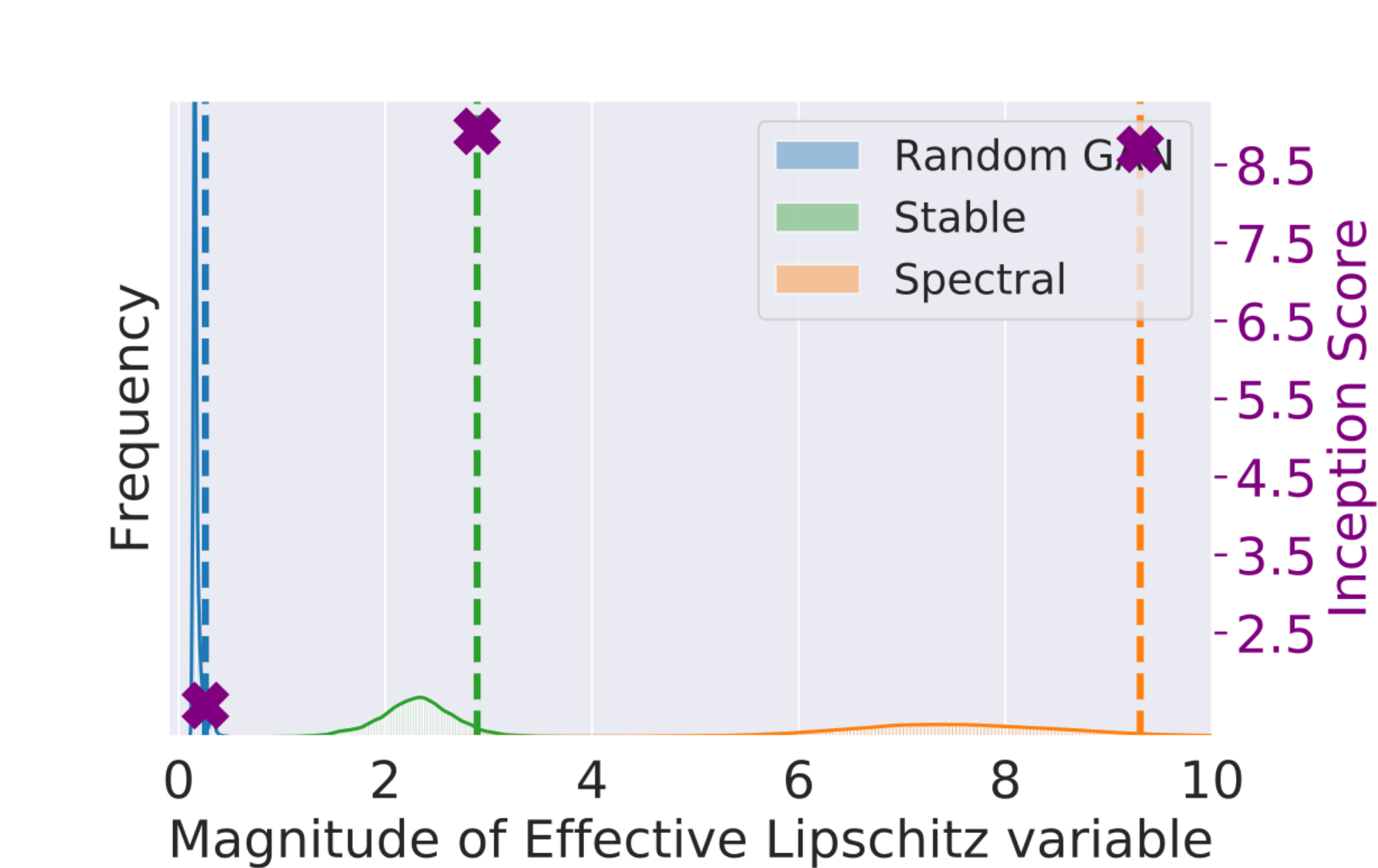}\caption{Conditional GAN with projection discriminator}  \label{fig:lip_rank_stable_cond_only_real_vicinity_double}
\end{subfigure}
\begin{subfigure}[!h]{0.5\linewidth}
  \centering
  \def\svgwidth{0.99\columnwidth}
  \resizebox{0.95\textwidth}{!}{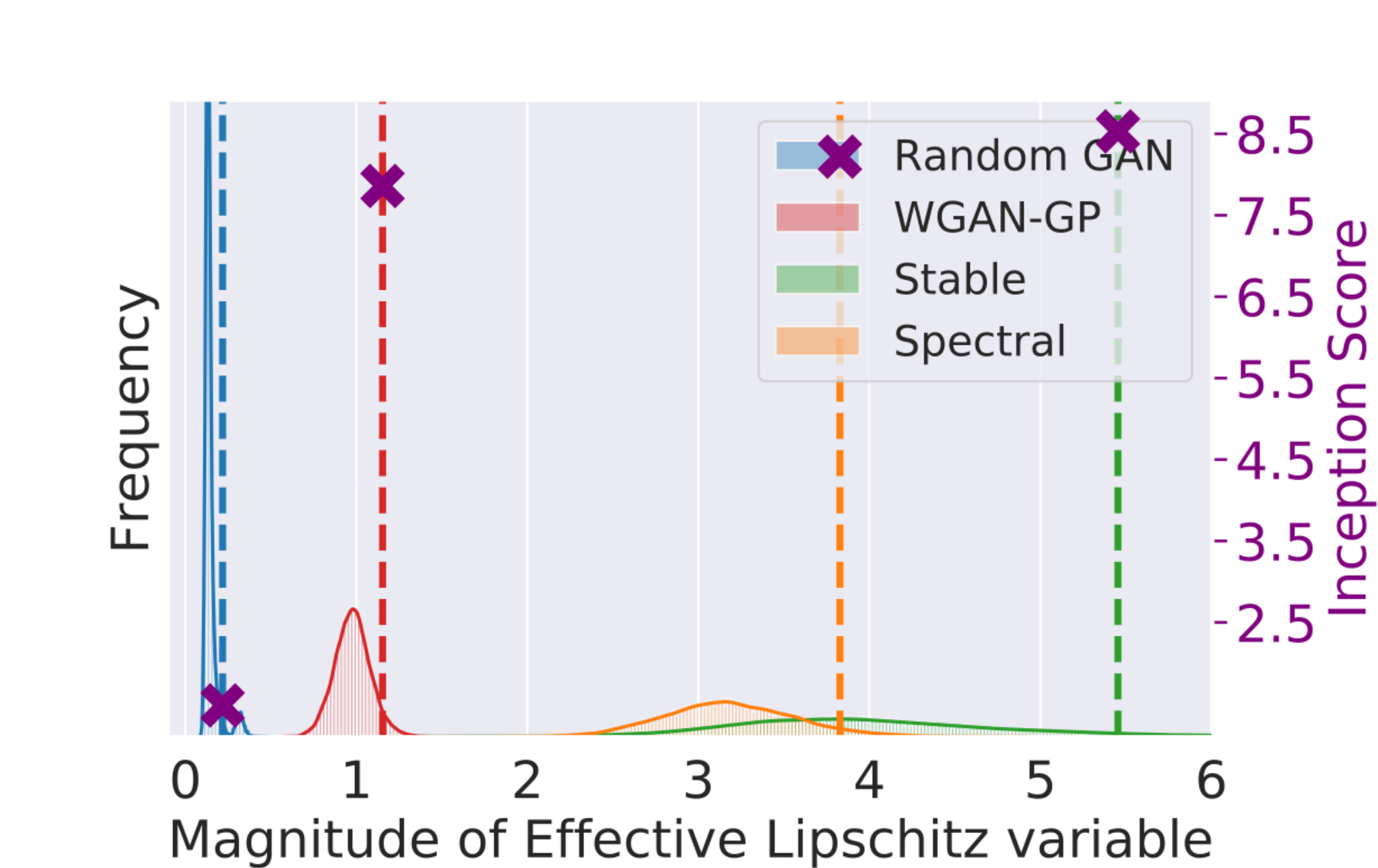} \caption{Unconditional GAN setting.}  \label{fig:lip_rank_stable_uncond_only_real_vicinity_2}
\end{subfigure}
\caption{\footnotesize Jacobian norm of the discriminator in the neighbourhood of the samples from the \emph{real dataset}  (CIFAR10).} \label{fig:lip_rank_stable_only_real_vicinity}
\end{figure}
\subsection{Training Stability}
\label{sec:noise-stability}

\begin{figure}[!h]
  \centering
  \def\svgwidth{0.5\columnwidth}
  \resizebox{0.95\textwidth}{!}{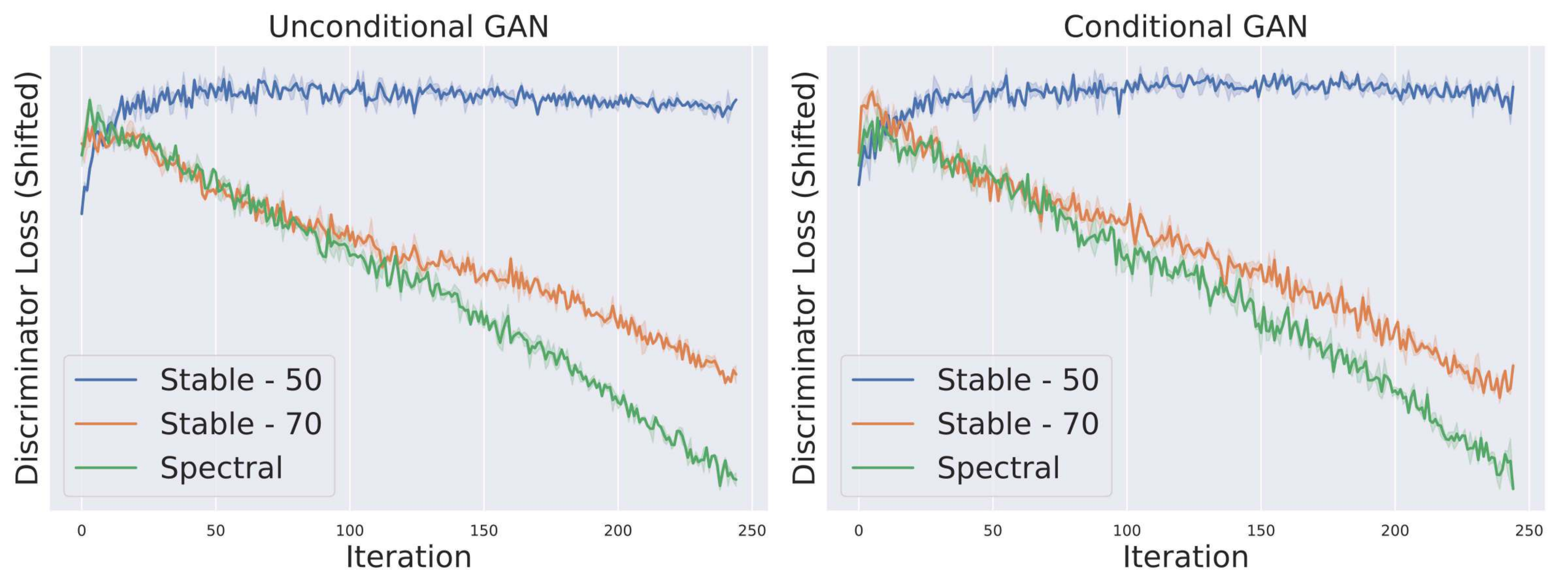} \caption{\footnotesize
    Loss incurred by the discriminator. The loss of \gls{srngan} with the stable rank constraint of $70$ is shifted upwards by $0.2$ so that we can compare the change of the loss during training as opposed to the absolute magnitude of the loss.}  \label{fig:disc_training_stability}
\end{figure}

\paragraph{Training Stability}

In~\cref{fig:disc_training_stability} we show the discriminator loss during the course of the training as an indicator of whether the generator gets sufficient gradient during training or not. These plots clearly suggest that the discriminator loss is more consistent for SRN than the SN.

\section{Examples of Generated Images}
\label{sec:sample-images}
\subsection{CelebA images}
For these images, we generated $100$ images from the respective models and hand-picked the $10$ best images in terms of visual quality.

\begin{center}
	\begin{figure}[!h]
	  \begin{subfigure}[c]{0.19\linewidth}
            \centering
            \def\svgwidth{0.99\columnwidth}
            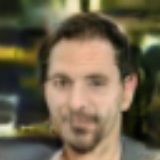 
	  \end{subfigure}
	  \begin{subfigure}[c]{0.19\linewidth}
            \centering
            \def\svgwidth{0.99\columnwidth}
            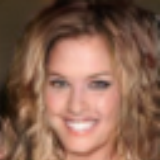 
	  \end{subfigure}
           \begin{subfigure}[c]{0.19\linewidth}
            \centering
            \def\svgwidth{0.99\columnwidth}
            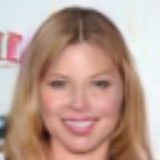 
	  \end{subfigure} \begin{subfigure}[c]{0.19\linewidth}
            \centering
            \def\svgwidth{0.99\columnwidth}
            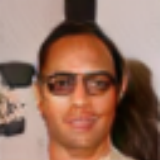 
	  \end{subfigure} \begin{subfigure}[c]{0.19\linewidth}
            \centering
            \def\svgwidth{0.99\columnwidth}
            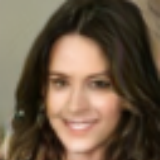 
	  \end{subfigure}\\ \begin{subfigure}[c]{0.19\linewidth}
            \centering
            \def\svgwidth{0.99\columnwidth}
            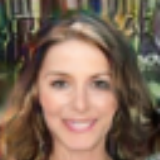 
	  \end{subfigure} \begin{subfigure}[c]{0.19\linewidth}
            \centering
            \def\svgwidth{0.99\columnwidth}
            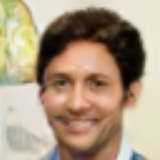 
	  \end{subfigure} \begin{subfigure}[c]{0.19\linewidth}
            \centering
            \def\svgwidth{0.99\columnwidth}
            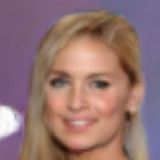 
	  \end{subfigure} \begin{subfigure}[c]{0.19\linewidth}
            \centering
            \def\svgwidth{0.99\columnwidth}
            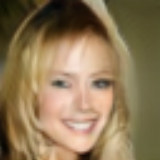 
	  \end{subfigure} \begin{subfigure}[c]{0.19\linewidth}
            \centering
            \def\svgwidth{0.99\columnwidth}
            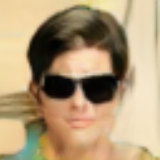 
	  \end{subfigure} 
            \caption{Image samples generated from the unconditional \gls{srngan}.}
            \label{fig:images_celeb_srn}
        \end{figure}
        
      \end{center}

\begin{center}
	\begin{figure}[!h]
	  \begin{subfigure}[c]{0.19\linewidth}
            \centering
            \def\svgwidth{0.99\columnwidth}
            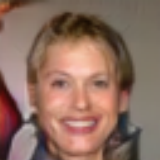 
	  \end{subfigure}
	  \begin{subfigure}[c]{0.19\linewidth}
            \centering
            \def\svgwidth{0.99\columnwidth}
            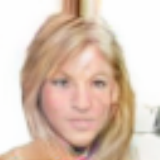 
	  \end{subfigure}
           \begin{subfigure}[c]{0.19\linewidth}
            \centering
            \def\svgwidth{0.99\columnwidth}
            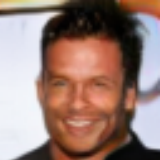 
	  \end{subfigure} \begin{subfigure}[c]{0.19\linewidth}
            \centering
            \def\svgwidth{0.99\columnwidth}
            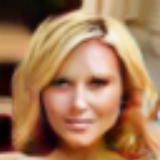 
	  \end{subfigure} \begin{subfigure}[c]{0.19\linewidth}
            \centering
            \def\svgwidth{0.99\columnwidth}
            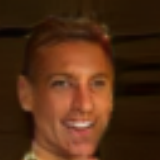 
	  \end{subfigure}\\ \begin{subfigure}[c]{0.19\linewidth}
            \centering
            \def\svgwidth{0.99\columnwidth}
            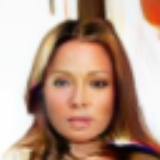 
	  \end{subfigure} \begin{subfigure}[c]{0.19\linewidth}
            \centering
            \def\svgwidth{0.99\columnwidth}
            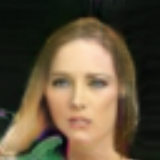 
	  \end{subfigure} \begin{subfigure}[c]{0.19\linewidth}
            \centering
            \def\svgwidth{0.99\columnwidth}
            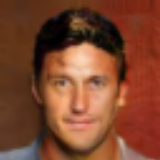 
	  \end{subfigure} \begin{subfigure}[c]{0.19\linewidth}
            \centering
            \def\svgwidth{0.99\columnwidth}
            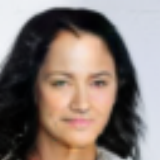 
	  \end{subfigure} \begin{subfigure}[c]{0.19\linewidth}
            \centering
            \def\svgwidth{0.99\columnwidth}
            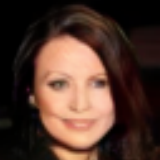 
	  \end{subfigure} 
            \caption{Image samples generated from the unconditional \gls{sngan}.}
            \label{fig:images_celeb_sn}
        \end{figure}
        
      \end{center}
      
\clearpage
\subsection{CIFAR10-Unconditional GAN}
\label{sec:unconditional-gan}
\begin{center}
  \begin{figure}[!h]\centering
    \begin{subfigure}[c]{0.9\linewidth}
      \centering
      \def\svgwidth{0.99\columnwidth}
      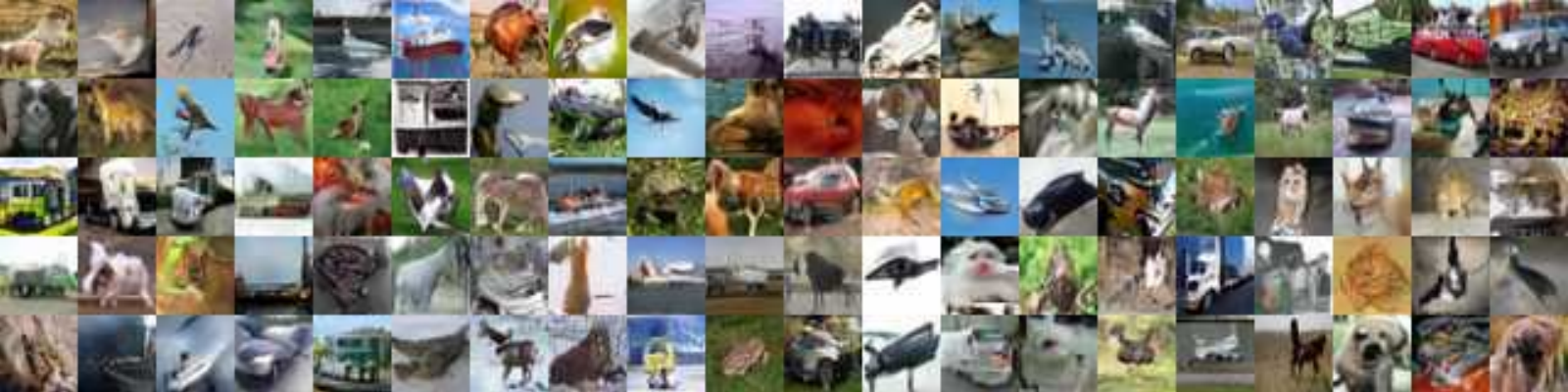\caption{SRN-70-GAN}  
    \end{subfigure}\\
    
    \begin{subfigure}[c]{0.9\linewidth}
      \centering
      \def\svgwidth{0.99\columnwidth}
      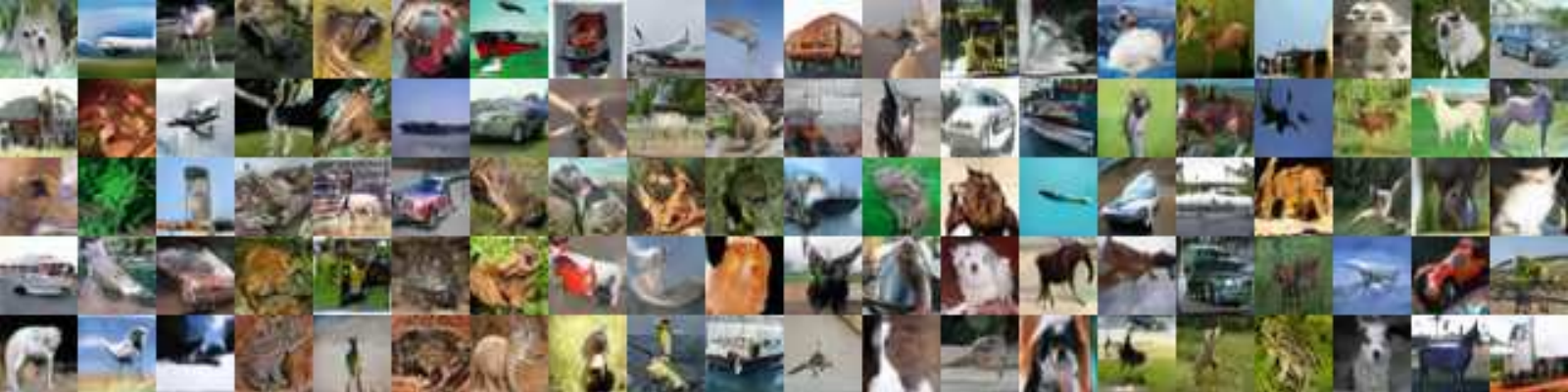\caption{SRN-50-GAN}  
    \end{subfigure}\\
    
    \begin{subfigure}[c]{0.9\linewidth}
      \centering
      \def\svgwidth{0.99\columnwidth}
      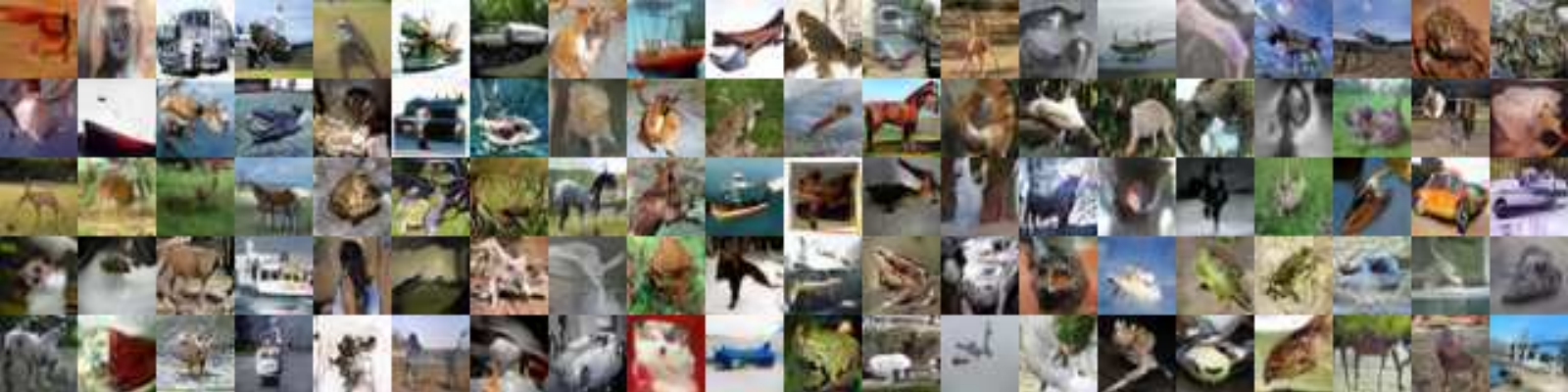\caption{\gls{sngan}}  
    \end{subfigure}\\
    
    \begin{subfigure}[c]{0.9\linewidth}
      \centering
      \def\svgwidth{0.99\columnwidth}
      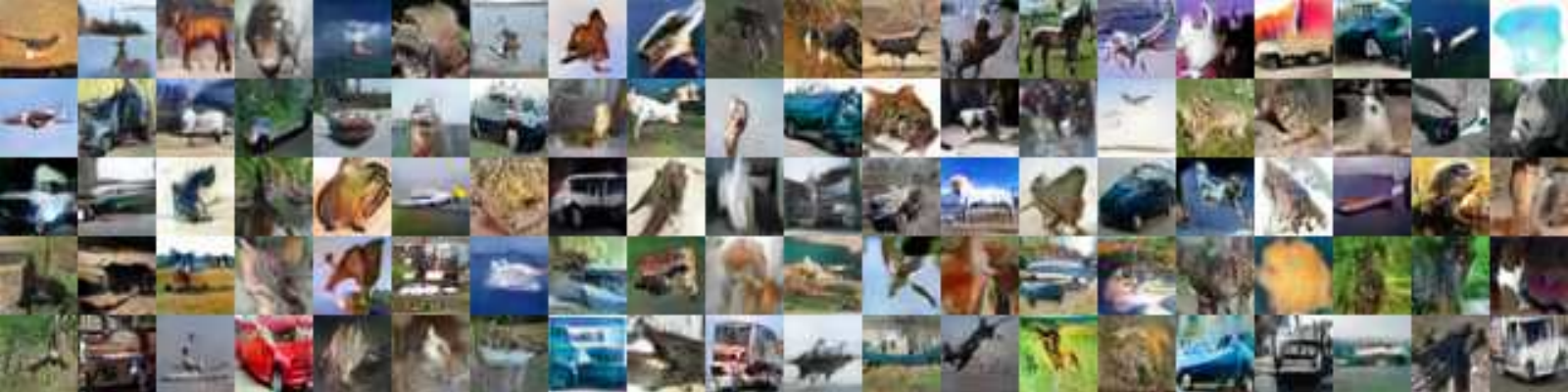\caption{WGAN-GP}  
    \end{subfigure}
    \caption{Image samples generated from the unconditional \gls{srngan},
      \gls{sngan}, and WGAN-GP.}
    \label{fig:uncond_images_gan}
  \end{figure}
  
\end{center}

\clearpage
\subsection{CIFAR10-Conditional \gls{srngan}}
\label{sec:uncond-gan}
\begin{center}
  \begin{figure}[h!]\centering
    \begin{subfigure}[c]{0.3\linewidth}
      \centering
      \def\svgwidth{0.99\columnwidth}
      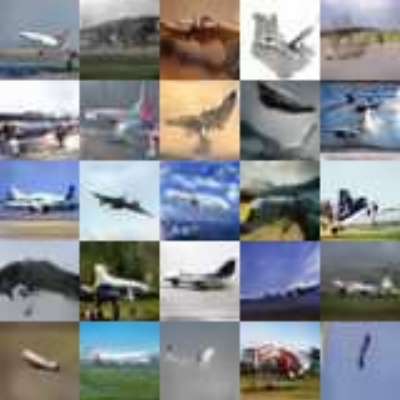\caption{Air-planes}  
    \end{subfigure}%
    \begin{subfigure}[c]{0.3\linewidth}
      \centering
      \def\svgwidth{0.99\columnwidth}
      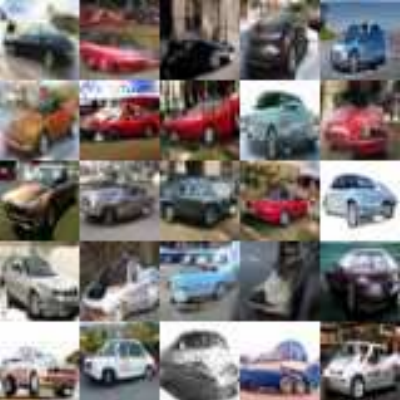\caption{Cars}  
    \end{subfigure}
    \begin{subfigure}[c]{0.3\linewidth}
      \centering
      \def\svgwidth{0.99\columnwidth}
      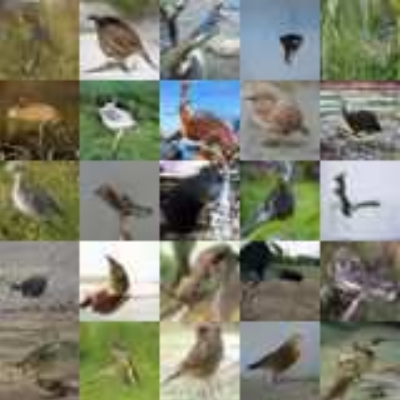\caption{Birds}  
    \end{subfigure}%
    \\
    \begin{subfigure}[c]{0.3\linewidth}
      \centering
      \def\svgwidth{0.99\columnwidth}
      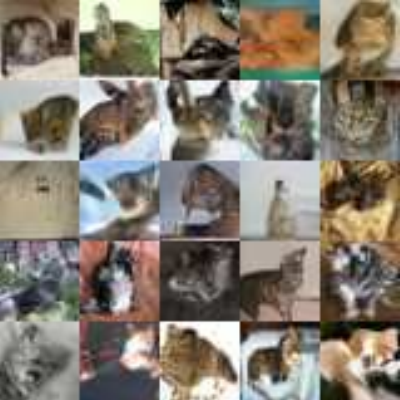\caption{Cats}  
    \end{subfigure}
       \begin{subfigure}[c]{0.3\linewidth}
      \centering
      \def\svgwidth{0.99\columnwidth}
      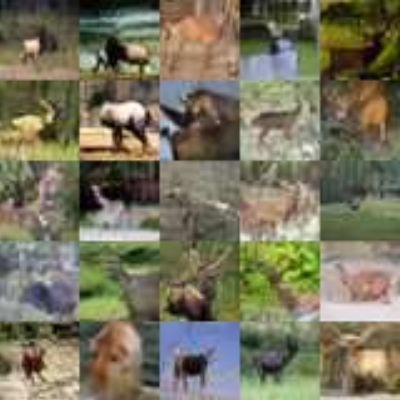\caption{Deers}  
    \end{subfigure}%
    \begin{subfigure}[c]{0.3\linewidth}
      \centering
      \def\svgwidth{0.99\columnwidth}
      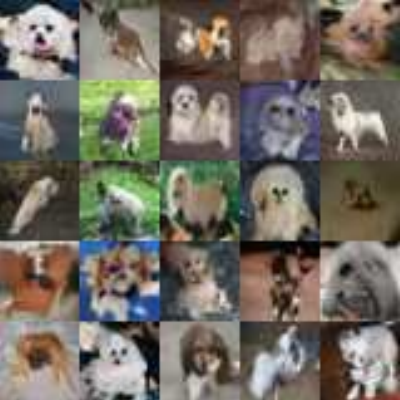\caption{Dogs}  
    \end{subfigure}
    \caption{Image samples generated from the conditional SRN-GAN  with
      projection discriminator.}
      \label{fig:Images_1}
  \end{figure}
\end{center}

\end{document}